\documentclass[final,3p,times]{elsarticle}
\usepackage[implicit=true,unicode=true,
 bookmarks=true,bookmarksnumbered=true,bookmarksopen=true]
 {hyperref}
\hypersetup{
  pdftitle={my title},
  pdfauthor={my name},
  pdfsubject={my subject},
  pdfkeywords={my keywords},
  colorlinks=true,
  allcolors=blue,
  pdfpagelayout=OneColumn,
  pdfnewwindow=true,
}

\usepackage{graphicx}
\usepackage{algpseudocode,algorithm}
\usepackage{algorithmicx}
\algrenewcommand\textproc{}% Used to be \textsc
\usepackage{subcaption}
\usepackage{amsthm,amssymb}
\usepackage{chngcntr}
\usepackage{amsfonts}
\usepackage{mathtools}
\usepackage{paralist}
\usepackage{autoaligne}
\usepackage{multirow}
\usepackage{xspace}
\usepackage{pgf, tikz}
\usetikzlibrary{arrows, automata}

\newtheorem{theorem}{Theorem}[section]

\newtheorem{example}{Example}[section]
\newtheorem{definition}{Definition}[section]
\newtheorem{problem}{Problem}[section]

\makeatletter
\def\ps@pprintTitle{%
    \let\@oddhead\@empty
    \let\@evenhead\@empty
    \def\@oddfoot{\footnotesize\itshape
         {} \hfill\today}%
    \let\@evenfoot\@oddfoot
    }
\makeatother

\begin{document}

\begin{frontmatter}



\title{An Extension-based Approach for Computing and Verifying Preferences in Abstract Argumentation}


\author[1]{Quratul-ain Mahesar\corref{cor1}}
\ead{q.mahesar@hud.ac.uk}
\cortext[cor1]{Corresponding author}
\author[2]{Nir Oren}
\ead{n.oren@abdn.ac.uk}
\author[2]{Wamberto W. Vasconcelos}
\ead{w.w.vasconcelos@abdn.ac.uk}

\affiliation[1]{organization={School of Computing and Engineering, University of Huddersfield},%Department and Organization
            % addressline={Queensgate}, 
            % city={Huddersfield},
            % postcode={HD1 3DH}, 
            % state={West Yorkshire},
            country={UK}}

\affiliation[2]{organization={Department of Computing Science, University of Aberdeen},%Department and Organization
            % addressline={Meston building}, 
            % city={Aberdeen},
            % postcode={AB24 3UE}, 
            % state={Aberdeenshire},
            country={UK}}

\begin{abstract}
We present an extension-based approach for computing and verifying preferences in an abstract argumentation system. Although numerous argumentation semantics have been developed previously for identifying acceptable sets of arguments from an argumentation framework, there is a lack of justification behind their acceptability based on implicit argument preferences. Preference-based argumentation frameworks allow one to determine what arguments are justified given a set of preferences. Our research considers the inverse of the standard reasoning problem, i.e., given an abstract argumentation framework and a set of justified arguments, we compute what the possible preferences over arguments are. Furthermore, there is a need to verify (i.e., assess) that the computed preferences would lead to the acceptable sets of arguments. This paper presents a novel approach and algorithm for exhaustively computing and enumerating all possible sets of preferences (restricted to three identified cases) for a conflict-free set of arguments in an abstract argumentation framework. We prove the soundness, completeness and termination of the algorithm. The research establishes that preferences are determined using an extension-based approach after the evaluation phase (acceptability of arguments) rather than stated beforehand. In this work, we focus our research study on grounded, preferred and stable semantics. We show that the complexity of computing sets of preferences is exponential in the number of arguments, and thus, describe an approximate approach and algorithm to compute the preferences. Furthermore, we present novel algorithms for verifying (i.e., assessing) the computed preferences. We provide details of the implementation of the algorithms (source code has been made available), various experiments performed to evaluate the algorithms and the analysis of the results.
\end{abstract}




\end{frontmatter}



\section{Introduction}
Preferences play a central part in decision making and have been extensively studied in various disciplines such as economy, operations research, psychology and philosophy~\cite{Pigozzi2016}. Preferences are used in many areas of artificial intelligence including planning, scheduling, multi-agent systems, combinatorial auctions and game playing~\cite{Walsh07}. Preference elicitation is a very difficult task and automating the process of preference extraction can be very difficult. The complexity of eliciting preferences and representational questions like dealing with uncertainty has remained a very active research area~\cite{Konczak05votingprocedures,PINI20111272,Walsh07}. Preference elicitation plays a vital role in decision support systems~\cite{Sprague:1993:DSS:167095,Mahesar17} and recommender systems~\cite{Ricci:2010:RSH:1941884}, where the most suitable decision(s) or recommendation(s) can be identified and justified with the help of preferences. Furthermore, elicited preferences can be utilized in dialogue strategies~\cite{ki-Thimm14,ijcai-RienstraTO13}, for instance in computational persuasion~\cite{argcom-Hunter18,HadouxHP23} or negotiation~\cite{RahwanRJMPS03} -- where an agent may have the capability of inferring preferences and reach her goal if (s)he enforces at least one of several desired sets of arguments with the application of preferences. The inferred preferences can be exploited in optimizing the choice of move in persuasion dialogues for behaviour change as well as in negotiation dialogues to reach agreement. 

Argumentation has gained an increasing popularity in Artificial Intelligence (AI). It has been widely used for handling inconsistent knowledge bases~\cite{BESNARD2001203,Garcia:2004,SIMARI1992125}, and dealing with uncertainty in decision making~\cite{AMGOUD2009413,Bonet:1996,Muller:2012}. Logic-based abstract argumentation~\cite{Dung95onthe} provides a formal representation of preferences. An abstract argumentation framework is a directed graph consisting of nodes that represent unique atomic arguments and directed edges that represent an attack between two arguments. This visual representation of an argumentation framework as a directed graph is also known as an argumentation graph. Acceptable sets of arguments called extensions for an argumentation framework can be computed based on various acceptability semantics~\cite{Dung95onthe}. 

Arguments can have different strengths, e.g., an argument relies on more certain or important information than another. This has led to the introduction of preference-based argumentation frameworks consisting of preference relations between arguments~\cite{Amgoud:1998,AMGOUD2009413,MODGIL2009901,Prakken1997ArgumentBasedEL}. Furthermore, preferences are taken into account in the evaluation of arguments at the semantic level, which is also known as preference-based acceptability~\cite{Amgoud2002}. The basic idea is to accept undefeated arguments and also arguments that are preferred to their attacking arguments, as these arguments can defend themselves against their attacking arguments. 

Preference-based argumentation framework (PAF)~\cite{Amgoud:1998,AMGOUD2014585} allows one to determine what arguments are justified given a set of preferences. In our research, we consider the inverse of the standard reasoning problem, i.e., given an abstract argumentation framework and a set of justified arguments, we compute what the possible preferences over arguments are.
Although a preference-based argumentation framework (PAF) has been previously studied to represent an abstract argumentation framework~\cite{KaciT08}, there seems to be no previous work on automatically computing implicit argument preferences in an abstract argumentation framework using an extension-based approach. Furthermore, there have been no attempts to perform an exhaustive search for all possible preferences, and their explicit enumeration. 

There are two aims of our research study. The first aim of our research is to exhaustively compute all possible sets of argument preferences (restricted to three identified cases) that hold for a given set of conflict-free arguments, i.e., extension, in an abstract argumentation framework. In this work, we focus our research study on grounded, preferred and stable semantics. We present a novel algorithm to perform this computation. We show that the complexity of computing sets of preferences is exponential in the number of arguments, and thus, describe an approximate approach and algorithm to compute the preferences that is scalable. The second aim of our research is to verify (i.e., assess) that all the computed sets of preferences are correct, i.e., each set of preferences when applied to a given abstract argumentation framework results in the original input extension under a given semantics. We present novel algorithms to perform this verification. All algorithms have been implemented. We have build a complete system for computing and verifying preferences, performed various experiments to evaluate the algorithms and analyze the results.

The current paper extends and improves our previous work~\cite{prima-Mahesar18}. The main contributions of our work are as follows:
\begin{enumerate}
	\item An extension-based approach is employed for computing and verifying argument preferences. Thus, preferences specifically justify the reasoning behind the acceptability of the arguments in an extension.
	\item Preferences are computed at the end of the argumentation process and need not be stated in advance.
	\item Exhaustive search is performed to compute all possible sets of preferences.
	\item The approach for computing preferences operates on a conflict-free extension as input which is the minimal acceptability semantic, therefore, it can take as input most of the extensions given in the literature and stated in this paper. 
	\item We present novel algorithms for computing preferences (and additional algorithms for filtering preferences).
    \item We present a novel approximate algorithm for computing preferences.
	\item We present novel algorithms for verifying preferences.
	\item Reference implementation of our algorithms is provided\footnote{The source code of the implementation of all algorithms is available at \url{https://github.com/Quratul-ain/AAF_Preferences}}.
	\item Experimental setup (including data sets in the Appendix\footnote{All data sets are included in the Appendix.}), various experiments that have been performed to evaluate the algorithms and analysis of the results obtained is presented.
\end{enumerate} 
In comparison to~\cite{prima-Mahesar18}, the contributions $6$, $7$, $8$ and $9$ stated above are new.

The remainder of this paper is structured as follows. In Section~\ref{sec:related-work}, we present related work. Section~\ref{sec:background} presents the preliminaries that include the background on abstract argumentation framework and acceptability semantics for acceptable set of arguments also known as extensions. This is followed by background on preference-based argumentation framework. In Section~\ref{sec:compute-preferences}, we present our approach and an algorithm for computing all possible sets of preferences for a given extension and abstract argumentation framework, and we prove the soundness, completeness and termination of the algorithm. Additionally, we present algorithms for filtering preferences. Furthermore, we present an approximate approach and algorithm for computing a set of preferences. In Section~\ref{sec:verify-prefs}, we present our approach and algorithms for verifying all computed sets of preferences for a given extension and abstract argumentation framework. In Section~\ref{sec:imp-eval}, we present the implementation details and evaluation. Finally, we conclude and suggest future work in Section~\ref{sec:conclusion}.

%\section{\hyperlink{sec:related-work}{Related Work}}
%\section{\hypertarget{sec:related-work{Related Work}}
\section{Related Work}
\label{sec:related-work}

Several variations of argumentation frameworks with preferences have been studied previously. Preference-based argumentation framework (PAF)~\cite{Amgoud:1998,AMGOUD2014585} is an extension of a standard argumentation framework~\cite{Dung95onthe} consisting of preference relations between arguments. The idea is to accept undefeated arguments and also arguments that are preferred to their attacking arguments. Value-based argumentation framework (VAF)~\cite{bench-capon03-persuasion} extends a standard argumentation framework to take into account values promoted by arguments. Preferences over arguments are determined by the values the arguments promote or support. The idea is to accept undefeated arguments and also arguments which promote values that are more important or preferred to the values promoted by their attacking arguments. Furthermore, value-based argumentation frameworks (VAF) have been extended to take into account the possibility that arguments may support multiple values, and therefore, various types of preferences over values could be considered in order to deal with real world situations~\cite{KaciT08}. Another variation is an extended argumentation framework (EAF)~\cite{MODGIL2009901} that considers the case where arguments can express preferences between other arguments. 

Further studies on preference-based argumentation frameworks led to the observation that ignoring the attacks where the attacked argument is stronger than the attacking argument does not always give intuitive results~\cite{MODGIL2009901}, since the resulting extension violates the basic condition imposed on acceptability semantics, which is the conflict-freeness of extensions, thus violating the rationality postulates given in~\cite{CAMINADA2007286}. This problem was later resolved in a new preference-based argumentation framework that guarantees conflict-free extensions with a symmetric conflict relation~\cite{Amgoud2011,MODGIL2009901}. The preference relation is then used to determine the direction of the defeat relation between the two arguments. Furthermore,  preference relations have been used to refine the results of a framework by comparing its extensions~\cite{AMGOUD2014585}. 

Although our work is based on abstract argumentation framework, several variations of structured argumentation frameworks with preferences have been studied previously. ABA\textsuperscript{+}~\cite{CyrasT16} generalises preference-based argumentation framework (PAF) \cite{AMGOUD2014585} that introduced the concept of attack reversal from less preferred arguments. Another extension of the ABA framework with preferences is (p\_ABA)~\cite{Wakaki17} that employs preferences on the extension level to discriminate among extensions. ASPIC\textsuperscript{+}~\cite{ModgilP14} encompasses many key elements of structured argumentation such as strict and defeasible rules, general contrariness mapping and various forms of attacks as well as preferences. DeLP~\cite{Garcia:2004}, an early version of preference-based argumentation framework~\cite{AMGOUD2014585}, and Deductive Argumentation~\cite{BesnardH14}, use preferences to discard attacks from arguments less preferred than the attacked arguments.

While the above argumentation frameworks allow handling of preferences over arguments, the main limitation is that preferences need to be stated in advance. More recently, in ~\cite{prima-Mahesar20} we have extended our work~\cite{prima-Mahesar18} in abstract argumentation to structured argumentation, i.e., assumption-based argumentation frameworks ABA~\cite{aba_Dung2009,aba_Toni14} and ABA\textsuperscript{+}~\cite{CyrasT16}, where preferences are computed at the assumption level rather than abstract arguments. Another recent extension is presented in~\cite{ijar-HungH21}, where hidden argument preferences of a group of agents are revealed using answer sets. Furthermore, to the best of our knowledge, no previous work has been performed for the verification of the newly computed preferences along with the implementation and experimental analysis of both the computation and verification of preferences, that is presented in this extended version of our paper~\cite{prima-Mahesar18}.

\section{Preliminaries}
\label{sec:background}

An \textit{argumentation framework} is a set of arguments and a binary attack relation among them. Given an argumentation framework, argumentation theory allows to identify the sets of arguments that can survive the conflicts expressed in the framework. In this work, we consider finite abstract argumentation frameworks.

\begin{definition}{(Abstract Argumentation Framework~\cite{Dung95onthe}):}
	\label{def:arg-framework}
	An abstract argumentation framework (AAF) is a pair $\mathit{AAF} = (\mathcal{A}, \mathcal{R})$, where $\mathcal{A}$ is a set of arguments and $\mathcal{R}$ is an attack relation $(\mathcal{R} \subseteq \mathcal{A} \times \mathcal{A})$. The notation $(A,B) \in \mathcal{R}$ where $A,B \in \mathcal{A}$ denotes that \textit{A attacks B}.
\end{definition}

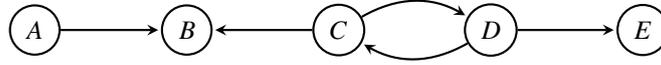
\begin{figure}[!htb]
	\centering
	\begin{tikzpicture}[
	> = stealth, % arrow head style
	shorten > = 1pt, % don't touch arrow head to node
	auto,
	node distance = 2cm, % distance between nodes
	thick % line style
	]
	\tikzstyle{every state}=[
	draw = black,
	thick,
	fill = white,
	minimum size = 4mm
	]
	\node[state] (A) {$A$};
	\node[state] (B) [right of=A] {$B$};
	\node[state] (C) [right of=B] {$C$};
	\node[state] (D) [right of=C] {$D$};
	\node[state] (E) [right of=D] {$E$};
	
	\path[->] (A) edge node {} (B);
	\path[->] (C) edge node {} (B);
	\path[->] (C) edge[bend left] node {} (D);        
	\path[->] (D) edge[bend left] node {} (C);     
	\path[->] (D) edge node {} (E);
	\end{tikzpicture}
	\caption{Example abstract argumentation framework $\mathit{AAF}_1$}
	\label{fig:abst_arg_graph1}
\end{figure}

An abstract argumentation framework is a directed graph where the arguments are represented as nodes and the attack relations as directed edges. An example abstract argumentation framework $(\mathcal{A}, \mathcal{R})$ is shown in Figure~\ref{fig:abst_arg_graph1}, where $\mathcal{A} = \lbrace A, B, C , D, E \rbrace$ and $\mathcal{R} = \lbrace (A, B), (C, B), (C, D), (D, C), (D, E) \rbrace$, which means that $A$ attacks $B$, $C$ attacks both $B$ and $D$, and $D$ attacks both $C$ and $E$.

Dung~\cite{Dung95onthe} originally introduced an extension approach to define the acceptability of arguments in an argumentation framework. An extension is a subset of $\mathcal{A}$ that represents the set of arguments that can be accepted together. Dung's semantics are based on a \textit{conflict-free} set of arguments, i.e., a set should not be self-contradictory nor include arguments that attack each other. This is defined formally as follows.
 
\begin{definition}{(Conflict-freeness):}
	Let $(\mathcal{A}, \mathcal{R})$ be an argumentation framework. The set $\mathcal{E} \subseteq \mathcal{A}$ is \textit{conflict-free} if and only if there are no $A, B \in \mathcal{E}$ such that $(A,B) \in \mathcal{R}$
\end{definition}

The minimal requirement for an extension to be acceptable is \textit{conflict-freeness}. Many other acceptability semantics have been introduced in the literature, and from these the most common are given as follows.

\begin{definition}{(Extensions):}
	\label{def:extensions}
	Let $\mathit{AAF} = (\mathcal{A}, \mathcal{R})$ be an argumentation framework, and set $\mathcal{E} \subseteq \mathcal{A}$ and $A,B,C \in \mathcal{A}$
	\begin{itemize}
		\item $\mathcal{E}$ is admissible iff it is conflict free and defends all its arguments. $\mathcal{E}$ defends $A$ iff for every argument $B \in \mathcal{A}$, if we have $(B,A) \in \mathcal{R}$ then there exists $C \in \mathcal{E}$ such that $(C,B) \in \mathcal{R}$.
		\item $\mathcal{E}$ is a complete extension iff $\mathcal{E}$ is an admissible set which contains all the arguments it defends.
		\item $\mathcal{E}$ is a preferred extension iff it is a maximal (with respect to set inclusion) admissible set.
		\item $\mathcal{E}$ is a stable extension iff it is conflict-free and for all $A \in \mathcal{A} \setminus \mathcal{E}$, there exists an argument $B \in \mathcal{E}$ such that $(B,A) \in \mathcal{R}$.
		\item $\mathcal{E}$ is a grounded extension iff $\mathcal{E}$ is a minimal (for set inclusion) complete extension.
	\end{itemize}
\end{definition}

Every argumentation framework has at least one admissible set (the empty set), exactly one grounded extension, one or more complete extensions, one or more preferred extensions, and zero or more stable extensions. The following example shows the extensions for the abstract argumentation framework of Figure~\ref{fig:abst_arg_graph1}.

\begin{example}
	\label{ex:arg-example1}
	Given the abstract argumentation framework of Figure~\ref{fig:abst_arg_graph1}, then we compute its extensions as follows:
	\begin{itemize}
		\item Conflict free: $\{A,C,E\}, \{A, D\}, \{B, D\}, \{A,C\}, \{A,E\}, \{B,E\}, \{C,E\}, \{A\}, \{B\},$
		
		$\{C\},\{D\}, \{E\}, \emptyset$
		\item Admissible: $\{A,C,E\}, \{A,C\}, \{A, D\}, \{C,E\}, \{A\}, \{C\},\{D\}, \emptyset$
		\item Complete: $\{A,C,E\}, \{A,D\}, \{A\}$
		\item Preferred: $\{A,C,E\}, \{A, D\}$
		\item Stable: $\{A,C,E\}, \{A, D\}$
		\item Grounded: $\{A\}$
	\end{itemize}
\end{example}

While an abstract argumentation framework captures the basic interactions between arguments, it does not consider factors such as argument strength, i.e., arguments may not necessarily have the same strengths~\cite{BENFERHAT1993411,Cayrol-Royer-Saurel,SIMARI1992125}. Consequently, preferences over arguments can be added to the argumentation framework and taken into account in order to evaluate arguments~\cite{AMGOUD2009413,MODGIL2009901,Prakken1997ArgumentBasedEL}, which is demonstrated in the following example~\cite{Amgoud2002}.

\begin{example}
	Let $(\mathcal{A}, \mathcal{R})$ be an argumentation framework with $\mathcal{A} = \lbrace A, B, C \rbrace$ and $\mathcal{R} = \lbrace (A,B), (B,C) \rbrace$. The set of acceptable argument is $\lbrace A,C \rbrace$. However, suppose argument $B$ is preferred to $A$ and $C$. How can we combine the preference over arguments and the attack relation to decide which arguments are acceptable? We can say that, since $B$ is preferred to $A$, it can defend itself from the attack of $A$. This would lead us to accepting $B$ and rejecting $C$.
\end{example}
Dung's framework has been extended by introducing preference relations into argumentation systems, which is known as a preference-based argumentation framework (PAF)~\cite{Amgoud:1998}. A PAF extends an abstract argumentation framework to account for preferences over arguments. The attack relation in a preference-based argumentation framework is called defeat, and is denoted by $\mathit{Def}$.

\begin{definition}{(Preference-based Argumentation Framework (PAF)~\cite{Amgoud:1998}):}
	\label{def:paf}
A preference-based argumentation framework is a triple $(\mathcal{A}, \mathit{Def}, \geq)$ where $\mathcal{A}$ is a set of arguments, $\mathit{Def}$ is the defeat binary relation on $\mathcal{A}$, and $\geq$ is a (partial or total) pre-ordering defined on $\mathcal{A} \times \mathcal{A}$. The notation $(A, B) \in \mathit{Def}$ means that argument $A$ defeats argument $B$. 
\end{definition}

The notation $A \geq B$ means that argument $A$ is at least as preferred as $B$ and the relation $>$ is the strict counterpart of $\geq$.

\begin{definition}
\label{def:pref-application}
    	Let there be an abstract argumentation framework. Preferences could be applied in two ways~\cite{AMGOUD2014585}: 
	\begin{enumerate}
	    \item one way is to apply preferences at the time of argument acceptability (semantic level); and
	    \item second way is to compute all preferred extensions and filter them by the application of the preferences.
	\end{enumerate}  
\end{definition}

\begin{figure}[!htb]
	\centering
	\begin{tikzpicture}[
	> = stealth, % arrow head style
	shorten > = 1pt, % don't touch arrow head to node
	auto,
	node distance = 2cm, % distance between nodes
	thick % line style
	]
	
	\tikzstyle{every state}=[
	draw = black,
	thick,
	fill = white,
	minimum size = 4mm
	]

	\node[state] (A) {$A$};
	\node[state] (B) [right of=A] {$B$};
	\node[state] (C) [right of=B] {$C$};
	\node[state] (D) [right of=C] {$D$};
	
	\path[->] (A) edge[bend left] node {} (B);
	\path[->] (B) edge[bend left] node {} (A);
	\path[->] (B) edge node {} (C);
	\path[->] (C) edge node {} (D);     
	
	\end{tikzpicture}
	\caption{Example abstract argumentation framework $\mathit{AAF}_2$}\label{fig:pref_arg_graph}
\end{figure}
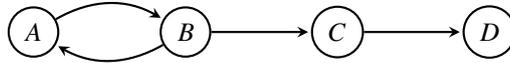

\begin{example}
	Let there be an abstract argumentation framework of Figure~\ref{fig:pref_arg_graph}. 	
	By using the first method given in Definition~\ref{def:pref-application}, if we assume $\{A>B, C>D\}$ is the set of preferences between arguments, then we get a single extension $\mathcal{E} = \{ A,C\}$. Now, by using the second method, we first compute all preferred extensions $\{ A,C\}, \{B, D\}$. These extensions could now be filtered by the application of the set of preferences $\{A>B, C>D\}$ which suggest $\{ A,C\}$ to be better than $\{ B,D\}$.
\end{example}

In this work, we use the first method of applying preferences at the time of argument acceptability (semantic level).

\section{Computing Preferences}
\label{sec:compute-preferences}

A preference-based argumentation framework (PAF) can represent an abstract argumentation framework~\cite{KaciT08}:
\begin{definition}{(PAF representing an AAF)}
	A preference-based argumentation framework $(\mathcal{A}, \mathit{Def}, \geq)$ represents an abstract argumentation framework $(\mathcal{A}, \mathcal{R})$ iff $\forall A,B \in \mathcal{A}$, it is the case that $(A,B) \in \mathcal{R}$ iff $(A,B)\in \mathit{Def}$ and it is not the case that $B>A$.
\end{definition}

It has been previously shown that each preference-based argumentation framework represents one abstract argumentation framework, however each abstract argumentation framework can be represented by various preference-based argumentation frameworks~\cite{KaciT08}. Following this, we introduce an extension-based approach for computing sets of preferences for a subset of conflict-free arguments in an abstract argumentation framework. For any two arguments $A$ and $B$ in an argumentation framework, we use the strict preference relation $A>B$ to denote that $A$ is strictly preferred to $B$, i.e., $A$ is of greater strength than $B$, and we use the preference relation $A=B$ to denote that $A$ and $B$ are of equal strength or preference. We list below the three cases we have identified for which the preferences are computed for a given conflict-free extension $\mathcal{E}$ in an abstract argumentation framework $\mathit{AAF} = \langle \mathcal{A}, \mathcal{R} \rangle$. The motivation behind the identified three cases is that we want to find out why a set of arguments are in an extension of the abstract argumentation framework based on the relationship of attack relations between the arguments and their strengths (i.e., preferences between the arguments). The three identified cases are:
\begin{itemize}
	\item \textbf{Case 1:} Suppose $\alpha, \beta \in \mathcal{A}$ and $\alpha \in \mathcal{E}$, $\beta \notin \mathcal{E}$ such that 
	$\alpha$ is attacked by argument $\beta$, and $\alpha$ is not defended by any other argument (not equal to $\alpha$) in the extension. We have the following preferences for all such $\alpha$ and $\beta$: $\alpha > \beta$. 
	\item \textbf{Case 2:} Suppose $\alpha,\beta \in \mathcal{A}$ and $\alpha \in \mathcal{E}$, $\beta \notin \mathcal{E}$, and suppose $\alpha$ attacks argument $\beta$ and $\beta$ does not attack $\alpha$. We have the following preferences for all such $\alpha$ and $\beta$: $\beta \ngtr \alpha$, i.e., $(\alpha > \beta) \vee (\alpha = \beta)$.
	\item \textbf{Case 3:} Suppose $\alpha, \beta, \gamma \in \mathcal{A}$ and $\alpha, \gamma \in \mathcal{E}$, $\beta \notin \mathcal{E}$ where $\alpha$, $\beta$ and $\gamma$ are different arguments, such that, $\alpha$ is attacked by argument $\beta$ but defended by argument $\gamma$ in the extension, i.e., $\gamma$ attacks $\beta$. We have the following preferences for all such $\alpha$ and $\beta$: $(\alpha > \beta) \vee (\alpha = \beta) \vee (\beta > \alpha)$.
\end{itemize}

In other words, we want to determine and compute the preferences between arguments that will ensure that a set of desired arguments are in an extension of the abstract argumentation framework by establishing that:
\begin{itemize}
\item the attacks from arguments that are not in the extension to the arguments that are in the extension that are not defended by any unattacked arguments in the extension do not succeed, as given in Case 1;
\item the attacks from arguments that are in the extension to the arguments that are not in the extension always succeed, as given in Case 2; and
\item the attacks from arguments that are not in the extension to the arguments that are in the extension that are defended by any unattacked arguments in the extension may or may not succeed, as given in Case 3.
\end{itemize}

A worked example of how preferences are computed using the above three cases is as follows:
\begin{example}
	Let there be the abstract argumentation framework $(\mathcal{A}, \mathcal{R})$ of Figure~\ref{fig:abst_arg_graph1}, where $\mathcal{A} = \lbrace A, B, C , D, E \rbrace$ and $\mathcal{R} = \lbrace (A, B), (C, B), (C, D),$ $(D, C), (D, E) \rbrace$. We consider the conflict-free extensions $\mathcal{E}_1 = \{A,C,E\}$, and $\mathcal{E}_2 =\{A, D\}$ for computing preferences. For the extension $\mathcal{E}_1 = \{A,C,E\}$, we have the following preferences for each case:
	\begin{itemize}
		\item \textbf{Case 1:} $(C>D)$
		\item \textbf{Case 2:} $((A>B) \vee (A=B))  \wedge ((C>B) \vee (C=B))$
		\item \textbf{Case 3:} $ (E>D) \vee (E=D) \vee (D>E)$
	\end{itemize}
	
	Combining the preferences from the three cases we get $(C>D) \wedge (((A>B) \vee (A=B))  \wedge ((C>B) \vee (C=B)) ) \wedge ((E>D) \vee (E=D) \vee (D>E))$, which gives us the following sets of preferences:
	  \begin{gather*}
	  	\{C>D, A>B, C>B, E>D\} \\
	  	\{C>D, A>B, C>B, E=D\} \\
	  	\{C>D, A>B, C>B, D>E\} \\
	  	\{C>D, A>B, C=B, E>D\} \\
	  	\{C>D, A>B, C=B, E=D\} \\
	  	\{C>D, A>B, C=B, D>E\} \\ 
	    \{C>D, A=B, C>B, E>D\} \\
	    \{C>D, A=B, C>B, E=D\} \\
	    \{C>D, A=B, C>B, D>E\} \\ 
	    \{C>D, A=B, C=B, E>D\} \\
	    \{C>D, A=B, C=B, E=D\} \\
	    \{C>D, A=B, C=B, D>E\} 	  
	  \end{gather*}

\noindent For the extension $\mathcal{E}_2 = \{A,D\}$, we have the following preferences for each case:
		\begin{itemize}
		\item \textbf{Case 1:} $(D>C)$
		\item \textbf{Case 2:} $((A>B) \vee (A=B)) \wedge ( (D>E) \vee (D=E) )$
		\item \textbf{Case 3:} $\emptyset$ 
	\end{itemize}

	Combining the preferences from the three cases we get $(D>C) \wedge ( ((A>B) \vee (A=B)) \wedge ( (D>E) \vee (D=E)) ) \wedge \emptyset$, which gives us the following sets of preferences:
	\begin{gather*}
	 \{ D>C, A>B, D>E \}\\
	 \{ D>C, A>B, D=E \}\\
	\{ D>C, A=B, D>E \}\\
	\{ D>C, A=B, D=E \} 
	\end{gather*}
\end{example}

\subsection{Algorithms for Computing Preferences}
\label{sec:alg-comp-prefs}

As stated previously, in our research study, we consider the inverse of the standard reasoning problem. We now state the problem of computing preferences precisely as follows. 

\begin{problem}
\label{prob:computing-preferences}
    Given an abstract argumentation framework $AAF$ and a single set of justified arguments, i.e., an extension $\mathcal{E}$ under a given semantics (grounded, preferred or stable), we compute what the possible preferences over arguments are, i.e., a set of sets of preferences $\mathit{PrefSet}$, such that each set of preferences $\mathit{Prefs} \in \mathit{PrefSet}$ when applied to the AAF results in the single input extension $\mathcal{E}$ under a given semantics (grounded, preferred or stable)\footnote{Please note, the application of preferences on an AAF would restrict the result to a single extension $\mathcal{E}$ under multi-extension semantics (i.e., preferred and stable) and single extension semantics (i.e., grounded).}.
\end{problem}

To solve Problem~\ref{prob:computing-preferences}, we present Algorithm~\ref{alg:computing-preferences} that performs the computation of preferences over arguments with the help of Algorithms~\ref{alg:computing-preferences-case1}, \ref{alg:computing-preferences-case2}, and \ref{alg:computing-preferences-case3}. We will now present and describe all Algorithms~\ref{alg:computing-preferences}, \ref{alg:computing-preferences-case1},  \ref{alg:computing-preferences-case2}, and \ref{alg:computing-preferences-case3}. Algorithm~\ref{alg:computing-preferences} exhaustively computes all possible sets of preferences for a given input extension (consisting of conflict-free arguments) in an abstract argumentation framework (AAF) using the above three cases. The input of Algorithm $1$ is a tuple $\langle AAF, \mathcal{E}\rangle$, where:
	\begin{itemize}
		\item Abstract argumentation framework $AAF = \langle \mathcal{A}, \mathcal{R} \rangle$, $\mathcal{A}$ denotes the set of all arguments in the $AAF$, and $\mathcal{R}$ denotes the attack relation between arguments.
		\item Extension $\mathcal{E}$ consists of a finite number of conflict-free arguments such that $\mathcal{E} \subseteq \mathcal{A}$.
	\end{itemize}
The algorithm computes and outputs a set consisting of finite sets of preferences, where each set of preferences is represented as $\mathit{Prefs} = \lbrace A > B, B = C, .... \rbrace$ such that $\lbrace A,B,C,... \rbrace \subseteq \mathcal{A}$. The following are the main steps in Algorithm~\ref{alg:computing-preferences}:
\begin{itemize}
	\item Line $2$: Invoke Algorithm~\ref{alg:computing-preferences-case1} with inputs $AAF$ and $\mathcal{E}$, to compute case $1$ set of preferences $\mathit{Prefs}$.
	\item Line $3$: Invoke Algorithm~\ref{alg:computing-preferences-case2} with inputs $AAF$, $\mathcal{E}$ and $\mathit{Prefs}$, to compute case $2$ preferences and combine them with case $1$ preferences. This results in $\mathit{PrefSet}$ which is a set of sets of preferences.
	\item Line $4$: Invoke Algorithm~\ref{alg:computing-preferences-case3} with inputs $AAF$, $\mathcal{E}$ and $\mathit{PrefSet}$, to compute case $3$ preferences and combine them with $\mathit{Prefs}$ and $\mathit{PrefSet}$. This results in an updated final $\mathit{PrefSet}$ which is a set of sets of preferences containing all three cases of preferences combined together.
	\item Line $5$: Return the final $\mathit{PrefSet}$.
\end{itemize}

\begin{algorithm}[h]
	\begin{algorithmic}[1]
		\Require $\mathit{AAF}$, an abstract argumentation framework
		\Require $\mathcal{E}$, an extension consisting of conflict-free arguments
		\Ensure $\mathit{PrefSet}$, the set of sets of all possible preferences
		\Function{ComputeAllPreferences}{$\mathit{AAF},\mathcal{E}$}
		\State $\mathit{Prefs}\gets$ ComputePreferences$_1(\mathit{AAF},\mathcal{E})$
		\State $\mathit{PrefSet}\gets$ ComputePreferences$_2(\mathit{AAF},\mathcal{E},\mathit{Prefs})$
		\State $\mathit{PrefSet}\gets$ ComputePreferences$_3(\mathit{AAF},\mathcal{E},\mathit{PrefSet})$
		\State \Return $\mathit{PrefSet}$
		\EndFunction
	\end{algorithmic}
	\caption{Compute all preferences}
	\label{alg:computing-preferences}
\end{algorithm}

% Algorithms ~\ref{alg:computing-preferences-case1},~\ref{alg:computing-preferences-case2},~\ref{alg:computing-preferences-case3} compute case $1$, case $2$ and case $3$ preferences respectively. 

The following are the main steps in Algorithm~\ref{alg:computing-preferences-case1}:
\begin{itemize}
\item Line $3$: Iteratively pick a single argument $A$ from the extension $\mathcal{E}$.
	\item Line $4$: Find all arguments $B$ that attack $A$.
	\item Lines $5-11$: For each $B$, if there is no unattacked argument $C$ (where $C \neq A$ and $C \in \mathcal{E}$) that attacks $B$, then compute each preference of the form $A > B$ and add it to the set of preferences $\mathit{Prefs}$.
\end{itemize}

\begin{algorithm}[h]
	\begin{algorithmic}[1]
		\Require $\mathit{AAF}$, an abstract argumentation framework
		\Require $\mathcal{E}$, an extension consisting of conflict-free arguments
		\Ensure $\mathit{Prefs}$, a set of preferences
		\Function{ComputePreferences$_1$}{$\mathit{AAF},\mathcal{E}$}
		\State $\mathit{Prefs}\gets\emptyset$
		\ForAll{$A \in\mathcal{E}$}  
		\State $\mathit{Attackers}\gets\{B\;|\;(B,A)\in\mathcal{R}\}$ \Comment get all attackers of $A$
		\ForAll{$B\in\mathit{Attackers}$}
		\State $\mathit{Defenders}\gets\{C\;|\;C\neq A, C\in\mathcal{E}, (C,B)\in\mathcal{R}, \nexists X \in\mathcal{A} \; s.t. \; (X,C)\in\mathcal{R}\}$ \Comment $C$ attacks $B$ \& defends $A$
		\If{$\mathit{Defenders}=\emptyset$} \Comment if $B$ not attacked by any $C$
		\State $\mathit{Prefs}\gets\mathit{Prefs}\cup\{A>B\}$ \Comment add preference $A>B$
		\EndIf
		\EndFor
		\EndFor
		\State \Return $\mathit{Prefs}$
		\EndFunction
	\end{algorithmic}
	\caption{Compute preferences (Case 1)}
	\label{alg:computing-preferences-case1}	
\end{algorithm}

The following are the main steps in Algorithm~\ref{alg:computing-preferences-case2}:
\begin{itemize}
	\item Line $4$: Iteratively pick a single argument $A$ from the extension $\mathcal{E}$.
	\item Line $5$: Find all arguments $B$ that $A$ attacks.
	\item Lines $6-13$: For all arguments $B$ attacked by $A$, compute preferences of the form $A>B$ and $A=B$, and add each preference relation to a different set of preferences, as per lines $8$ and $9$. 
\end{itemize}

\begin{algorithm}[h]
	\begin{algorithmic}[1]
		\Require $\mathit{AAF}$, an abstract argumentation framework
		\Require $\mathcal{E}$, an extension consisting of conflict free arguments
		\Require $\mathit{Prefs}$, a set of preferences
		\Ensure $\mathit{PrefSet}$, a set of sets of preferences
		\Function{ComputePreferences$_2$}{$\mathit{AAF},\mathcal{E}, \mathit{Prefs}$}
		\State $\mathit{PrefSet}\gets \{\mathit{Prefs}\}$
		\State $\mathit{PrefSet'}\gets\emptyset$
		\ForAll{$A\in\mathcal{E}$}  
		\State $\mathit{Attacked}\gets\{B\;|\;(A,B)\in\mathcal{R} \wedge (B,A)\notin\mathcal{R}\}$ \Comment get arguments $A$ attacks
		\ForAll{$B\in\mathit{Attacked}$} \Comment for all $B$ attacked by $A$
		\ForAll{$\mathit{Prefs}\in\mathit{PrefSet}$} \Comment for all sets of preferences $\mathit{Prefs}$
		\State $\mathit{PrefSet'} \gets \mathit{PrefSet'} \cup \{\mathit{Prefs} \cup \{A > B\}\}$ \Comment add $\mathit{Prefs}\cup \{A>B\}$ 		
		\State $\mathit{PrefSet'} \gets \mathit{PrefSet'} \cup \{\mathit{Prefs} \cup \{A = B\}\}$ \Comment add $\mathit{Prefs}\cup \{A=B\}$ 
		\EndFor
		\State $\mathit{PrefSet} \gets \mathit{PrefSet'}$
		\State $\mathit{PrefSet'}\gets\emptyset$	
		\EndFor
		\EndFor
		\State \Return $\mathit{PrefSet}$
		\EndFunction
	\end{algorithmic}
	\caption{Compute preferences (Case 2)}
\label{alg:computing-preferences-case2}		
\end{algorithm}

The following are the main steps in Algorithm~\ref{alg:computing-preferences-case3}:
\begin{itemize}
	\item Line $3$: Iteratively pick a single argument $A$ from the extension $\mathcal{E}$.
	\item Line $4$: Find all arguments $B$ that attack $A$.
	\item Lines $5-16$: For each $B$, if there is an unattacked argument $C$ (where $C \neq A$ and $C \in \mathcal{E}$) that attacks $B$, then compute preferences of the form $A>B$, $A=B$ and $B>A$, and add each preference relation to a different set of preferences, as per lines $9-11$.
\end{itemize}

\begin{algorithm}[h]
	\begin{algorithmic}[1]
		\Require $\mathit{AAF}$, an abstract argumentation framework
		\Require $\mathcal{E}$, an extension consisting of conflict free arguments
		\Require $\mathit{PrefSet}$, a set of sets of preferences
		\Ensure $\mathit{PrefSet}$, an updated set of sets of preferences
		\Function{ComputePreferences$_3$}{$\mathit{AAF},\mathcal{E},\mathit{PrefSet}$}
		\State $\mathit{PrefSet'}\gets\emptyset$
		\ForAll{$A\in\mathcal{E}$}  
		\State $\mathit{Attackers}\gets\{B\;|\;(B,A)\in\mathcal{R}\}$ \Comment get all attackers of $A$
		\ForAll{$B\in\mathit{Attackers}$}
		\State $\mathit{Defenders}\gets\{C\;|\;C\neq A, C\in\mathcal{E}, (C,B)\in\mathcal{R}, \nexists X \in\mathcal{A} \; s.t. \; (X,C)\in\mathcal{R}\}$ \Comment $C$ attacks $B$ \& defends $A$
		\If{$\mathit{Defenders}\neq\emptyset$}
		\ForAll{$\mathit{Prefs} \in \mathit{PrefSet}$} \Comment for all sets of preferences $\mathit{Prefs}$
		\State $\mathit{PrefSet'} \gets \mathit{PrefSet'} \cup \{\mathit{Prefs} \cup \{A > B\}\}$ \Comment add $\mathit{Prefs}\cup \{A>B\}$ 		 
		\State $\mathit{PrefSet'} \gets \mathit{PrefSet'} \cup \{\mathit{Prefs} \cup \{A = B\}\}$ \Comment add $\mathit{Prefs}\cup \{A=B\}$ 		
		\State $\mathit{PrefSet'} \gets \mathit{PrefSet'} \cup \{\mathit{Prefs} \cup \{B > A\}\}$ \Comment add $\mathit{Prefs}\cup \{B>A\}$ 		
		\EndFor
		\State $\mathit{PrefSet} \gets \mathit{PrefSet'}$
		\State $\mathit{PrefSet'}\gets\emptyset$		
		\EndIf
		\EndFor
		\EndFor
		\State \Return $\mathit{PrefSet}$
		\EndFunction
	\end{algorithmic}
	\caption{Compute preferences (Case 3)}
	\label{alg:computing-preferences-case3}	
\end{algorithm}

We establish that our approach is sound (that is, all its outputs are correct) and complete (that is, it outputs all possible solutions). We start with its soundness:

\begin{theorem}{(Soundness):}
\label{theorem:soundness}
	Algorithm~\ref{alg:computing-preferences} is sound in that given an abstract argumentation framework  $\mathit{AAF}$ and an extension $\mathcal{E}$ as input, every output preference set $\mathit{Prefs}\in\mathit{PrefSet}$, when applied to the $\mathit{AAF}$ results in the input $\mathcal{E}$ (under a given semantics).
\end{theorem}

\begin{proof}
We prove this by exploring all cases and how these are handled by algorithms $2$-$4$.
Each set of preferences computed for each subset of arguments $\alpha, \beta, \gamma \subseteq \mathcal{A}$  is such that $\alpha, \gamma \subseteq \mathcal{E}, \beta\cap\mathcal{E} = \emptyset$. We proceed to show how each of the auxiliary algorithms $2$-$4$ help us achieve this.

Algorithm~\ref{alg:computing-preferences-case1} computing each case $1$ preference of the form $A> B, A\in\mathcal{E}, B\in\beta, (B,A)\in \mathcal{R}$ ensures that the following holds:
\begin{enumerate}
	\item There is no $C\in\mathcal{E}, C\neq A$ such that $(C,B)\in\mathcal{R}$ (lines 6-7).
	\item $A \in \mathcal{E}$ since $A$ is preferred to its attacking argument $B$, which invalidates the attack $(B,A)\in\mathcal{R}$.
	\item Since the input extension $\mathcal{E}$ consists of conflict free arguments, if $A\in\mathcal{E}$ then its attacking argument $B\not\in\mathcal{E}$. This supports that $\beta\cap\mathcal{E}=\emptyset$.
\end{enumerate}
Algorithm~\ref{alg:computing-preferences-case2} computing each case $2$ preferences of the form $A>B, A=B, A\in\mathcal{E}, B\in\beta, (A,B)\in\mathcal{R}, (B,A)\not\in\mathcal{R}$ ensures the following holds:
\begin{enumerate}
	\item Since $A$ attacks $B$ and $B$ does not attack $A$, we have two different
	preferences between $A$ and $B$, namely, $A > B, A = B$. Therefore $A\in\mathcal{E}$ 
	with respect to each of these preferences.
	\item Preferences $A > B, A = B$ will be in different preference sets, as per lines 8 and 9. 
	We will have $\mathit{Prefs}_1\gets\mathit{Prefs}\cup \{A > B\}$ and
	$\mathit{Prefs}_2\gets\mathit{Prefs}\cup \{A = B\}$, where $\mathit{Prefs}$ consists of preferences of case $1$.
\end{enumerate}
Algorithm~\ref{alg:computing-preferences-case3} computing each case $3$ preferences of the form $A > B, A = B, B > A, A\in\alpha, B\in\beta, C\in\gamma, (B,A)\in\mathcal{R}, (C,B)\in\mathcal{R}$ ensures the following holds:
\begin{enumerate}
	\item  Since $C$ defends $A$ from the attack of $B$, we have three different preferences between $A$ and $B$, namely, $A>B$, $A=B$ and $B>A$. Therefore $A \in \mathcal{E}$ with respect to each of these preferences. 
	\item Preferences $A > B, A = B, B > A$ will be in different preference sets, as per lines 9, 10 and 11. 
	We will have $\mathit{Prefs}_1\gets\mathit{Prefs}\cup \{A > B\}$, $\mathit{Prefs}_2\gets\mathit{Prefs}\cup \{A = B\}$ and $\mathit{Prefs}_3\gets\mathit{Prefs}\cup \{B > A\}$, where $\mathit{Prefs}$ consists of preferences of cases $1$ and $2$.
\end{enumerate}
\end{proof}

\begin{theorem}{(Completeness):}
	Algorithm~\ref{alg:computing-preferences} is complete in that given an abstract argumentation framework  $\mathit{AAF}$ and an extension $\mathcal{E}$ as input, if there is a preference set $\mathit{Prefs}\in\mathit{PrefSet}$ which when applied to the $\mathit{AAF}$ results in the input $\mathcal{E}$ (under a given semantics), then algorithm $1$ will find it. 
\end{theorem}

\begin{proof}
	
Similar to above, we prove this by exploring all cases and how these are handled by algorithms 2-4. We find all sets of preferences computed for each subset of arguments $\alpha, \beta, \gamma \subseteq \mathcal{A}, \alpha, \gamma \subseteq \mathcal{E}, \beta\cap\mathcal{E} = \emptyset$. We proceed to show how each of the auxiliary algorithms 2-4 help us achieve this.

Algorithm~\ref{alg:computing-preferences-case1} computes all case $1$ preferences of the form $A> B, A\in\mathcal{E}, B\in\beta, (B,A)\in \mathcal{R}$. Lines $3$-$11$ exhaustively search for $A \in \mathcal{E}$ for which there is an attacker $B$ (not attacked by any $C\neq A$). If there are such $A, B\in\mathcal{A}$, the algorithm will find them and add $A>B$ to a set of preferences.

Algorithm~\ref{alg:computing-preferences-case2} computes all case $2$ preferences of the form  $A>B, A=B, B\in\beta, A\in\mathcal{E}, (A,B)\in\mathcal{R}, (B,A)\not\in\mathcal{R}$. Lines $4$-$14$ exhaustively search for $A \in \mathcal{E}$ for which there is an attacked argument $B$ and $B$ does not attack $A$. If there are such $A, B\in\mathcal{A}$, the algorithm will find them and add each $A>B, A=B$ to a different set of preferences.

Algorithm~\ref{alg:computing-preferences-case3} computes all case $3$ preferences of the form  $A>B, A=B, B>A, A\in\alpha, B\in\beta, C\in\gamma, (B,A)\in\mathcal{R}, (C,B)\in\mathcal{R}$. Lines 3-17 exhaustively search for $A \in \mathcal{E}$ for which there is an attacker $B$ and there is a defender $C$ that attacks $B$. If there are such $A, B, C\in\mathcal{A}$, the algorithm will find them and add each $A>B, A=B, B>A$ to a different set of preferences.
\end{proof}

After having proved the soundness and completeness of Algorithm~\ref{alg:computing-preferences}, we establish its termination.

\begin{theorem}{(Termination):}
	Given an abstract argumentation framework  $\mathit{AAF}$ and an extension $\mathcal{E}$ as input, Algorithm~\ref{alg:computing-preferences} always terminates.
\end{theorem}
\begin{proof}
Algorithm~\ref{alg:computing-preferences} invokes Algorithms~\ref{alg:computing-preferences-case1}, \ref{alg:computing-preferences-case2} and \ref{alg:computing-preferences-case3} in lines $2$-$4$ to compute a set of all sets of preferences. To prove Algorithm~\ref{alg:computing-preferences} terminates we consider the termination of Algorithms~\ref{alg:computing-preferences-case1},~\ref{alg:computing-preferences-case2} and~\ref{alg:computing-preferences-case3} individually. 
Since we assume that both the input abstract argumentation framework  $\mathit{AAF}$ and an extension $\mathcal{E}$ are finite, therefore the \textit{for loop} which iterates over all the elements of the extension in Algorithms~\ref{alg:computing-preferences-case1},~\ref{alg:computing-preferences-case2} and~\ref{alg:computing-preferences-case3} will always terminate. The rest of the proof explores each algorithm in turn.

In Algorithm~\ref{alg:computing-preferences-case1}, since the set of $\mathit{Attackers}$, i.e., all the arguments that attack $A$, is finite therefore the \textit{for loop} in lines $5$-$10$ will always terminate. In Algorithm~\ref{alg:computing-preferences-case2}, since the set of $\mathit{Attacked}$, i.e., all the arguments that are attacked by $A$, is finite therefore the \textit{for loop} in lines $6$-$13$ will always terminate. Furthermore, since the set of sets of preferences $\mathit{PrefSet}$ is finite, therefore the \textit{for loop} in lines $7$-$10$ will always terminate.
In Algorithm~\ref{alg:computing-preferences-case3}, since the set of $\mathit{Attackers}$, i.e., all the arguments that attack $A$ is finite, therefore the \textit{for loop} in lines $5$-$16$ will always terminate. Furthermore, since the set of sets of preferences $\mathit{PrefSet}$ is finite, so the \textit{for loop} in lines $8$-$12$ will always terminate.

Thus, we have proved that Algorithms~\ref{alg:computing-preferences-case1},~\ref{alg:computing-preferences-case2} and~\ref{alg:computing-preferences-case3} terminate. Therefore, Algorithm~\ref{alg:computing-preferences} that invokes Algorithms~\ref{alg:computing-preferences-case1},~\ref{alg:computing-preferences-case2} and~\ref{alg:computing-preferences-case3} always terminates.
\end{proof}

\subsubsection{Algorithms for Filtering Preferences}
\label{sec:alg-filter-prefs}

Additionally, we present Algorithm~\ref{alg:computing-preferences-unique} to compute the unique preferences for an extension in comparison to another extension, and Algorithm~\ref{alg:computing-preferences-common} to compute the common preferences for any two extensions. Algorithm~\ref{alg:computing-preferences-unique} takes as input two different set of of sets of preferences $\mathit{PrefSet_1}$ and $\mathit{PrefSet_2}$ for two different extensions, and computes the set of unique preferences that do not overlap between $\mathit{PrefSet_1}$ and $\mathit{PrefSet_2}$. The following are the main steps in Algorithm~\ref{alg:computing-preferences-unique}:
\begin{itemize}
	\item Line $2$: Iterate over each element $\mathit{Prefs_1}$ in the set of sets of preferences $\mathit{PrefSet_1}$.
	\item Line $3$: Iterate over each element $p$ in the set of preferences $\mathit{Prefs_1}$.
	\item Lines $4$-$6$: Check the condition if there is no set of preferences $\mathit{Prefs_2}$ in the set of sets of preferences $\mathit{PrefSet_2}$, which consists of $p$, then add $p$ to the unique set of preferences $\mathit{UniquePrefs}$.
	\item Line $5$: Return the unique set of preferences $\mathit{UniquePrefs}$.
\end{itemize}
\begin{algorithm}[!h]
	\caption{Algorithm for Computing Unique Preferences}
	\label{alg:computing-preferences-unique}	
	\begin{algorithmic}[1]
		\Require $\mathit{PrefSet_1}$, the set of sets of preferences for first extension.
		\Require $\mathit{PrefSet_2}$, the set of sets of preferences for second extension.
		\Ensure $\mathit{UniquePrefs}$, unique preferences for first extension.
		\Function{ComputeUniquePreferences}{$\mathit{PrefSet1}$, $\mathit{PrefSet2}$}
		\ForAll{$\mathit{Prefs_1} \in \mathit{PrefSet_1}$} 
		\ForAll{$p \in \mathit{Prefs_1}$} 
		\If{$\nexists \mathit{Prefs_2} \in \mathit{PrefSet_2} \; s.t. \; p \in \mathit{Prefs_2}$}
		\State $\mathit{UniquePrefs} \gets \mathit{UniquePrefs} \cup p$
		\EndIf
		\EndFor	
		\EndFor
		\State \Return $\mathit{UniquePrefs}$ 
		\EndFunction
	\end{algorithmic}
\end{algorithm}

Algorithm~\ref{alg:computing-preferences-common} takes as input two different set of sets of preferences $\mathit{PrefSet_1}$ and $\mathit{PrefSet_2}$ for two different extensions, and computes the set of common preferences that overlap between $\mathit{PrefSet_1}$ and $\mathit{PrefSet_2}$. The following are the main steps in Algorithm~\ref{alg:computing-preferences-common}:
\begin{itemize}
	\item Line $2$: Iterate over each element $\mathit{Prefs_1}$ in the set of sets of preferences $\mathit{PrefSet_1}$.
	\item Line $3$: Iterate over each element $p$ in the set of preferences $\mathit{Prefs_1}$.
	\item Lines $4$-$6$: Check the condition if there is a set of preferences $\mathit{Prefs_2}$ in the set of sets of preferences $\mathit{PrefSet_2}$, which consists of $p$, then add $p$ to the common set of preferences $\mathit{CommonPrefs}$.
	\item Line $5$: Return the common set of preferences $\mathit{CommonPrefs}$.
\end{itemize}
\begin{algorithm}[!h]
	\caption{Algorithm for Computing Common Preferences}
	\label{alg:computing-preferences-common}	
	\begin{algorithmic}[1]
		\Require $\mathit{PrefSet_1}$, the set of sets of preferences for first extension.
		\Require $\mathit{PrefSet_2}$, the set of sets of preferences for second extension.
		\Ensure $\mathit{CommonPrefs}$, common preferences for both extensions.
		\Function{ComputeCommonPreferences}{$\mathit{PrefSet_1}$, $\mathit{PrefSet_2}$}
		\ForAll{$\mathit{Prefs_1} \in \mathit{PrefSet_1}$} 
		\ForAll{$p \in \mathit{Prefs_1}$} 
		\If{$\exists \mathit{Prefs_2} \in \mathit{PrefSet_2} \; s.t. \; p \in \mathit{Prefs_2}$}
		\State $\mathit{CommonPrefs} \gets \mathit{CommonPrefs} \cup p$
		\EndIf
		\EndFor	
		\EndFor
		\State \Return $\mathit{CommonPrefs}$ 
		\EndFunction
	\end{algorithmic}
\end{algorithm}

\subsubsection{Illustrative Example}
\label{sec:evaluation}
In this section, we present an illustrative example to demonstrate how Algorithm~\ref{alg:computing-preferences} works. Suppose we have an input abstract argumentation framework $(\mathcal{A}, \mathcal{R})$ shown in Figure~\ref{fig:abst_arg_graph1}, where $\mathcal{A} = \lbrace A, B, C , D, E \rbrace$ and $\mathcal{R} = \lbrace (A, B), (C, B)$, $(C, D), (D, C), (D, E) \rbrace$. We consider the conflict-free extension $\mathcal{E}_1 = \{A,C,E\}$ for computing preferences. Table~\ref{table:extension1-demonstration} shows the preferences computed in lines $2$, $3$ and $4$  of Algorithm~\ref{alg:computing-preferences}. 
\begin{itemize}
\item On line $2$, Algorithm~\ref{alg:computing-preferences-case1} is invoked, which returns the set of case $1$ preferences $\{ C>D \}$. 
\item On line $3$, Algorithm~\ref{alg:computing-preferences-case2} is invoked, which returns a set of sets of preferences (cases $1$ and $2$ combined together) $\{ \{C>D, A>B, C>B\}$, $\{C>D, A>B, C=B\}$, $\{C>D, A=B, C>B\}$, $\{C>D, A=B, C=B\} \}$. 
\item Finally on line $4$, Algorithm~\ref{alg:computing-preferences-case3} is invoked, which returns a set of sets of preferences (cases $1$, $2$ and $3$ combined together) $\{ \{C>D, A>B, C>B, E>D\}$, $\{C>D, A>B, C>B, E=D\}$, $\{C>D, A>B, C>B, D>E\}$, $\{C>D, A>B, C=B, E>D\}$, $\{C>D, A>B, C=B, E=D\}$, $\{C>D, A>B, C=B, D>E\}$, $\{C>D, A=B, C>B, E>D\}$, $\{C>D, A=B, C>B, E=D\}$, $\{C>D, A=B, C>B, D>E\}$, $\{C>D, A=B, C=B, E>D\}$, $\{C>D, A=B, C=B, E=D\}$, $\{C>D, A=B, C=B, D>E\} \}$.
\end{itemize}
\begin{table}
	\caption{Computing Preferences for Extension $\{A,C,E\}$}
	\label{table:extension1-demonstration}
	\centering
\begin{tabular}{|c|p{60mm}|}
	\hline
	Line No.& Preference Sets\\\hline
	$2$ & $\{ C>D \}$\\ \hline
	$3$ & $\{ \{C>D, A>B, C>B\}, \newline \{C>D, A>B, C=B\}, \newline \{C>D, A=B, C>B\}, \newline \{C>D, A=B, C=B\} \}$
	\\ \hline
	$4$ & 
	$ \{ \{C>D, A>B, C>B, E>D\},$ \newline $\{C>D, A>B, C>B, E=D\},$ \newline $\{C>D, A>B, C>B, D>E\},$ 
	\newline
	$\{C>D, A>B, C=B, E>D\},$ \newline $\{C>D, A>B, C=B, E=D\},$ \newline $\{C>D, A>B, C=B, D>E\},$ 
	\newline
	$\{C>D, A=B, C>B, E>D\},$ \newline $\{C>D, A=B, C>B, E=D\},$ \newline $\{C>D, A=B, C>B, D>E\},$ 
	\newline
	$\{C>D, A=B, C=B, E>D\},$ \newline $\{C>D, A=B, C=B, E=D\},$ \newline $\{C>D, A=B, C=B, D>E\} \}$ 
	\\
	\hline
\end{tabular}
\end{table}
Table~\ref{table:extensions-preferences} presents the sets of preferences for the two preferred extensions $\{A,C,E\}$ and $\{A,D\}$ of the abstract argumentation framework given above and shown in Figure~\ref{fig:abst_arg_graph1}. The sets of preferences for all conflict-free extensions for this example abstract argumentation framework are shown in Table~\ref{table:conflictfree-preferences} in the Appendix.
\begin{table}[!h]
	\caption{Preferences for the Preferred extensions $\{A,C,E\}$ and $\{A,D\}$}
	\label{table:extensions-preferences}
 \centering\renewcommand\arraystretch{1.2}
	\begin{tabular}{|c|p{58mm}|p{20mm}|p{20mm}|}
		\hline
		Preferred Extensions & Preference Sets & Unique Preferences & Common Preferences \\ \hline
		$\{A,C,E\}$ &  
		$\{ \{C>D, A>B, C>B, E>D\},$ \newline $\{C>D, A>B, C>B, E=D\},$ \newline $\{C>D, A>B, C>B, D>E\},$ 
		\newline
		$\{C>D, A>B, C=B, E>D\},$ \newline $\{C>D, A>B, C=B, E=D\},$ \newline $\{C>D, A>B, C=B, D>E\},$ 
		\newline
		$\{C>D, A=B, C>B, E>D\},$ \newline $\{C>D, A=B, C>B, E=D\},$ \newline $\{C>D, A=B, C>B, D>E\},$ 
		\newline
		$\{C>D, A=B, C=B, E>D\},$ \newline $\{C>D, A=B, C=B, E=D\},$ \newline $\{C>D, A=B, C=B, D>E\} \}$ 
		& $C>D$\newline $E>D$ \newline $C>B$ \newline $C=B$ & $A>B$\newline $A=B$ \newline $D>E$\newline $D=E$
				\\
		\cline{1-3}
		$\{A,D\}$ & 	 $\{ \{ D>C, A>B, D>E \},$
		\newline $\{ D>C, A>B, D=E \},$
		\newline $\{ D>C, A=B, D>E \},$
		\newline $\{ D>C, A=B, D=E \} \}$ 
		& $D>C$ &
		\\
		\hline
	\end{tabular}
\end{table}

The unique preferences for an extension in comparison to another extension can be computed by Algorithm~\ref{alg:computing-preferences-unique}. By analysing the preference sets shown in Table~\ref{table:extensions-preferences}, we can identify the unique preferences for extension $\{A,C,E\}$, which are $C>D$, $E>D$, $C>B$ and $C=B$\footnote{This means it could be either $C>B$ or $C=B$.}, and the unique preferences for extension $\{A,D\}$, which is $D>C$. Since, at least one unique preference for each extension is in its corresponding preference set, therefore it can be concluded that if we evaluate the example abstract argumentation framework given in Figure~\ref{fig:abst_arg_graph1} with a corresponding preference set of a given preferred extension, then the evaluation results in exactly the same preferred extension. 

Furthermore, we can identify preferences that are common to both extensions, which are $A>B$, $A=B$, $D>E$, $D=E$\footnote{This means it could be either $A>B$ or $A=B$, and similarly $D>E$ or $D=E$.}. The common preferences for any two extensions can be computed by Algorithm~\ref{alg:computing-preferences-common}. It is interesting to note that, extension $\{A,C,E\}$ can have preferences $D>E$ and $D=E$, considering $D$ is not present in the extension.  
It can be concluded that if we evaluate the example abstract argumentation framework given in Figure~\ref{fig:abst_arg_graph1}, then we get both preferred extensions with the following preference sets: $\{A>B, D>E\}$, $\{A>B, D=E\}$, $\{A=B, D>E\}$ and $\{A=B, D=E\}$. 

\subsection{An Approximate Algorithm for Computing Preferences}
While Algorithm~\ref{alg:computing-preferences} can compute a set of all sets of preferences using the three cases described earlier, the number of the possible sets of preferences increases exponentially with regards to the number of arguments within the abstract argumentation framework, resulting in exponential time complexity. This is impractical for a large set of arguments, and in this section, we describe an approximate method for computing a set of preferences.

We now present and describe  Algorithm~\ref{alg:approximate-computing-preferences} that approximately computes a possible set of preferences for a given input extension (consisting of conflict-free arguments) in an abstract argumentation framework (AAF) using the three cases described earlier. The input of Algorithm~\ref{alg:approximate-computing-preferences} is a tuple $\langle AAF, \mathcal{E}\rangle$, where:
	\begin{itemize}
		\item Abstract argumentation framework $AAF = \langle \mathcal{A}, \mathcal{R} \rangle$, $\mathcal{A}$ denotes the set of all arguments in the $AAF$, and $\mathcal{R}$ denotes the attack relation between arguments.
		\item Extension $\mathcal{E}$ consists of a finite number of conflict-free arguments such that $\mathcal{E} \subseteq \mathcal{A}$.
	\end{itemize}
The algorithm computes and outputs a finite set of preferences, which is represented as $\mathit{Prefs} = \lbrace A > B, C > B, .... \rbrace$ such that $\lbrace A,B,C,... \rbrace \subseteq \mathcal{A}$. 

\begin{algorithm}[h]
	\begin{algorithmic}[1]
		\Require $\mathit{AAF}$, an abstract argumentation framework
		\Require $\mathcal{E}$, an extension consisting of conflict-free arguments
		\Ensure $\mathit{Prefs}$, an approximate set of possible preferences
		\Function{ComputePreferencesApproximately}{$\mathit{AAF},\mathcal{E}$}
	    \State $\mathit{Prefs}\gets\emptyset$
		\ForAll{$A \in\mathcal{E}$}  
		\State $\mathit{Attackers}\gets\{B\;|\;(B,A)\in\mathcal{R}\}$ \Comment get all attackers of $A$
		\ForAll{$B\in\mathit{Attackers}$}
		\State $\mathit{Defenders}\gets\{C\;|\;C\neq A, C\in\mathcal{E}, (C,B)\in\mathcal{R}, \nexists X \in\mathcal{A} \; s.t. \; (X,C)\in\mathcal{R}\}$ \Comment $C$ attacks $B$ \& defends $A$
		\If{$\mathit{Defenders}=\emptyset$} \Comment if $B$ not attacked by any $C$
		\State $\mathit{Prefs}\gets\mathit{Prefs}\cup\{A>B\}$ \Comment add preference $A>B$
		\EndIf
		\EndFor
		\EndFor
%Case 2 preferences
		\ForAll{$A\in\mathcal{E}$}  
		\State $\mathit{Attacked}\gets\{B\;|\;(A,B)\in\mathcal{R} \wedge (B,A)\notin\mathcal{R}\}$ \Comment get arguments $A$ attacks
		\ForAll{$B\in\mathit{Attacked}$} \Comment for all $B$ attacked by $A$
        \State Generate a random preference $p$ such that $p \in \{A>B,\: A=B\}$
		\State $\mathit{Prefs}\gets\mathit{Prefs}\cup p$ \Comment add preference $p$
		\EndFor
		\EndFor
%Case 3 preferences
		\ForAll{$A\in\mathcal{E}$}  
		\State $\mathit{Attackers}\gets\{B\;|\;(B,A)\in\mathcal{R}\}$ \Comment get all attackers of $A$
		\ForAll{$B\in\mathit{Attackers}$}
		\State $\mathit{Defenders}\gets\{C\;|\;C\neq A, C\in\mathcal{E}, (C,B)\in\mathcal{R}, \nexists X \in\mathcal{A} \; s.t. \; (X,C)\in\mathcal{R}\}$ \Comment $C$ attacks $B$ \& defends $A$
		\If{$\mathit{Defenders}\neq\emptyset$} \Comment if $B$ is attacked by any $C$
        \State Generate a random preference $p$ such that $p \in \{A>B,\: A=B,\: B>A\}$
		\State $\mathit{Prefs}\gets\mathit{Prefs}\cup p$ \Comment add preference $p$
		\EndIf
		\EndFor
		\EndFor
		\State \Return $\mathit{Prefs}$
		\EndFunction
	\end{algorithmic}
	\caption{Approximately compute preferences}
	\label{alg:approximate-computing-preferences}
\end{algorithm}

The following are the main steps in Algorithm~\ref{alg:approximate-computing-preferences}:
\begin{itemize}
\item Lines $3-11$: Compute all Case $1$ preferences.
    \begin{itemize}
        \item Line $3$: Iteratively pick a single argument $A$ from the extension $\mathcal{E}$.
	    \item Line $4$: Find all arguments $B$ that attack $A$.
	    \item Lines $5-10$: For each $B$, if there is no unattacked argument $C$ (where $C \neq A$ and $C \in \mathcal{E}$) that attacks $B$, then compute each preference of the form $A > B$ and add it to the set of preferences $\mathit{Prefs}$.
\end{itemize}
\item Lines $12-18$: Compute Case $2$ preferences and combine them with Case $1$ preferences.
    \begin{itemize}
	   \item Line $12$: Iteratively pick a single argument $A$ from the extension $\mathcal{E}$.
	   \item Line $13$: Find all argument $B$ that $A$ attacks.
	   \item Lines $14-17$: For all arguments $B$ attacked by $A$, generate a random preference $p$ such that $p \in \{A>B,\: A=B\}$, and add it to the set of preferences $\mathit{Prefs}$, as per lines $15-16$. 
    \end{itemize}
\item Lines $19-28$: Compute Case $3$ preferences and combine them with Case $2$ and Case $3$ preferences.
    \begin{itemize}
	   \item Line $19$: Iteratively pick a single argument $A$ from the extension $\mathcal{E}$.
	   \item Line $20$: Find all arguments $B$ that attack $A$.
	   \item Lines $21-27$: For each $B$, if there is an unattacked argument $C$ (where $C \neq A$ and $C \in \mathcal{E}$) that attacks $B$, then generate a random preference $p$ such that $p \in \{A>B,\: A=B,\: B>A\}$, and add it to the set of preferences $\mathit{Prefs}$, as per lines $24-25$.
    \end{itemize}
\item Line $29$: Return the final set of preferences $\mathit{Prefs}$.
\end{itemize}

We establish that our approach is sound (that is, its output is correct).

\begin{theorem}{(Soundness):}
\label{theorem:soundness}
	Algorithm~\ref{alg:approximate-computing-preferences} is sound in that given an abstract argumentation framework  $\mathit{AAF}$ and an extension $\mathcal{E}$ as input, the output preference set $\mathit{Prefs}$, when applied to the $\mathit{AAF}$ results in the input $\mathcal{E}$ (under a given semantics).
\end{theorem}

\begin{proof}
We prove this by exploring all cases and how these are handled by the algorithm.
The set of preferences computed for each subset of arguments $\alpha, \beta, \gamma \subseteq \mathcal{A}$  is such that $\alpha, \gamma \subseteq \mathcal{E}, \beta\cap\mathcal{E} = \emptyset$. 
%We proceed to show how each of the auxiliary algorithms $2$-$4$ help us achieve this.

\noindent Lines $3-11$ computing each case $1$ preference of the form $A> B, A\in\mathcal{E}, B\in\beta, (B,A)\in \mathcal{R}$ ensure that the following holds:
\begin{enumerate}
	\item There is no $C\in\mathcal{E}, C\neq A$ such that $(C,B)\in\mathcal{R}$.
	\item $A \in \mathcal{E}$ since $A$ is preferred to its attacking argument $B$, which invalidates the attack $(B,A)\in\mathcal{R}$.
	\item Since the input extension $\mathcal{E}$ consists of conflict free arguments, if $A\in\mathcal{E}$ then its attacking argument $B\not\in\mathcal{E}$. This supports that $\beta\cap\mathcal{E}=\emptyset$.
\end{enumerate}
Lines $12-18$ computing a random case $2$ preference $p$ of the form $p\in \{A>B, A=B\}, A\in\mathcal{E}, B\in\beta, (A,B)\in\mathcal{R}, (B,A)\not\in\mathcal{R}$ ensure the following holds:
\begin{enumerate}
	\item Since $A$ attacks $B$ and $B$ does not attack $A$, we have two different
	preferences between $A$ and $B$, namely, $A > B, A = B$. Therefore $A\in\mathcal{E}$ 
	with respect to each of these preferences.
	\item A randomly generated preference $p \in \{A > B, A = B\}$ will be added to the preference set $\mathit{Prefs}$, as per lines 15 and 16. 
	We will have $\mathit{Prefs}\gets\mathit{Prefs}\cup p$, where $\mathit{Prefs}$ consists of preferences of case $1$.
\end{enumerate}
Lines $19-28$ computing a random case $3$ preference of the form $p \in \{A > B, A = B, B > A\}, A\in\alpha, B\in\beta, C\in\gamma, (B,A)\in\mathcal{R}, (C,B)\in\mathcal{R}$ ensure the following holds:
\begin{enumerate}
	\item  Since $C$ defends $A$ from the attack of $B$, we have three different preferences between $A$ and $B$, namely, $A>B$, $A=B$ and $B>A$. Therefore $A \in \mathcal{E}$ with respect to each of these preferences. 
	\item A randomly generated preference $p \in \{A > B, A = B, B > A\}$ will be added to the preference set $\mathit{Prefs}$, as per lines 24 and 25. 
	We will have $\mathit{Prefs}\gets\mathit{Prefs}\cup p$, where $\mathit{Prefs}$ consists of preferences of cases $1$ and $2$.
\end{enumerate}
\end{proof}

After having proved the soundness of Algorithm~\ref{alg:approximate-computing-preferences}, we establish its termination.

\begin{theorem}{(Termination):}
	Given an abstract argumentation framework  $\mathit{AAF}$ and an extension $\mathcal{E}$ as input, Algorithm~\ref{alg:approximate-computing-preferences} always terminates.
\end{theorem}
\begin{proof}
Algorithm~\ref{alg:approximate-computing-preferences} consists of three \textit{for loops} for computing each of the case $1$ (in lines $3-11$), case $2$ (in lines $12-18$), and case $3$ (in lines $19-28$) preferences for the output set of preferences. To prove Algorithm~\ref{alg:approximate-computing-preferences} terminates we consider the termination of each of the \textit{for loops} for the three cases individually. Since we assume that both the input abstract argumentation framework  $\mathit{AAF}$ and an extension $\mathcal{E}$ are finite, therefore the three \textit{for loops} which iterate over all the elements of the extension in lines $3-11$, lines $12-18$ and lines $19-28$ will always terminate. The rest of the proof explores each case in turn.

Within the \textit{for loop} in lines $3-11$, since the set of $\mathit{Attackers}$, i.e., all the arguments that attack $A$, is finite therefore the \textit{for loop} in lines $5$-$10$ will always terminate. Within the \textit{for loop} in lines $12-18$, since the set of $\mathit{Attacked}$, i.e., all the arguments that are attacked by $A$, is finite therefore the \textit{for loop} in lines $14$-$17$ will always terminate. Within the \textit{for loop} in lines $19-28$, since the set of $\mathit{Attackers}$, i.e., all the arguments that attack $A$ is finite, therefore the \textit{for loop} in lines $21$-$27$ will always terminate. 

Thus, we have proved that Algorithm~\ref{alg:approximate-computing-preferences} always terminates.
\end{proof}

\section{Verifying Preferences}
\label{sec:verify-prefs}

As mentioned in Section~\ref{sec:related-work}, preference-based argumentation frameworks~\cite{Amgoud:1998} ignore or remove the attacks where the attacked argument is stronger than the attacking argument. It was found out later that the resulting extension violates the basic condition imposed on acceptability semantics, which is the conflict-freeness of extensions. This problem was later resolved in a new preference-based argumentation framework that guarantees conflict-free extensions with a symmetric conflict relation~\cite{Amgoud2011,AMGOUD2014585,MODGIL2009901}. The preference relation is then used to determine the direction of the defeat relation between the two arguments. 

Since we assume our input extension to be conflict-free both approaches of attack removal or reversal would work. Therefore, we consider two methods for the application of preferences, and verifying that this results in our desired input extension. A preference-based argumentation framework (PAF) can be transformed into an abstract argumentation framework (AAF) by: \begin{enumerate}
\item applying preferences by attack removal, or, 
\item applying preferences by attack reversal.
\end{enumerate}

We now present the formal definitions for both methods of applying preferences to transform a PAF into an AAF.

\begin{definition}{(Applying preferences by attack removal):}
\label{def:attack_removal}
A preference-based argumentation framework $(\mathcal{A}, \mathit{Def}, \geq)$ can be transformed into an abstract argumentation framework $(\mathcal{A}, \mathcal{R})$ as follows: $\forall A,B \in \mathcal{A}$, it is the case that $(B,A) \in \mathcal{R}$ iff $(B,A)\in \mathit{Def}$ and it is not the case that $A>B$.
% 	\begin{enumerate}[(i)]
% 	\item $\forall A,B \in \mathcal{A}$, it is the case that $(B,A) \not\in \mathcal{R}$ iff $(B,A)\in \mathit{Def}$ and it is the case that $A>B$.
% 	\item $\forall A,B \in \mathcal{A}$, it is the case that $(B,A) \in \mathcal{R}$ iff $(B,A)\in \mathit{Def}$ and it is not the case that $A>B$.
% 	\end{enumerate}
\end{definition}

In other words, if $B$ defeats $A$ and $A$ is preferred to $B$ then the attack $(B,A)$ will not appear in the abstract argumentation framework $(\mathcal{A}, \mathcal{R})$.

\begin{definition}{(Applying preferences by attack reversal):}
\label{def:attack_reversal}
A preference-based argumentation framework $(\mathcal{A}, \mathit{Def}, \geq)$ can be transformed into an abstract argumentation framework $(\mathcal{A}, \mathcal{R})$ as follows: 
	\begin{enumerate}[(i)]
	\item $\forall A,B \in \mathcal{A}$, it is the case that $(B,A) \not\in \mathcal{R}$ and $(A,B) \in \mathcal{R}$ iff $(B,A)\in \mathit{Def}$ and it is the case that $A>B$.
	\item $\forall A,B \in \mathcal{A}$, it is the case that $(B,A) \in \mathcal{R}$ iff $(B,A)\in \mathit{Def}$ and it is not the case that $A>B$.	
	\end{enumerate}
\end{definition}

In other words, if $B$ defeats $A$ and $A$ is preferred to $B$ then the attack $(B,A)$ will not appear in the abstract argumentation framework $(\mathcal{A}, \mathcal{R})$, and the reverse attack $(A,B)$ will appear in the abstract argumentation framework $(\mathcal{A}, \mathcal{R})$. Also, if $B$ defeats $A$ and $A$ is not preferred to $B$ then the attack $(B,A)$ will appear in the abstract argumentation framework $(\mathcal{A}, \mathcal{R})$.

% \begin{figure}
% 	\centering
% 	\begin{tikzpicture}[
% 	> = stealth, % arrow head style
% 	shorten > = 1pt, % don't touch arrow head to node
% 	auto,
% 	node distance = 2cm, % distance between nodes
% 	thick % line style
% 	]
	
% 	\tikzstyle{every state}=[
% 	draw = black,
% 	thick,
% 	fill = white,
% 	minimum size = 4mm
% 	]
% 	\node[state] (A) {$A$};
% 	\node[state] (B) [right of=A] {$B$};
% 	\node[state] (C) [right of=B] {$C$};
% 	\node[state] (D) [right of=C] {$D$};
% 	\node[state] (E) [right of=D] {$E$};
	
% 	\path[->] (A) edge node {} (B);
% 	\path[->] (C) edge node {} (B);
% 	\path[->] (C) edge[bend left] node {} (D);        
% 	\path[->] (D) edge[bend left] node {} (C);     
% 	\path[->] (D) edge node {} (E);
	
% 	\end{tikzpicture}
% 	\caption{Example abstract argumentation framework $\mathit{AAF}_1$}
% 	\label{fig:abst_arg_graph}
% \end{figure}

\begin{example}
	\label{ex:verification-example}
We refer back to the abstract argumentation framework $(\mathcal{A}, \mathcal{R})$ of Figure~\ref{fig:abst_arg_graph1}, where $\mathcal{A} = \lbrace A, B, C , D, E \rbrace$ and $\mathcal{R} = \lbrace (A, B), (C, B)$, $(C, D), (D, C), (D, E) \rbrace$. We consider the preferred extension $\mathcal{E}_1 = \{A,C,E\}$ to compute the following set of sets of preferences using Algorithm $1$: 
 %\left\{\begin{array}...\end{array}\right\}

	  \begin{gather*}
	  	\{ \{C>D, A>B, C>B, E>D\}, \\
	  	\{C>D, A>B, C>B, E=D\}, \\
	  	\{C>D, A>B, C>B, D>E\}, \\
	  	\{C>D, A>B, C=B, E>D\}, \\
	  	\{C>D, A>B, C=B, E=D\}, \\
	  	\{C>D, A>B, C=B, D>E\}, \\ 
	    \{C>D, A=B, C>B, E>D\}, \\
	    \{C>D, A=B, C>B, E=D\}, \\
	    \{C>D, A=B, C>B, D>E\}, \\ 
	    \{C>D, A=B, C=B, E>D\}, \\
	    \{C>D, A=B, C=B, E=D\}, \\
	    \{C>D, A=B, C=B, D>E\}, \}
	  \end{gather*}

 We consider the first set of preferences $\mathit{Pref}_1 = \{C>D, A>B, C>B, E>D\}$ to demonstrate the application of preferences. The corresponding preference-based argumentation framework is denoted as $\mathit{PAF} = (\mathcal{A}, \mathit{Def}, \geq)$ which represents the AAF of Figure~\ref{fig:abst_arg_graph1} with a set of preferences $\mathit{Pref}_1 = \{C>D, A>B, C>B, E>D\}$. The preferences can be applied using two different methods, attack removal and attack reversal, defined earlier in Definitions~\ref{def:attack_removal} and~\ref{def:attack_reversal} to transform the PAF to an AAF. 
The application of the first method, i.e., attack removal, results in an abstract argumentation framework shown in Figure~\ref{fig:abst_arg_graph_removal}, where the attacks $(D,C) \in \mathcal{R}$ and $(D,E) \in \mathcal{R}$ have been removed. The application of the second method, i.e., attack reversal, results in an abstract argumentation framework shown in Figure~\ref{fig:abst_arg_graph_reversal}, where the attack $(C,D) \in \mathcal{R}$ represented as a dashed arrow denotes both the normal and reverse attacks and the attack $(E,D) \in \mathcal{R}$ represented as a dotted arrow denotes reverse attack.
 
 \begin{figure}
	\centering
	\begin{tikzpicture}[
	> = stealth, % arrow head style
	shorten > = 1pt, % don't touch arrow head to node
	auto,
	node distance = 2cm, % distance between nodes
	thick % line style
	]
	
	\tikzstyle{every state}=[
	draw = black,
	thick,
	fill = white,
	minimum size = 4mm
	]
	\node[state] (A) {$A$};
	\node[state] (B) [right of=A] {$B$};
	\node[state] (C) [right of=B] {$C$};
	\node[state] (D) [right of=C] {$D$};
	\node[state] (E) [right of=D] {$E$};
	
	\path[->] (A) edge node {} (B);
	\path[->] (C) edge node {} (B);
	\path[->] (C) edge node {} (D);
	%\path[->] (C) edge[bend left] node {} (D);        
	%\path[->] (D) edge[bend left] node {} (C);     
	%\path[->] (D) edge node {} (E);
	
	\end{tikzpicture}
	\caption{Transformed abstract argumentation framework $\mathit{AAF}_3$}
	\label{fig:abst_arg_graph_removal}
\end{figure}
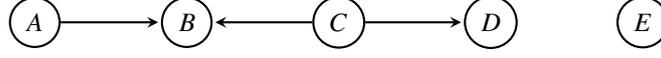

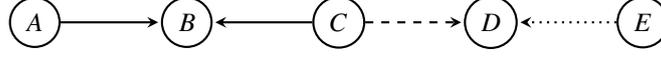
\begin{figure}
	\centering
	\begin{tikzpicture}[
	> = stealth, % arrow head style
	shorten > = 1pt, % don't touch arrow head to node
	auto,
	node distance = 2cm, % distance between nodes
	thick % line style
	]
	
	\tikzstyle{every state}=[
	draw = black,
	thick,
	fill = white,
	minimum size = 4mm
	]
	\node[state] (A) {$A$};
	\node[state] (B) [right of=A] {$B$};
	\node[state] (C) [right of=B] {$C$};
	\node[state] (D) [right of=C] {$D$};
	\node[state] (E) [right of=D] {$E$};
	
	\path[->] (A) edge node {} (B);
	\path[->] (C) edge node {} (B);
	\path[->] (C) edge[dashed] node {} (D);
    \path[->] (E) edge[dotted] node {} (D);
	%\path[->] (C) edge[bend left] node {} (D);        
	%\path[->] (D) edge[bend left] node {} (C);     
	%\path[->] (D) edge node {} (E);
	
	\end{tikzpicture}
	\caption{Transformed abstract argumentation framework $\mathit{AAF}_4$}
	\label{fig:abst_arg_graph_reversal}
\end{figure}

To verify that the set of preferences $\mathit{Pref}_1 = \{C>D, A>B, C>B, E>D\}$ is correct. We apply the preferences to get the transformed AAFs shown in Figures~\ref{fig:abst_arg_graph_removal} and~\ref{fig:abst_arg_graph_reversal}. The resulting preferred extension for both transformed AAFs is $\{A,C,E\}$. Thus, the set of preferences $\mathit{Pref}_1$ has been verified to be correct.

\end{example}

\subsection{Algorithms for Verifying Preferences}
\label{sec:alg-verify-prefs}

In order to assess and verify whether the computed set of preferences over arguments is correct, we need to test that the set of preferences when applied to a given abstract argumentation framework results in our original input extension under a given semantics. We now state the verification problem precisely as follows. 

\begin{problem}
\label{prob:verifying-preferences}
    Given an abstraction argumentation framework $AAF$, an extension $\mathcal{E}$, a semantics $\sigma \in \{\mathit{grounded, preferred, stable}\}$, and a set of sets of preferences $\mathit{PrefSet}$, each preference set $\mathit{Prefs} \in \mathit{PrefSet}$ when applied to the $AAF$ results in the single input extension $\mathcal{E}$ under a given semantics $\sigma$.
\end{problem}

To solve Problem~\ref{prob:verifying-preferences}, we present and describe Algorithms~\ref{alg:apply-preferences1} and~\ref{alg:apply-preferences2} for the verification of preferences following the two methods described above\footnote{Please note, although an algorithm for preference application by attack removal will suffice for this work, we have presented the algorithm for preference application by attack reversal for comparison of both approaches in our experiments that are presented later.}. The input to both algorithms is a tuple $\langle \mathit{AAF},\mathcal{E},\sigma,\mathit{PrefSet} \rangle$, consisting of:
\begin{itemize}
% 		\item Preference-based argumentation framework $\mathit{PAF} = \langle \mathcal{A}, \mathit{Def}, \geq \rangle$, where $\mathcal{A}$ is a set of arguments, $\mathit{Def}$ is the defeat binary relation on $\mathcal{A}$, and $\geq$ is a pre-ordering defined on $\mathcal{A} \times \mathcal{A}$.

\item An abstract argumentation framework $\mathit{AAF} = (\mathcal{A}, \mathcal{R})$, where $\mathcal{A}$ is a set of arguments and $\mathcal{R}$ is an attack relation $(\mathcal{R} \subseteq \mathcal{A} \times \mathcal{A})$.
		\item An extension $\mathcal{E}$ consists of a finite number of conflict-free arguments such that $\mathcal{E} \subseteq \mathcal{A}$.
		\item $\sigma$ is a given semantic, where 
		$\sigma \in \{\mathit{grounded, preferred, stable}\}$.
		\item $\mathit{PrefSet}$ is a set of sets of preferences, where each set of preferences $\mathit{Prefs} \in \mathit{PrefSet}$ is represented as $\mathit{Prefs} = \lbrace A > B, B = C, .... \rbrace$ such that $\lbrace A,B,C, .... \rbrace \subseteq \mathcal{A}$. 
\end{itemize}

Both algorithms output a Boolean variable $\mathit{vcheck}$, which is \textit{true} if all $\mathit{Prefs} \in \mathit{PrefSet}$ are correct, otherwise it is \textit{false}.

The main functionalities of Algorithm~\ref{alg:apply-preferences1} are:
\begin{itemize}
    \item \textit{ApplyPreferences1} applies the preferences using the first method of attack removal in the abstract argumentation framework given in Definition~\ref{def:attack_removal}.
    \item \textit{ComputeExtensions} computes and returns all the extensions in the abstract argumentation framework given a particular semantics $\sigma$\footnote{We note that this functionality is not defined as we use an existing function implemented in the Tweety library~\cite{Thimm:2017e} for this purpose.}.
    \item The set of preferences are verified by getting the transformed abstract argumentation framework returned by \textit{ApplyPreferences1} and checking that the resulting extension for a given semantics $\sigma$ returned by \textit{ComputeExtensions} function is the same (and only) as the input extension.
\end{itemize} 
\begin{algorithm}[H]
	\begin{algorithmic}[1]
		\Require $\mathit{AAF}$, an abstract argumentation framework
		\Require $\mathcal{E}$, an extension consisting of conflict-free arguments
		\Require $\sigma$, semantics for computing the extension
		\Require $\mathit{PrefSet}$, a set of sets of preferences
		\Ensure $\mathit{vcheck}$, a boolean value indicating preference verification success or failure
		\Function{VerifyPreferences$_1$}{$\mathit{AAF},\mathcal{E},\sigma,\mathit{PrefSet}$}
		\State $\mathit{count} \gets 0$ \Comment{number of correct sets of preferences}
		\ForAll{$\mathit{Prefs} \in\mathit{PrefSet}$}  
		\State $\mathit{AAF'} \gets$ ApplyPreferences$_1(\mathit{AAF}, \mathit{Prefs})$
		\State $\mathcal{E}_s \gets$ ComputeExtensions$(\mathit{AAF'}, \sigma)$
 		\If{$|\mathcal{E}_s|=1$ \& $\mathcal{E'} \in \mathcal{E}_s$ s.t. $\mathcal{E'} = \mathcal{E}$}
 		\Comment computed extension is equal to $\mathcal{E}$
  		\State $\mathit{count} \gets \mathit{count} + 1$ \Comment increment $\mathit{count}$ 
		\EndIf  
		\EndFor
		\If{$|\mathit{PrefSet}|=|count|$} \Comment all sets of preferences in $\mathit{PrefSet}$ are correct
		\State $\mathit{vcheck} \gets \mathit{true}$
		\Else 
		\State $\mathit{vcheck} \gets \mathit{false}$
		\EndIf
		\State \Return $\mathit{vcheck}$
		\EndFunction
	\end{algorithmic}
	\caption{Verify Preferences (Attack Removal)}
	\label{alg:apply-preferences1}	
\end{algorithm}

\begin{algorithm}[H]
	\begin{algorithmic}[1]
		\Require $\mathit{AAF}$, an abstract argumentation framework
		\Require $\mathit{Prefs}$, a set of preferences
      	\Ensure $\mathit{AAF}'$, an updated abstract argumentation framework
		\Function{ApplyPreferences$_1$}{$\mathit{AAF},\mathit{Prefs}$}
		\ForAll{$(B,A) \in \mathcal{R}$}  
		\If{$A>B \not\in \mathit{Prefs}$}  \Comment $A$ is not preferred to its attacker $B$
		%\State remove $(B,A) \in \mathcal{R}$
		\State $\mathcal{R'} \gets \mathcal{R'} \cup \{ (B,A) \}$ \Comment add attack
		\EndIf
		\EndFor		
		\State $\mathit{AAF'} \gets (\mathcal{A}, \mathcal{R'})$ \Comment argumentation framework $\mathit{AAF'}$ with attack relation $\mathcal{R'}$
		\State \Return $\mathit{AAF'}$
		\EndFunction
	\end{algorithmic}
	\caption{Apply Preferences (Attack Removal)}
	\label{alg:apply-preferences1-algorithm}	
\end{algorithm}

We establish that Algorithm~\ref{alg:apply-preferences1} is sound (that is, its output is correct) and complete (that is, it outputs all solutions). We start with its soundness:

\begin{theorem}{(Soundness):}
\label{theorem:soundness}
	Algorithm~\ref{alg:apply-preferences1} is sound in that given an abstract argumentation framework  $\mathit{AAF}$, an extension $\mathcal{E}$, a set of sets of preferences $\mathit{PrefSet}$, and semantics $\sigma$ as input, each preference set $\mathit{Prefs}\in \mathit{PrefSet}$, when applied to the $\mathit{AAF}$ results in the input $\mathcal{E}$ (under a given semantics $\sigma$).
\end{theorem}

\begin{proof}
We prove this by first proving the auxiliary Algorithm~\ref{alg:apply-preferences1-algorithm} for applying preferences. Algorithm~\ref{alg:apply-preferences1-algorithm} goes through all attacks $(B,A) \in \mathcal{R}$ in the input AAF and adds the attack $(B,A)$ to the attack relation $\mathcal{R}'$ where the attacked argument $A$ is not preferred to the attacking argument $B$ as per lines $3-4$. This ensures that any attacks of the form $(B,A)$ where the attacked argument $A$ is preferred to the attacking argument $B$ are not added to the attack relation $\mathcal{R}'$. This results in the transformed $AAF'= (\mathcal{A}, \mathcal{R}')$.

Verification is then performed by invoking \textit{ComputeExtensions} that computes and returns all the extensions in the transformed abstract argumentation framework $AAF'$ for a given semantics $\sigma$. 
Since, if the output is a single extension which is equal to the input extension as per lines $6-8$, therefore $\mathit{Prefs}$ is verified to be correct. Thus, $\mathit{PrefSet}$ is then verified to be correct if the verification for each preference set $\mathit{Prefs}\in \mathit{PrefSet}$ is correct as per lines $10-14$.

\end{proof}

\begin{theorem}{(Completeness):}
	Algorithm~\ref{alg:apply-preferences1} is complete in that given an abstract argumentation framework  $\mathit{AAF}$, an extension $\mathcal{E}$, a set of sets preferences $\mathit{PrefSet}$, and semantics $\sigma$ as input, each $\mathit{Prefs}\in\mathit{PrefSet}$ is verified to be correct, i.e., each $\mathit{Prefs}\in\mathit{PrefSet}$ when applied to the $\mathit{AAF}$ results in the input $\mathcal{E}$ (under a given semantics $\sigma$). 
\end{theorem}

\begin{proof}
Algorithm~\ref{alg:apply-preferences1} goes through each $\mathit{Prefs}\in\mathit{PrefSet}$ to verify its correctness, as per lines $3-9$. It then invokes Algorithm~\ref{alg:apply-preferences1-algorithm} to apply each $\mathit{Prefs}$ to the $AAF$ and get the updated $AAF'$ as per line $4$. \textit{ComputeExtensions} is then invoked to get the extension of the updated $AAF'$ under a given semantics $\sigma$.

Algorithm~\ref{alg:apply-preferences1-algorithm} goes through all attacks $(B,A) \in \mathcal{R}$ in the input AAF and adds the attack $(B,A)$ to the attack relation $\mathcal{R}'$ where the attacked argument $A$ is not preferred to the attacking argument $B$ as per lines $2-6$.

Lines $10-14$ in Algorithm~\ref{alg:apply-preferences1} ensure that all $\mathit{Prefs}\in\mathit{PrefSet}$ are verified to be correct.
\end{proof}

After having proved the soundness and completeness of Algorithm~\ref{alg:apply-preferences1}, we establish its termination.

\begin{theorem}{(Termination):}
	Given an abstract argumentation framework  $\mathit{AAF}$, an extension $\mathcal{E}$, a set of sets preferences $\mathit{PrefSet}$, and semantics $\sigma$ as input, Algorithm~\ref{alg:apply-preferences1} always terminates.
\end{theorem}
\begin{proof}
Algorithm~\ref{alg:apply-preferences1} consists of one \textit{for loop} that goes through each $\mathit{Prefs}\in\mathit{PrefSet}$ to verify its correctness, as per lines $3-9$. Since we assume that  the input $\mathit{PrefSet}$ is finite, therefore the \textit{for loop} which iterate over all the elements $\mathit{Prefs}\in\mathit{PrefSet}$ in lines $3-9$ will always terminate. 

Algorithm~\ref{alg:apply-preferences1} invokes Algorithm~\ref{alg:apply-preferences1-algorithm} to apply each $\mathit{Prefs}\in\mathit{PrefSet}$ to the $AAF$ and get the updated $AAF'$. We now prove that Algorithm~\ref{alg:apply-preferences1-algorithm} terminates. Since we assume that the attack relation $\mathcal{R}$ of the input $AAF$ is finite, therefore the \textit{for loop} which iterates over all the elements $(B,A)\in \mathcal{R}$ in lines $2-6$ will always terminate.

Thus, we have proved that Algorithm~\ref{alg:apply-preferences1} always terminates.
\end{proof}

The main functionalities of Algorithm~\ref{alg:apply-preferences2} are:
\begin{itemize}
\item \textit{ApplyPreferences2} applies the preferences using the second method of attack reversal in the abstract argumentation framework given in Definition~\ref{def:attack_reversal}.
   \item \textit{ComputeExtensions} computes and returns all the extensions in the abstract argumentation framework given a particular semantics $\sigma$. 
    \item The set of preferences are verified by getting the transformed abstract argumentation framework returned by \textit{ApplyPreferences2} and checking that the resulting extension for a given semantic $\sigma$ returned by \textit{ComputeExtensions} function  is the same (and only) as the input extension.
\end{itemize}

\begin{algorithm}[H]
	\begin{algorithmic}[1]
		\Require $\mathit{AAF}$, an abstract argumentation framework
		\Require $\mathcal{E}$, an extension consisting of conflict-free arguments
		\Require $\sigma$, semantics for computing the extension
		\Require $\mathit{PrefSet}$, a set of sets of preferences
		\Ensure $\mathit{vcheck}$, a boolean value indicating preference verification success or failure		
		\Function{VerifyPreferences$_2$}{$\mathit{AAF},\mathcal{E},\sigma, \mathit{PrefSet}$}
		\State $\mathit{count} \gets 0$ \Comment{number of correct sets of preferences}
		\ForAll{$\mathit{Prefs} \in\mathit{PrefSet}$}  
		\State $\mathit{AAF'} \gets$ ApplyPreferences$_2(\mathit{AAF}, \mathit{Prefs})$
		\State $\mathcal{E}_s \gets$ ComputeExtensions$(\mathit{AAF'}, \sigma)$
 		\If{$|\mathcal{E}_s|=1$ \& $\mathcal{E'} \in \mathcal{E}_s$ s.t. $\mathcal{E'} = \mathcal{E}$}
 		\Comment computed extension is equal to $\mathcal{E}$
  		\State $\mathit{count} \gets \mathit{count} + 1$ \Comment increment count
		\EndIf  
		\EndFor
		\If{$|\mathit{PrefSet}|=|count|$} \Comment all sets of preferences in $\mathit{PrefSet}$ are correct
		\State $\mathit{vcheck} \gets \mathit{true}$
		\Else 
		\State $\mathit{vcheck} \gets \mathit{false}$
		\EndIf
		\State \Return $\mathit{vcheck}$
		\EndFunction	
	\end{algorithmic}
	\caption{Verify Preferences (Attack Reversal)}
	\label{alg:apply-preferences2}	
\end{algorithm}

\begin{algorithm}[H]
	\begin{algorithmic}[1]
		\Require $\mathit{AAF}$, an abstract argumentation framework
  		\Require $\mathit{Prefs}$, a set of preferences		
    	\Ensure $\mathit{AAF}'$, an updated abstract argumentation framework
		\Function{ApplyPreferences$_2$}{$\mathit{AAF}, \mathit{Prefs}$}
		\ForAll{$(B,A) \in \mathcal{R}$}  
		\If{$A>B \in \mathit{Prefs}$} \Comment $A$ is preferred to its attacker $B$
		%\State remove $(B,A) \in \mathcal{R}$ 
		\State $\mathcal{R'} \gets \mathcal{R'} \cup \{ (A,B) \}$ \Comment add reverse attack
		\ElsIf{$A>B \not\in \mathit{Prefs}$} \Comment $A$ is not preferred to its attacker $B$
		%\State add $(A,B) \in \mathcal{R}$
		\State $\mathcal{R'} \gets \mathcal{R'} \cup \{ (B,A) \}$ \Comment add attack
		\EndIf
		\EndFor
		\State $\mathit{AAF'} \gets (\mathcal{A}, \mathcal{R'})$ \Comment argumentation framework $\mathit{AAF'}$ with attack relation $\mathcal{R'}$
        \State \Return $\mathit{AAF'}$      
		\EndFunction		
	\end{algorithmic}
	\caption{Apply Preferences (Attack Reversal)}
	\label{alg:apply-preferences2-algorithm}	
\end{algorithm}

We establish that Algorithm~\ref{alg:apply-preferences2} is sound (that is, its output is correct) and complete (that is, it outputs all solutions). We start with its soundness:

\begin{theorem}{(Soundness):}
\label{theorem:soundness}
	Algorithm~\ref{alg:apply-preferences2} is sound in that given an abstract argumentation framework  $\mathit{AAF}$, an extension $\mathcal{E}$, a set of sets of preferences $\mathit{PrefSet}$, and semantics $\sigma$ as input, each preference set $\mathit{Prefs}\in \mathit{PrefSet}$, when applied to the $\mathit{AAF}$ results in the input $\mathcal{E}$ (under a given semantics $\sigma$).
\end{theorem}

\begin{proof}
We prove this by first proving the auxiliary Algorithm~\ref{alg:apply-preferences2-algorithm} for applying preferences. Algorithm~\ref{alg:apply-preferences2-algorithm} goes through all attacks $(B,A) \in \mathcal{R}$ in the input AAF and adds:
\begin{itemize}
    \item the reverse attack $(A,B)$ to the the attack relation $\mathcal{R}'$ where the attacked argument $A$ is preferred to the attacking argument $B$ as per lines $3-4$.
    \item the attack $(B,A)$ to the the attack relation $\mathcal{R}'$ where the attacked argument $A$ is not preferred to the attacking argument $B$ as per lines $5-6$.
\end{itemize}
This ensures that any attacks of the form $(B,A)$ where the attacked argument $A$ is preferred to the attacking argument $B$ are not added to the attack relation $\mathcal{R}'$ and the reverse attacks of the form $(A,B)$ are added to the attack relation $\mathcal{R}'$. This results in the transformed $AAF'= (\mathcal{A}, \mathcal{R}')$.

Verification is then performed by invoking \textit{ComputeExtensions} that computes and returns all the extensions in the transformed abstract argumentation framework $AAF'$ for a given semantics $\sigma$. 
Since, if the output is a single extension which is equal to the input extension as per lines $6-8$, therefore the $\mathit{Prefs}$ is verified to be correct. Thus, $\mathit{PrefSet}$ is then verified to be correct if the verification for each preference set $\mathit{Prefs}\in \mathit{PrefSet}$ is correct as per lines $10-14$.

\end{proof}

\begin{theorem}{(Completeness):}
	Algorithm~\ref{alg:apply-preferences2} is complete in that given an abstract argumentation framework  $\mathit{AAF}$, an extension $\mathcal{E}$, a set of sets of preferences $\mathit{PrefSet}$, and semantics $\sigma$ as input, each $\mathit{Prefs}\in\mathit{PrefSet}$ is verified to be correct, i.e., each $\mathit{Prefs}\in\mathit{PrefSet}$ when applied to the $\mathit{AAF}$ results in the input $\mathcal{E}$ (under a given semantics $\sigma$). 
\end{theorem}

\begin{proof}
Algorithm~\ref{alg:apply-preferences2} goes through each $\mathit{Prefs}\in\mathit{PrefSet}$ to verify its correctness, as per lines $3-9$. It then invokes Algorithm~\ref{alg:apply-preferences2-algorithm} to apply each $\mathit{Prefs}$ to the $AAF$ and get the updated $AAF'$ as per line $4$. \textit{ComputeExtensions} is then invoked to get the extension of the updated $AAF'$ under a given semantics $\sigma$.

Algorithm~\ref{alg:apply-preferences2-algorithm} goes through all attacks $(B,A) \in \mathcal{R}$ in the input AAF as per lines $2-8$ and adds:
\begin{itemize}
    \item the reverse attack $(A,B)$ to the the attack relation $\mathcal{R}'$ where the attacked argument $A$ is preferred to the attacking argument $B$ as per lines $3-4$.
    \item the attack $(B,A)$ to the the attack relation $\mathcal{R}'$ where the attacked argument $A$ is not preferred to the attacking argument $B$ as per lines $5-6$.
\end{itemize}

Lines $10-14$ in Algorithm~\ref{alg:apply-preferences1} ensure that all $\mathit{Prefs}\in\mathit{PrefSet}$ are verified to be correct.
\end{proof}

After having proved the soundness and completeness of Algorithm~\ref{alg:apply-preferences2}, we establish its termination.

\begin{theorem}{(Termination):}
	Given an abstract argumentation framework  $\mathit{AAF}$, an extension $\mathcal{E}$, a set of sets preferences $\mathit{PrefSet}$, and semantics $\sigma$ as input, Algorithm~\ref{alg:apply-preferences2} always terminates.
\end{theorem}
\begin{proof}
Algorithm~\ref{alg:apply-preferences2} consists of one \textit{for loop} that goes through each $\mathit{Prefs}\in\mathit{PrefSet}$ to verify its correctness, as per lines $3-9$. Since we assume that  the input $\mathit{PrefSet}$ is finite, therefore the \textit{for loop} which iterates over all the elements $\mathit{Prefs}\in\mathit{PrefSet}$ in lines $3-9$ will always terminate. 

Algorithm~\ref{alg:apply-preferences2} invokes Algorithm~\ref{alg:apply-preferences2-algorithm} to apply each $\mathit{Prefs}\in\mathit{PrefSet}$ to the $AAF$ and get the updated $AAF'$. We now prove that Algorithm~\ref{alg:apply-preferences2-algorithm} terminates. Since we assume that the attack relation $\mathcal{R}$ of the input $AAF$ is finite, therefore the \textit{for loop} which iterates over all the elements $(B,A)\in \mathcal{R}$ in lines $2-8$ will always terminate.

Thus, we have proved that Algorithm~\ref{alg:apply-preferences2} always terminates.
\end{proof}

\section{Implementation and Evaluation}
\label{sec:imp-eval}

We now discuss the implementation details and evaluation of the algorithms presented in the previous sections.

\subsection{Implementation}
We have implemented our proposed algorithms\footnote{The source code of the implementation of all algorithms is available at \url{https://github.com/Quratul-ain/AAF_Preferences}} for evaluation purposes, in Java and using the Tweety library~\cite{Thimm:2017e}. Figure~\ref{fig:preferences_controlflow} presents an overview of the system functionalities and information flow of the original approach. 

\begin{figure}[h]
	\centering
	\includegraphics[width=1.00\textwidth]{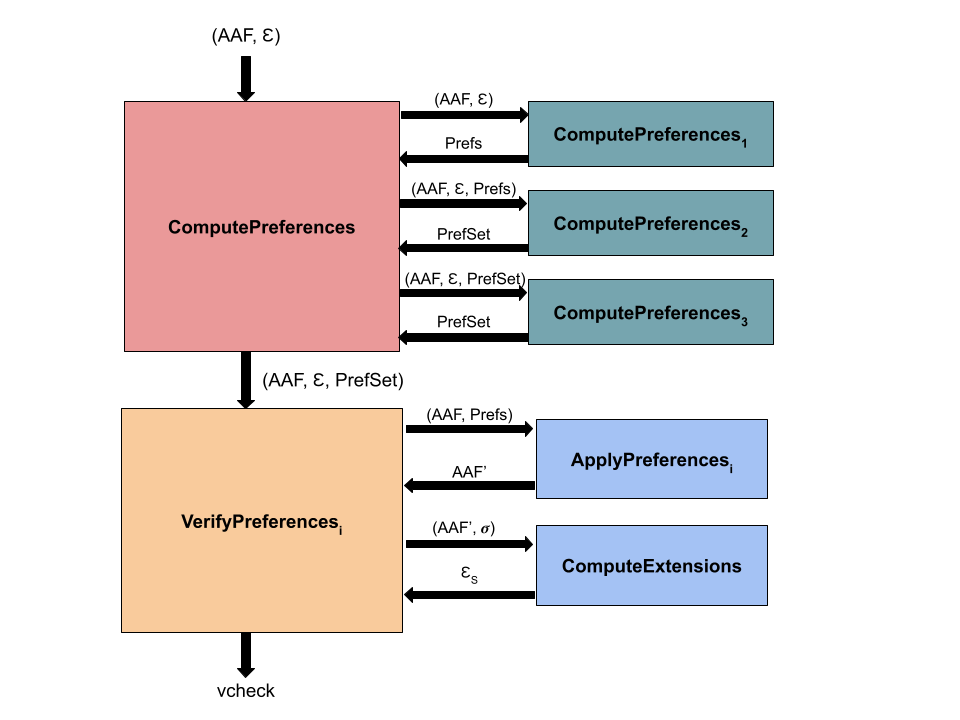}
	\caption{An overview of the system functionalities and information flow (Original Approach)}
	\label{fig:preferences_controlflow}
\end{figure}

The steps are as follows:
\begin{enumerate}
	\item The \textbf{input} to the system is an abstract argumentation graph $\mathit{AAF}$ (set of arguments $\mathcal{A}$ and attack relation $\mathcal{R}$ between arguments) and an extension $\mathcal{E}$ consisting of acceptable arguments. This is denoted by $(\mathit{AAF}, \mathcal{E})$.
	\item The \textbf{ComputePreferences} function:
	\begin{enumerate}[(i)]
		\item invokes the \textbf{ComputePreferences\textsubscript{1}} function (with input $(\mathit{AAF}, \mathcal{E})$) as given in Algorithm~\ref{alg:computing-preferences-case1} that computes the set of case $1$ preferences $\mathit{Prefs}$.
		\item invokes the \textbf{ComputePreferences\textsubscript{2}} function (with input $(\mathit{AAF}, \mathcal{E}, \mathit{Prefs})$) as given in Algorithm~\ref{alg:computing-preferences-case2} that computes case $2$ preferences and combines them with case $1$ preferences. This results in $\mathit{PrefSet}$ which is a set of sets of preferences.
% 		computes the sets of case $2$ preferences (combined with case $1$ preferences) $\mathit{PrefSet}$.
		\item invokes the \textbf{ComputePreferences\textsubscript{3}} function (with input $(\mathit{AAF}, \mathcal{E}, \mathit{PrefSet})$) as given in Algorithm~\ref{alg:computing-preferences-case3} that computes case $3$ preferences and combines them with case $1$ and case $2$ preferences, i.e., $\mathit{Prefs}$ and $\mathit{PrefSet}$. This results in an updated final $\mathit{PrefSet}$ which is a set of sets of preferences containing all three cases of preferences combined together.
	\end{enumerate}
	\item The \textbf{VerifyPreferences\textsubscript{i}} function (with input $(\mathit{AAF}, \mathcal{E}, \mathit{PrefSet})$):
	\begin{enumerate}[(i)]
		\item invokes the \textbf{ApplyPreferences\textsubscript{i}} function (with input $(\mathit{AAF}, \mathit{Prefs})$, where $i=1$ or $i=2$ referring to the two methods of attack removal or reversal respectively) that applies the preferences to the abstract argumentation framework $\mathit{AAF}$ and returns an updated abstract argumentation framework $\mathit{AAF}'$. 
		\item invokes the \textbf{ComputeExtensions} function that computes the set of extensions $\mathcal{E}_s$ and checks the conditions that $ |\mathcal{E}_s=1| \: \& \: \mathcal{E}' \in \mathcal{E}_s \: s.t. \: \mathcal{E}' = \mathcal{E}$. It returns a Boolean variable vcheck, which is \textit{true} if the conditions are true for the preference sets; else it is \textit{false}.
	\end{enumerate}
\end{enumerate}

\begin{figure}[h]
	\centering
	\includegraphics[width=1.00\textwidth]{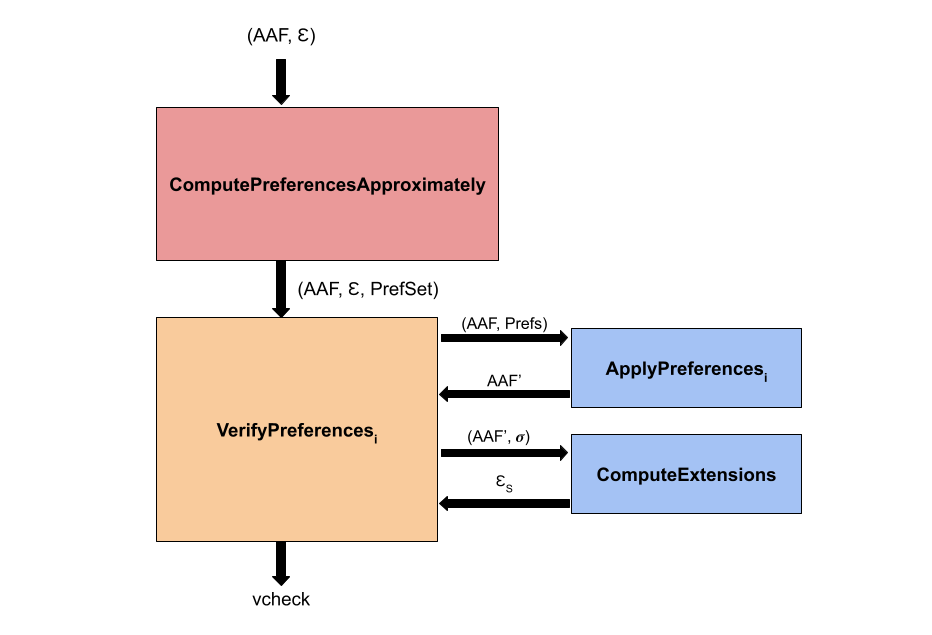}
	\caption{An overview of the system functionalities and information flow (Approximate Approach)}
	\label{fig:preferences_controlflow_approximate}
\end{figure}

Figure~\ref{fig:preferences_controlflow_approximate} presents an overview of the system functionalities and information flow of the approximate approach. 
The steps are as follows:
\begin{enumerate}
	\item The \textbf{input} to the system is an abstract argumentation graph $\mathit{AAF}$ (set of arguments $\mathcal{A}$ and attack relation $\mathcal{R}$ between arguments) and an extension $\mathcal{E}$ consisting of acceptable arguments. This is denoted by $(\mathit{AAF}, \mathcal{E})$.
	\item The \textbf{ComputePreferencesApproximately} function (with input $(\mathit{AAF}, \mathcal{E})$) as given in Algorithm~\ref{alg:approximate-computing-preferences} computes case $1$, case $2$ preferences randomly, and case $3$ preferences randomly and combines them, which results in a set of preferences $\mathit{Prefs}$ containing all three cases of preferences combined together.
	\item $\mathit{PrefSet}$ is created with a set of preferences $\mathit{Prefs}$ from the previous step.
	\item The \textbf{VerifyPreferences\textsubscript{i}} function (with input $(\mathit{AAF}, \mathcal{E}, \mathit{PrefSet})$):
	\begin{enumerate}[(i)]
		\item invokes the \textbf{ApplyPreferences\textsubscript{i}} function (with input $(\mathit{AAF}, \mathit{Prefs})$, where $i=1$ or $i=2$ referring to the two methods of attack removal or reversal respectively) that applies the preferences to the abstract argumentation framework $\mathit{AAF}$ and returns an updated abstract argumentation framework $\mathit{AAF}'$. 
		\item invokes the \textbf{ComputeExtensions} function that computes the set of extensions $\mathcal{E}_s$ and checks the conditions that $ |\mathcal{E}_s=1| \: \& \: \mathcal{E}' \in \mathcal{E}_s \: s.t. \: \mathcal{E}' = \mathcal{E}$. It returns a Boolean variable vcheck, which is \textit{true} if the conditions are true for the preference sets; else it is \textit{false}.
	\end{enumerate}
\end{enumerate}

\subsection{Evaluation}
\label{sec:experimental_evaluation}
To evaluate how our algorithms perform, we carried out several experiments\footnote{All data sets are included in the appendix.} to analyse various metrics of performance. In this section, we present the methodology for our experimental set-up, experiments performed and the analysis of their results.

\subsubsection{Methodology}
We now present the experimental methodology that we have adopted in order to evaluate the algorithms used in our system as follows:
\begin{enumerate}
\item We generated abstract argumentation frameworks of increasing size, for which, we used an existing benchmark abstract argumentation framework generator from the Tweety library~\cite{Thimm:2017e}. 
\item We input the number of arguments $\mathit{AAF_{size}}$ of the abstract argumentation framework and attack probability $\mathit{Pr}$ to the generator in order to randomly generate the abstract argumentation frameworks. 
\item The attack probability $\mathit{Pr}$ is taken as $0.25, 0.50$ and $0.75$.
\item We then compute the grounded, preferred and stable extensions respectively. For each type of extension, a different set of experiments is performed.
\item For the grounded extension the value of $\mathit{AAF_{size}}$ starts at $4$ and ends at $12$\footnote{Please note, the original algorithm (i.e., Algorithm $1$) becomes impractical after the maximum $\mathit{AAF_{size}}$ value $12$.}. We only select the generated $\mathit{AAF}$ where the grounded extension is non-empty and at least of size $1$. Additionally, we evaluated the approximate algorithm (with attack probability $\mathit{Pr}=0.25$ in the AAF\footnote{We were not able to generate AAFs of larger sizes with $\mathit{Pr}=0.50$ and $\mathit{Pr}=0.75$ that have at least one non-empty grounded extension using the  existing benchmark abstract argumentation framework generator from the Tweety library~\cite{Thimm:2017e}}) for the grounded extension with larger value of $\mathit{AAF_{size}}$ that starts at $5$ and ends at $60$.
\item For the preferred and stable extensions the value of $\mathit{AAF_{size}}$ starts at $4$ and ends at $16$\footnote{Please note, the original algorithm (i.e., Algorithm $1$) becomes impractical after the maximum $\mathit{AAF_{size}}$ value $16$.}. We select the extension of the largest size if there are more than two preferred extensions for the $AAF$ and similarly if there are more than two stable extensions for the $AAF$. Additionally, we evaluated the approximate algorithm for the preferred and stable extensions with larger value of $\mathit{AAF_{size}}$ that starts at $5$ and ends at $60$.
\end{enumerate}

\subsubsection{Experiments}
The purpose of our experimental analysis is not only to present the scalability of the algorithms but also to analyse the effect of the attack probability on the computation and verification of preferences on the grounded, preferred and stable extensions. Furthermore, we experimentally check that our implementation reflects the soundness property (i.e., Theorem~\ref{theorem:soundness}) proven before for the grounded, preferred and stable extensions, i.e., the correctness of all the computed preference sets. Full detail of the data sets generated for the experiments is given in Tables $4$-$28$ in the Appendix. The experiments were run on an Apple Mac book Pro machine, with 16GB of memory and a 3.22 GHz, Apple M1 Pro 8-Core.

\noindent We investigate the following hypotheses for the original algorithm (i.e., Algorithm $1$).
\begin{itemize}
\item \textit{Hypothesis $1$}. The computation of set of sets of preferences is low/high for the grounded, preferred and stable extensions with lower/higher probability of attacks in the AAF.
\item \textit{Hypothesis $2$}. The number of sets of preferences grow exponentially with the increasing size of AAFs.
\end{itemize}

\noindent We investigate the following hypothesis for the approximate algorithm (i.e., Algorithm $7$).
\begin{itemize}
    \item \textit{Hypothesis $3$}. Algorithm $7$ is scalable for increasing size of AAFs.
\end{itemize}

\noindent We investigate the following hypotheses for both original algorithm (i.e., Algorithm $1$) and approximate algorithm (i.e., Algorithm $7$).
\begin{itemize}
    \item \textit{Hypothesis $4$}. The number of preferences in each preference set does not grow exponentially with the increasing size of AAFs.
    \item \textit{Hypothesis $5$}. The verification of preferences by \textit{attack removal} approach has lower/higher computation run time for the grounded, preferred and stable extensions compared to the verification of preferences by \textit{attack reversal} approach.
    \item \textit{Hypothesis $6$}. Algorithm $1$ and Algorithm $7$ hold the soundness property (i.e., Theorem $4$) for the grounded, preferred and stable semantics.
\end{itemize}

For the \textit{first} set of experiments we measured \textbf{the average execution time in milliseconds} (for 10 instances) between inputting an abstract argumentation framework (where attack probability $\mathit{Pr}$ is taken as $0.25, 0.50$ and $0.75$) and extension under the grounded, preferred and stable semantics, and computing all set of sets of preferences by the Original Algorithm (i.e. Algorithm~\ref{alg:computing-preferences}) as shown in Figure~\ref{fig:graph1_grounded_computing_original}, Figure~\ref{fig:graph1_preferred_computing_original} and Figure~\ref{fig:graph1_stable_computing_original} respectively, and by the Approximate Algorithm (i.e. Algorithm~\ref{alg:approximate-computing-preferences}) as shown in Figure~\ref{fig:graph1_grounded_computing_approx}, Figure~\ref{fig:graph1_preferred_computing_approx} and Figure~\ref{fig:graph1_stable_computing_approx} respectively. We performed the following analysis:
\begin{enumerate}
\item For the \textbf{grounded extension},
    \begin{itemize}
        \item Original Algorithm: In Figure~\ref{fig:graph1_grounded_computing_original}, the line graph of the AAF with $\mathit{Pr}=0.50$ shows a high computation time (worst-performance) at size $11$ and falls down at size $12$, where as with $\mathit{Pr}=0.25$ computation time is a steady low (best-performance) and starts to increase at size $12$, and for $\mathit{Pr}=0.75$ shows an increase in computation at size $10$ which starts to fall down at size $12$.
        \item Approximate Algorithm: In Figure~\ref{fig:graph1_grounded_computing_approx}, the line graphs of the AAF with $\mathit{Pr}=0.25, \mathit{Pr}=0.50,$ and $\mathit{Pr}=0.75$ show a very low computation time that is not more than $1.0$ ms for all AAF sizes.
    \end{itemize}
\item For the \textbf{preferred extension}, 
    \begin{itemize}
        \item Original Algorithm: In Figure~\ref{fig:graph1_preferred_computing_original}, the line graph of the AAF with $\mathit{Pr}=0.25$ shows a high computation time (worst-performance) which starts to increase at size $15$, where as with $\mathit{Pr}=0.75$ shows a steady low computation time (best-performance) till the largest size $16$, and for $\mathit{Pr}=0.50$ shows an increase in computation time at size $16$.
        \item Approximate Algorithm: In Figure~\ref{fig:graph1_preferred_computing_approx}, the line graphs of the AAF with $\mathit{Pr}=0.25, \mathit{Pr}=0.50,$ and $\mathit{Pr}=0.75$ show a very low computation time that is not more than $1.0$ ms for all AAF sizes.
    \end{itemize}

\item For the \textbf{stable extension}, 
    \begin{itemize}
        \item Original Algorithm: In Figure~\ref{fig:graph1_stable_computing_original}, the line graph of the AAF with $\mathit{Pr}=0.25$ shows a high computation time (worst-performance) which starts to increase at size $15$ and decreases slightly at size $16$, where as with $\mathit{Pr}=0.50$ and $\mathit{Pr}=0.75$ shows a steady low computation time (best-performance) till the largest size $16$.
        \item Approximate Algorithm: In Figure~\ref{fig:graph1_stable_computing_approx}, the line graphs of the AAF with $\mathit{Pr}=0.25, \mathit{Pr}=0.50,$ and $\mathit{Pr}=0.75$ show a very low computation time that is not more than $1.0$ ms for all AAF sizes.
    \end{itemize}
\end{enumerate}

\textit{Hypothesis $1$}. For the original algorithm (i.e., Algorithm $1$), we conclude that when the number of attacks is less (i.e., $\mathit{Pr}=0.25$) in the increasing sizes of AAF, the run time for computation of preferences for the grounded extension is lower compared to AAFs with higher number of attacks. On the other hand, we conclude that when the number of attacks is higher (i.e., $\mathit{Pr}=0.75$) in the increasing sizes of AAF, the run time for computation of preferences for the preferred and stable extensions is lower compared to AAFs with lower number of attacks.

For the \textit{second} set of experiments we measured \textbf{the average number of all the set of sets of preferences} (for 10 instances) computed for an input extension under the grounded, preferred and stable semantics for an abstract argumentation framework (where attack probability $\mathit{Pr}$ is taken as $0.25, 0.50$ and $0.75$) by the Original Algorithm (i.e. Algorithm~\ref{alg:computing-preferences}) as shown in Figure~\ref{fig:graph2_grounded_computing_original}, Figure~\ref{fig:graph2_preferred_computing_original} and Figure~\ref{fig:graph2_stable_computing_original} respectively. We performed the following analysis:
\begin{enumerate}
\item For the \textbf{grounded extension}, the line graph of the AAF (as shown in Figure~\ref{fig:graph2_grounded_computing_original}) with $\mathit{Pr}=0.75$ shows a high number of preference sets at size $11$ and falls down at size $12$, where as with $\mathit{Pr}=0.25$ the number of preference sets is a steady low and starts to increase at size $12$, and for $\mathit{Pr}=0.50$ the number of preference sets increases at size $9$ and falls down at size $12$.
\item For the \textbf{preferred extension}, the line graph of the AAF (as shown in Figure~\ref{fig:graph2_preferred_computing_original}) with $\mathit{Pr}=0.50$ shows that the number of preference sets starts to increase at size $14$, the line graph of the AAF with $\mathit{Pr}=0.25$ shows that the number of preference sets starts to gradually increase at size $8$ and grows rapidly at size $15$, and the the line graph of the AAF with $\mathit{Pr}=0.75$ shows a steady low number of preferences sets till the largest size $16$.
\item For the \textbf{stable extension}, the line graph of the AAF (as shown in Figure~\ref{fig:graph2_stable_computing_original}) with $\mathit{Pr}=0.25$ shows that the number of preference sets starts to increase rapidly at size $15$ and falls down slightly at size $16$, the line graph of the AAF with $\mathit{Pr}=0.50$ shows that the number of preference sets starts to increase at size $15$, and the the line graph of the AAF with $\mathit{Pr}=0.75$ shows a steady low number of preferences sets till the largest size $16$.
\end{enumerate}

\textit{Hypothesis $2$}. For the original algorithm (i.e., Algorithm $1$), we conclude that the number of preference sets computed for the grounded extension is less when the AAFs of increasing sizes have lower number of attacks (i.e., $\mathit{Pr}=0.25$) compared to AAFs with higher number of attacks. On the other hand, we conclude that the number of preference sets computed for the preferred and stable extensions is less when the AAFs of increasing sizes have higher number of attacks (i.e., $\mathit{Pr}=0.75$) compared to AAFs with lower number of attacks. For all extensions, the growth of the number of sets of preferences is exponential with the increasing size of AAFs.

For the \textit{third} set of experiments we measured \textbf{the average number of preferences in each set of preferences} (for 10 instances) computed for an input extension under the grounded, preferred and stable semantics for an abstract argumentation framework (where attack probability $\mathit{Pr}$ is taken as $0.25, 0.50$ and $0.75$) by the Original Algorithm (i.e. Algorithm~\ref{alg:computing-preferences}) as shown in Figure~\ref{fig:graph3_grounded_computing_original}, Figure~\ref{fig:graph3_preferred_computing_original} and Figure~\ref{fig:graph3_stable_computing_original} respectively, and by the Approximate Algorithm (i.e. Algorithm~\ref{alg:approximate-computing-preferences}) as shown in Figure~\ref{fig:graph3_grounded_computing_approx}, Figure~\ref{fig:graph3_preferred_computing_approx} and Figure~\ref{fig:graph3_stable_computing_approx} respectively. We performed the following analysis:
\begin{enumerate}
\item For the \textbf{grounded extension}, as shown in Figure~\ref{fig:graph3_grounded_computing_original} and Figure~\ref{fig:graph3_grounded_computing_approx}, the average number of preferences in each preference set does not grow more than $12$ till AAF size $12$. There is a steady increase in the number of preferences in each preference set with the increasing size of the AAF.
\item For the \textbf{preferred extension}, as shown in Figure~\ref{fig:graph3_preferred_computing_original} and Figure~\ref{fig:graph3_preferred_computing_approx}, the average number of preferences in each preference set remains under $30$ till AAF size $16$. Similar to the grounded extension, there is a steady increase in the number of preferences in each preference set with the increasing size of the AAF.
\item For the \textbf{stable extension}, as shown in Figure~\ref{fig:graph3_stable_computing_original} and Figure~\ref{fig:graph3_stable_computing_approx}, the average number of preferences in each preference set remains under $30$ till AAF size $16$. Similar to the grounded and preferred extensions, there is a steady increase in the number of preferences in each preference set with the increasing size of the AAF.
\end{enumerate}

\textit{Hypothesis $4$}. For both the original algorithm (i.e., Algorithm $1$) and approximate algorithm (i.e., Algorithm $7$), we conclude that the number of preferences for all extensions increases steadily with the increasing size of the AAFs which is mostly similar for all attack probabilities. The growth in the number of preferences in each preference set for increasing size of the AAFs is not exponential.

Finally, for the \textit{fourth} set of experiments we measured \textbf{the average elapsed time in milliseconds} (for 10 instances) between inputting an abstract argumentation framework (where attack probability $\mathit{Pr}$ is taken as $0.25, 0.50$ and $0.75$), extension under the grounded, preferred and stable semantics and the set of sets of preferences computed by the Original Algorithm (i.e. Algorithm~\ref{alg:computing-preferences}) or the Approximate Algorithm (i.e. Algorithm~\ref{alg:approximate-computing-preferences}), and verifying all the set of sets of preferences to be correct using the two verification methods\footnote{This is done separately for both the Original and Approximate algorithms.}: 
\begin{enumerate}[(i)] 
\item \textit{Attack removal}, as shown in Figure~\ref{fig:graph_grounded_verification1_original}, Figure~\ref{fig:graph_preferred_verification1_original} and Figure~\ref{fig:graph_stable_verification1_original} for the Original Algorithm, and Figure~\ref{fig:graph_grounded_verification1_approx}, Figure~\ref{fig:graph_preferred_verification1_approx} and Figure~\ref{fig:graph_stable_verification1_approx} for the Approximate Algorithm.
\item \textit{Attack reversal}, as shown in Figure~\ref{fig:graph_grounded_verification2_original}, Figure~\ref{fig:graph_preferred_verification2_original} and Figure~\ref{fig:graph_stable_verification2_original} for the Original Algorithm, and Figure~\ref{fig:graph_grounded_verification2_approx}, Figure~\ref{fig:graph_preferred_verification2_approx} and Figure~\ref{fig:graph_stable_verification2_approx} for the Approximate Algorithm. 
\end{enumerate} 
We performed the following analysis:
\begin{enumerate}
\item For the \textbf{grounded extension}, 
\begin{itemize}
    \item For the preferences computed by the Original Algorithm: In Figure~\ref{fig:graph_grounded_verification1_original} and Figure~\ref{fig:graph_grounded_verification2_original}, the line graphs of the AAF with $\mathit{Pr}=0.75$ for both verification approaches show a high computation time (worst-performance) at size $11$ and this falls down at size $12$, where as with $\mathit{Pr}=0.25$ computation time is a steady low (best-performance) and starts to increase at size $12$, and for $\mathit{Pr}=0.50$ computation time increases at size $9$ and falls down at size $12$. Overall, although the line graph pattern of both verification methods is similar, the computation time of verification by attack removal is lower compared to the attack reversal verification.
    \item For the preferences computed by the Approximate Algorithm: In Figure~\ref{fig:graph_grounded_verification1_approx} and Figure~\ref{fig:graph_grounded_verification2_approx}, the line graphs of the AAF with $\mathit{Pr}=0.25, \mathit{Pr}=0.50,$ and $\mathit{Pr}=0.75$ for both verification approaches show a very low computation time that is not more than $1.0$ ms for all AAF sizes.
\end{itemize}
\item For the \textbf{preferred extension}, 
\begin{itemize}
    \item For the preferences computed by the Original Algorithm: In Figure~\ref{fig:graph_preferred_verification1_original} and Figure~\ref{fig:graph_preferred_verification2_original}, the line graphs of the AAF with $\mathit{Pr}=0.25$ for both verification approaches show a high computation time (worst-performance) which starts to increase rapidly at size $15$, where as with $\mathit{Pr}=0.75$ shows a steady low computation time (best-performance) till the largest size $16$, and with $\mathit{Pr}=0.50$ the computation time goes up at size $14$ and again at size $16$ for both verification approaches. Overall, although the line graph pattern of both verification methods is similar, the computation time of verification by attack removal is lower compared to the attack reversal verification.
    \item For the preferences computed by the Approximate Algorithm: In Figure~\ref{fig:graph_preferred_verification1_approx} and Figure~\ref{fig:graph_preferred_verification2_approx}, the line graph of the AAF with $\mathit{Pr}=0.25, \mathit{Pr}=0.50,$ and $\mathit{Pr}=0.75$ for attack removal verification method shows a very low computation time that is not more than $5.0$ ms for all AAF sizes, and for attack reversal verification method shows a very low computation time that is not more than $3.0$ ms for all AAF sizes.
\end{itemize}
\item For the \textbf{stable extension}, 
\begin{itemize}
    \item For the preferences computed by the Original Algorithm: In Figure~\ref{fig:graph_stable_verification1_original} and Figure~\ref{fig:graph_stable_verification2_original}, the line graphs of the AAF with $\mathit{Pr}=0.25$ for both verification approaches show a high computation time (worst-performance) which starts to increase at size $15$ and for the verification method of attack reversal goes down slightly at size $16$, where as with $\mathit{Pr}=0.75$ for both verification approaches show a steady low computation time (best-performance) till the largest size $16$, and with $\mathit{Pr}=0.50$ for both verification approaches show a small increase at size $15$. Overall, although the line graph pattern of both verification methods is almost similar, the computation time of verification by attack reversal is lower compared to the attack removal verification.
    \item For the preferences computed by the Approximate Algorithm: In Figure~\ref{fig:graph_stable_verification1_approx} and Figure~\ref{fig:graph_stable_verification2_approx}, the line graphs of the AAF with $\mathit{Pr}=0.25, \mathit{Pr}=0.50,$ and $\mathit{Pr}=0.75$ for both verification approaches show a very low computation time that is not more than $4.0$ ms for all AAF sizes.
\end{itemize}
\end{enumerate}

\textit{Hypothesis $5$}. Thus, for the original algorithm, we conclude that the verification of preferences by \textit{attack removal} approach has lower run time for grounded and preferred extensions, whereas, the verification of preferences by \textit{attack reversal} approach has lower run time for the stable extension. The line graph patterns are similar for both verification approaches, since this is dependent on the number of sets of preferences computed previously. For the approximate algorithm, we conclude that the run time for verification of preferences using both approaches is very low (only a few milliseconds).

\noindent \textit{Hypothesis $6$}. For both original algorithm (Algorithm~\ref{alg:computing-preferences}) and approximate algorithm (Algorithm~\ref{alg:approximate-computing-preferences}), both verification methods resulted in the original input extension for all semantics (i.e., grounded, preferred and stable). Thus, we conclude that Algorithm~\ref{alg:computing-preferences} and Algorithm~\ref{alg:approximate-computing-preferences} hold the soundness property for the grounded, preferred and stable semantics in our experimental evaluation.

From the above experimental analysis it is clear that the original Algorithm~\ref{alg:computing-preferences} is exponential in complexity due to the exponential growth of sets of preferences for larger AAF sizes, however, it is useful to identify the maximum AAF size at which point this becomes impractical, which is AAF size $12$ for the grounded extension and AAF size $16$ for the preferred and stable extensions as shown in the experiments. The approximate Algorithm~\ref{alg:approximate-computing-preferences} has a very low run time (i.e., just a few milliseconds) and will be suited for real world settings where scalability of AAF size is important. 

\noindent \textit{Hypothesis $3$}. We carried out further experiments to evaluate the scalability of the approximate Algorithm~\ref{alg:approximate-computing-preferences} for computing preferences and Algorithms~\ref{alg:apply-preferences1} and~\ref{alg:apply-preferences2} for verifying preferences on larger AAF sizes (i.e., sizes between $5$ to $60$) as shown in Figures~\ref{fig:graph_grounded_computing_approx_scaled},~\ref{fig:graph_grounded_verification_approx_scaled},~\ref{fig:graph_preferred_computing_approx_scalable},~\ref{fig:graph_preferred_verification_approx_scaled},~\ref{fig:graph_stable_computing_approx_scaled} and~\ref{fig:graph_stable_verification_approx_scaled}. It is evident from the figures that the approximate Algorithm~\ref{alg:approximate-computing-preferences}, and Algorithms~\ref{alg:apply-preferences1} and~\ref{alg:apply-preferences2} for verifying preferences have very low run time (i.e., just a few milliseconds) even for larger AAF sizes.

\begin{figure}[H]
	\centering
    \caption{Computing Preferences of the Grounded Extension (Original Algorithm)}
\begin{subfigure}[b]{0.60\textwidth}
		\includegraphics[width=\linewidth]{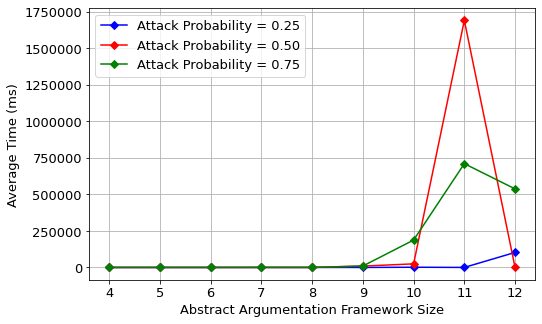}	
    	\caption{Average time in milliseconds for computing all preference sets}
	\label{fig:graph1_grounded_computing_original}
	\end{subfigure}
	\hfill
\begin{subfigure}[b]{0.60\textwidth}
	\centering
	\includegraphics[width=\linewidth]{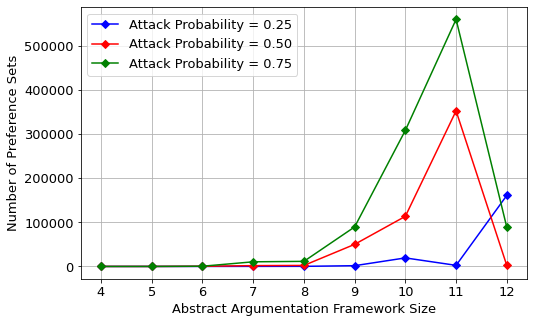}
	\caption{Average number of preference sets}
	\label{fig:graph2_grounded_computing_original}
\end{subfigure}
\hfill 
\begin{subfigure}[b]{0.60\textwidth}
	\centering
	\includegraphics[width=\textwidth]{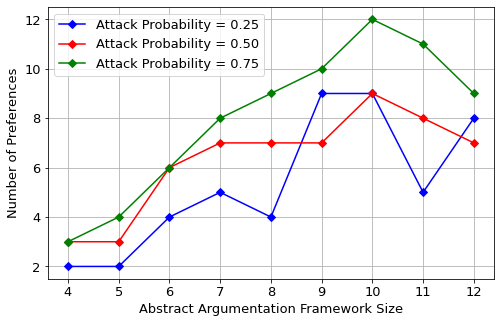}
	\caption{Average number of preferences in each preference set}
	\label{fig:graph3_grounded_computing_original}
\end{subfigure}
    \label{fig:graph_grounded_computing_original}
\end{figure}

\begin{figure}[H]
	\centering
    \caption{Computing Preferences of the Grounded Extension (Approximate Algorithm)}
\begin{subfigure}[b]{0.60\textwidth}
		\includegraphics[width=\linewidth]{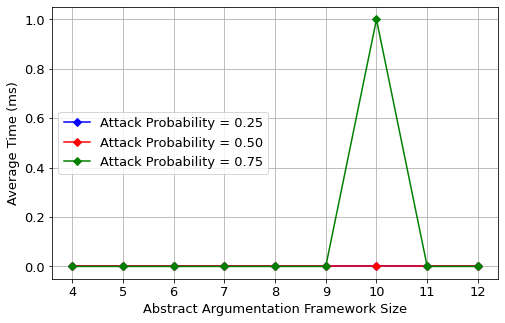}	
    	\caption{Average time in milliseconds for computing all preference sets}
	\label{fig:graph1_grounded_computing_approx}
	\end{subfigure}
\hfill 
\begin{subfigure}[b]{0.60\textwidth}
	\centering
	\includegraphics[width=\textwidth]{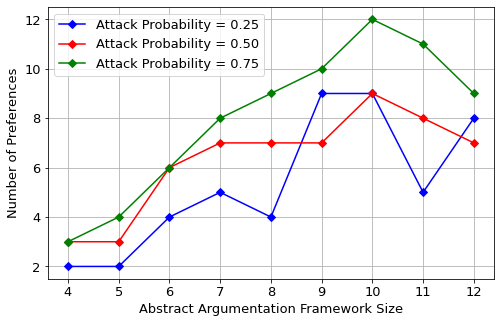}
	\caption{Average number of preferences in each preference set}
	\label{fig:graph3_grounded_computing_approx}
\end{subfigure}
     \label{fig:graph_grounded_computing_approx}
\end{figure}

\begin{figure}[H]
	\centering
    \caption{Verifying Preferences of the Grounded Extension (Original Algorithm)}
    \begin{subfigure}[b]{0.60\textwidth}
	\centering
	\includegraphics[width=\textwidth]{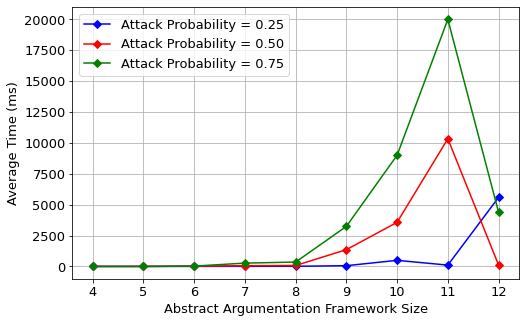}
	\caption{Average time in milliseconds for verifying all preference sets (attack removal)}
	\label{fig:graph_grounded_verification1_original}
    \end{subfigure}
\hfill  
    \begin{subfigure}[b]{0.60\textwidth}
	\centering
	\includegraphics[width=\textwidth]{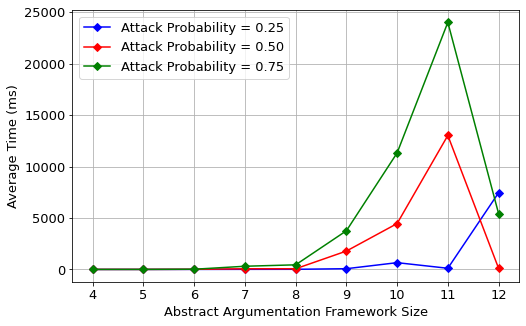}
	\caption{Average time in milliseconds for verifying all preference sets (attack reversal)}
	\label{fig:graph_grounded_verification2_original}
    \end{subfigure}	
    \label{fig:graph_grounded_verification_original}
\end{figure}

\begin{figure}[H]
	\centering
    \caption{Verifying Preferences of the Grounded Extension (Approximate Algorithm)}
    \begin{subfigure}[b]{0.60\textwidth}
	\centering
	\includegraphics[width=\textwidth]{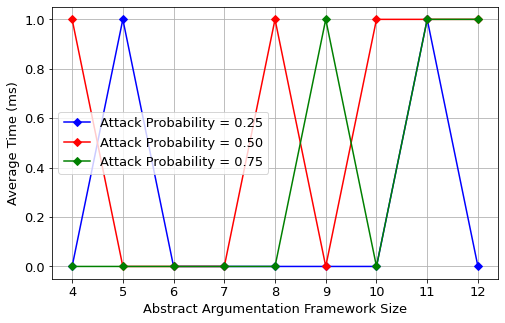}
	\caption{Average time in milliseconds for verifying all preference sets (attack removal)}
	\label{fig:graph_grounded_verification1_approx}
    \end{subfigure}
\hfill  
    \begin{subfigure}[b]{0.60\textwidth}
	\centering
	\includegraphics[width=\textwidth]{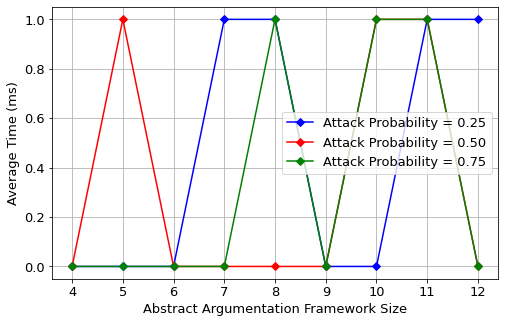}
	\caption{Average time in milliseconds for verifying all preference sets (attack reversal)}
	\label{fig:graph_grounded_verification2_approx}
    \end{subfigure}	
    \label{fig:graph_grounded_verification_approx}
\end{figure}

\begin{figure}[H]
	\centering
    \caption{Computing Preferences of the Preferred Extension (Original Algorithm)}	
    \begin{subfigure}[b]{.6\textwidth}
	\centering	
	\includegraphics[width=\textwidth]{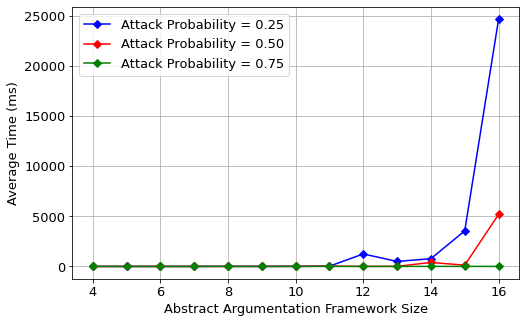}
	\caption{Average time in milliseconds for computing all preference sets}
	\label{fig:graph1_preferred_computing_original}
	\end{subfigure}
\hfill 
    \begin{subfigure}[b]{.6\textwidth}
	\centering
	\includegraphics[width=\textwidth]{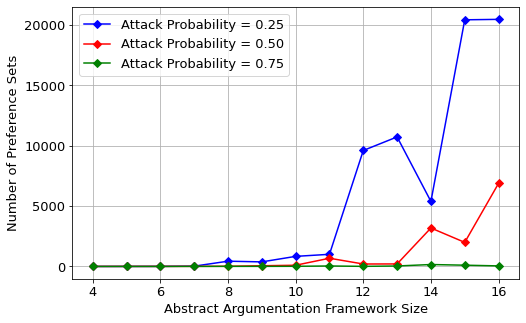} 
	\caption{Average number of preference sets}
	\label{fig:graph2_preferred_computing_original}
    \end{subfigure}	
\hfill
    \begin{subfigure}[b]{.6\textwidth}
	\centering
	\includegraphics[width=\textwidth]{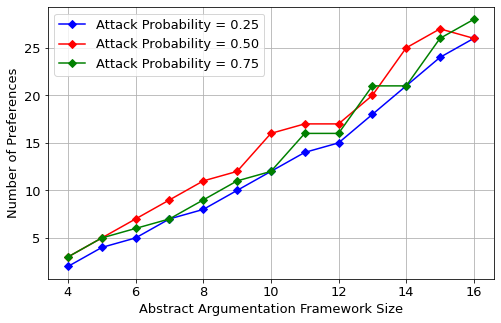}
	\caption{Average number of preferences in each preference set}
	\label{fig:graph3_preferred_computing_original}
    \end{subfigure}
    \label{fig:graph_preferred_computing_original}
\end{figure}

\begin{figure}[H]
	\centering
    \caption{Computing Preferences of the Preferred Extension (Approximate Algorithm)}	
    \begin{subfigure}[b]{.6\textwidth}
	\centering	
	\includegraphics[width=\textwidth]{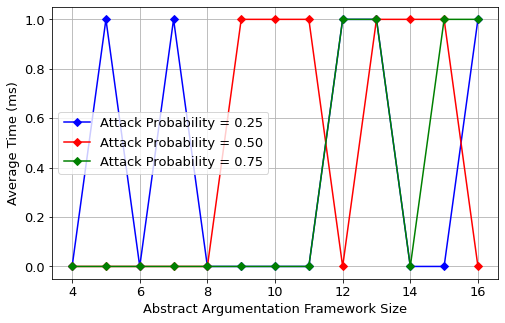}
	\caption{Average time in milliseconds for computing all preference sets}
	\label{fig:graph1_preferred_computing_approx}
	\end{subfigure}
 \hfill
    \begin{subfigure}[b]{.6\textwidth}
	\centering
	\includegraphics[width=\textwidth]{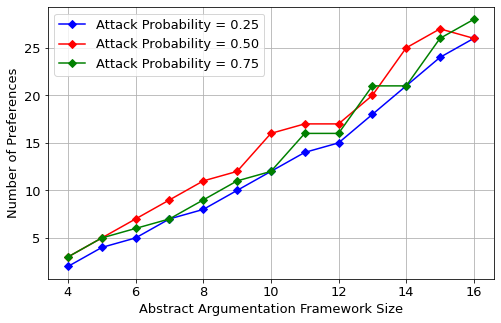}
	\caption{Average number of preferences in each preference set}
	\label{fig:graph3_preferred_computing_approx}
    \end{subfigure}
    \label{fig:graph_preferred_computing_approx}
\end{figure}

\begin{figure}[H]
	\centering
    \caption{Verifying Preferences of the Preferred Extension (Original Algorithm)}
    \begin{subfigure}[b]{.6\textwidth}
	\centering
	\includegraphics[width=\textwidth]{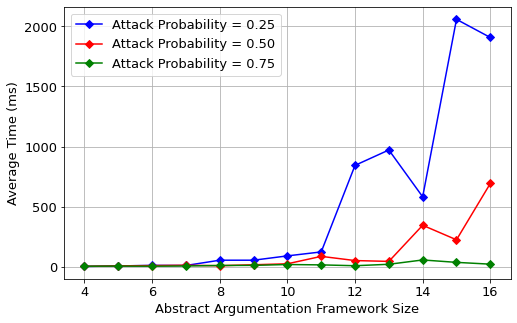}
	\caption{Average time in milliseconds for verifying all preference sets (attack removal)}
	\label{fig:graph_preferred_verification1_original}
    \end{subfigure}
\hfill
    \begin{subfigure}[b]{.6\textwidth}
	\centering
	\includegraphics[width=\textwidth]{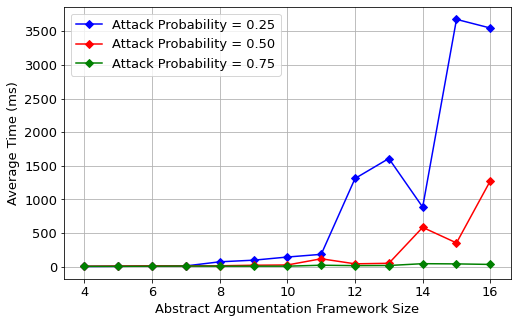}
	\caption{Average time in milliseconds for verifying all preference sets (attack reversal)}
	\label{fig:graph_preferred_verification2_original}
    \end{subfigure}    
    \label{fig:graph_preferred_verification_original}	
\end{figure}

\begin{figure}[H]
	\centering
    \caption{Verifying Preferences of the Preferred Extension (Approximate Algorithm)}
    \begin{subfigure}[b]{.6\textwidth}
	\centering
	\includegraphics[width=\textwidth]{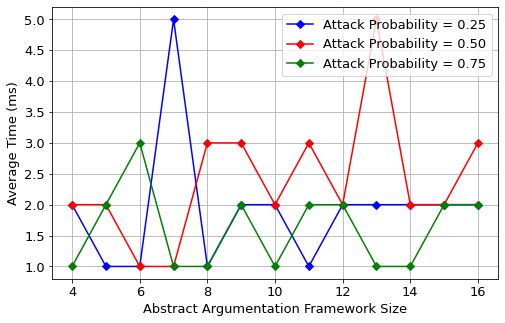}
	\caption{Average time in milliseconds for verifying all preference sets (attack removal)}
	\label{fig:graph_preferred_verification1_approx}
    \end{subfigure}
\hfill
    \begin{subfigure}[b]{.6\textwidth}
	\centering
	\includegraphics[width=\textwidth]{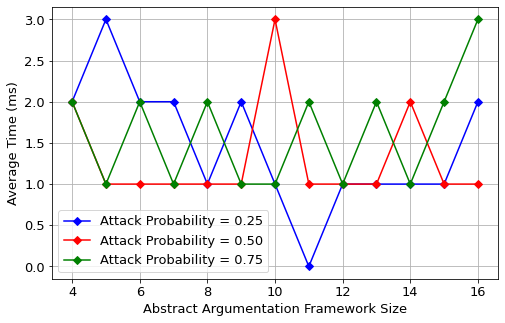}
	\caption{Average time in milliseconds for verifying all preference sets (attack reversal)}
	\label{fig:graph_preferred_verification2_approx}
    \end{subfigure}    
    \label{fig:graph_preferred_verification_approx}	
\end{figure}

\begin{figure}[H]
	\centering
    \caption{Computing Preferences of the Stable Extension (Original Algorithm)}
\begin{subfigure}[b]{.6\textwidth}
		\includegraphics[width=\linewidth]{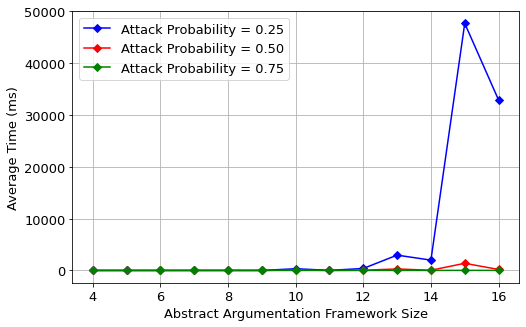}	  
    	\caption{Average time in milliseconds for computing all preference sets}
	\label{fig:graph1_stable_computing_original}
	\end{subfigure}
	\hfill
\begin{subfigure}[b]{.6\textwidth}
	\centering
	\includegraphics[width=\linewidth]{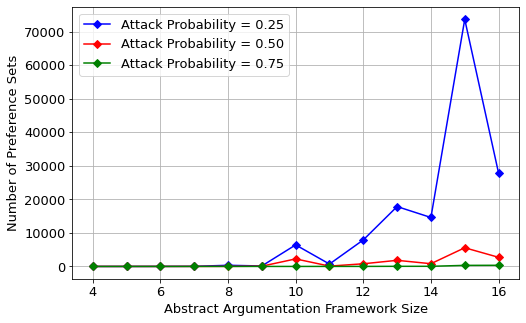}
	\caption{Average number of preference sets}
	\label{fig:graph2_stable_computing_original}
\end{subfigure}
\hfill 
\begin{subfigure}[b]{.6\textwidth}
	\centering
	\includegraphics[width=\textwidth]{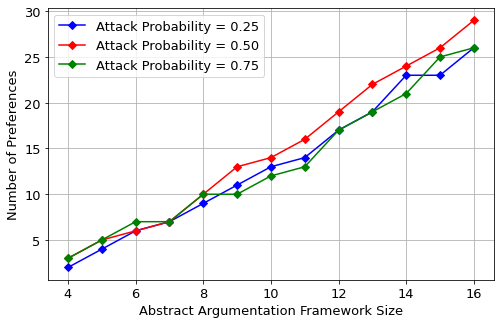}
	\caption{Average number of preferences in each preference set}
	\label{fig:graph3_stable_computing_original}
\end{subfigure}
    \label{fig:graph_stable_computing_original}
\end{figure}

\begin{figure}[H]
	\centering
    \caption{Computing Preferences of the Stable Extension (Approximate Algorithm)}
\begin{subfigure}[b]{.6\textwidth}
		\includegraphics[width=\linewidth]{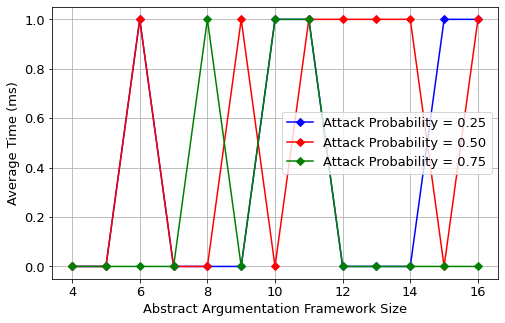}	  
    	\caption{Average time in milliseconds for computing all preference sets}
	\label{fig:graph1_stable_computing_approx}
	\end{subfigure}
\hfill 
\begin{subfigure}[b]{.6\textwidth}
	\centering
	\includegraphics[width=\textwidth]{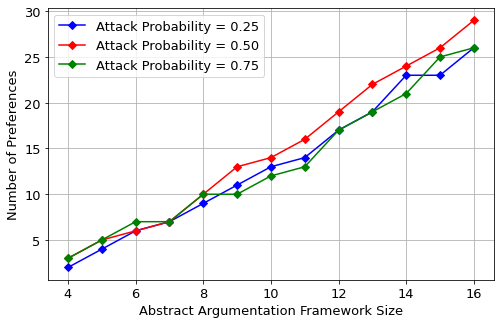}
	\caption{Average number of preferences in each preference set}
	\label{fig:graph3_stable_computing_approx}
\end{subfigure}
    \label{fig:graph_stable_computing_approx}
\end{figure}

\begin{figure}[H]
	\centering
    \caption{Verifying Preferences of the Stable Extension (Original Algorithm)}
    \begin{subfigure}[b]{.6\textwidth}
	\centering
	\includegraphics[width=\textwidth]{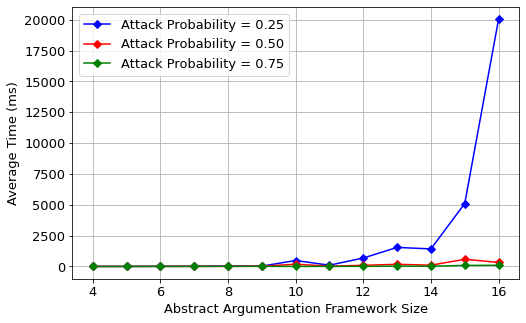}
	\caption{Average time in milliseconds for verifying all preference sets (attack removal)}
	\label{fig:graph_stable_verification1_original}
    \end{subfigure}
\hfill  
    \begin{subfigure}[b]{.6\textwidth}
	\centering
	\includegraphics[width=\textwidth]{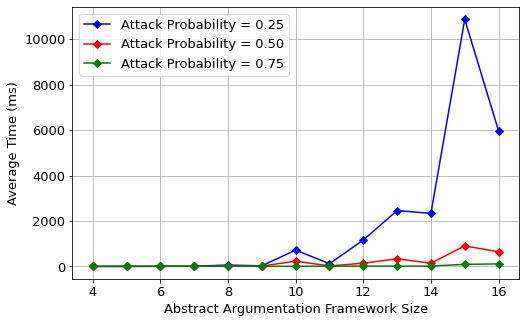}
	\caption{Average time in milliseconds for verifying all preference sets (attack reversal)}
	\label{fig:graph_stable_verification2_original}
    \end{subfigure}	
    \label{fig:graph_stable_verification_original}
\end{figure}

\begin{figure}[H]
	\centering
    \caption{Verifying Preferences of the Stable Extension (Approximate Algorithm)}
    \begin{subfigure}[b]{.6\textwidth}
	\centering
	\includegraphics[width=\textwidth]{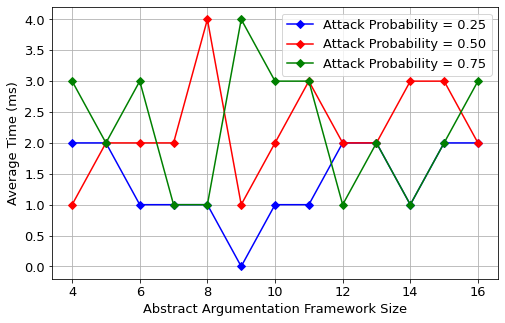}
	\caption{Average time in milliseconds for verifying all preference sets (attack removal)}
	\label{fig:graph_stable_verification1_approx}
    \end{subfigure}
\hfill  
    \begin{subfigure}[b]{.6\textwidth}
	\centering
	\includegraphics[width=\textwidth]{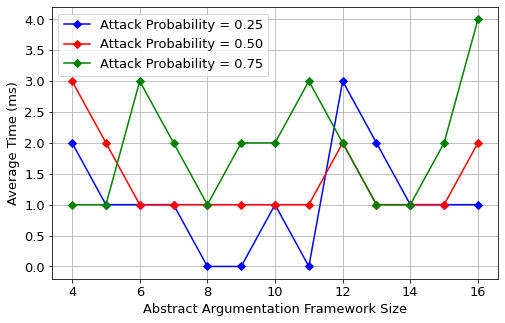}
	\caption{Average time in milliseconds for verifying all preference sets (attack reversal)}
	\label{fig:graph_stable_verification2_approx}
    \end{subfigure}	
    \label{fig:graph_stable_verification_approx}
\end{figure}

\begin{figure}[H]
	\centering
    \caption{Computing Preferences of the Grounded Extension for larger AAF Sizes (Approximate Algorithm)}
\begin{subfigure}[b]{0.60\textwidth}
		\includegraphics[width=\linewidth]{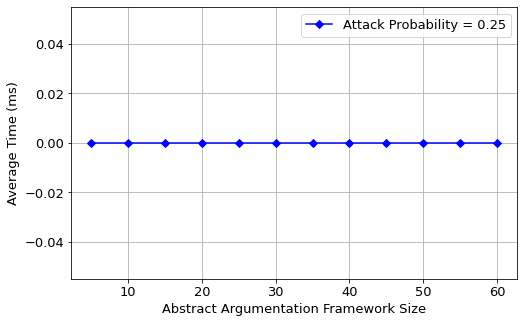}	
    	\caption{Average time in milliseconds for computing all preference sets}
	\label{fig:graph1_grounded_computing_approx_scaled}
	\end{subfigure}
\hfill 
\begin{subfigure}[b]{0.60\textwidth}
	\centering
	\includegraphics[width=\textwidth]{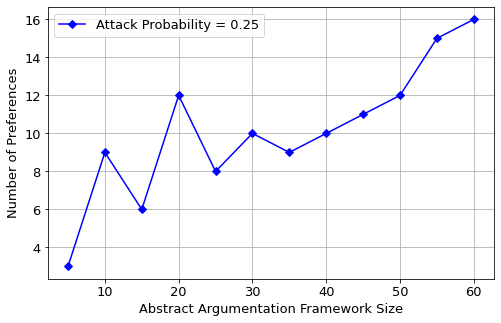}
	\caption{Average number of preferences in each preference set}
	\label{fig:graph3_grounded_computing_approx_scaled}
\end{subfigure}
     \label{fig:graph_grounded_computing_approx_scaled}
\end{figure}

\begin{figure}[H]
	\centering
    \caption{Verifying Preferences of the Grounded Extension for larger AAF Sizes (Approximate Algorithm)}
    \begin{subfigure}[b]{0.60\textwidth}
	\centering
	\includegraphics[width=\textwidth]{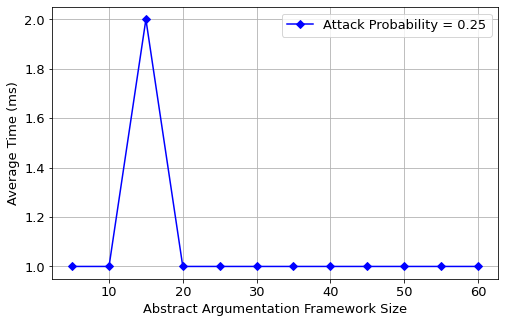}
	\caption{Average time in milliseconds for verifying all preference sets (attack removal)}
	\label{fig:graph_grounded_verification1_approx_scaled}
    \end{subfigure}
\hfill  
    \begin{subfigure}[b]{0.60\textwidth}
	\centering
	\includegraphics[width=\textwidth]{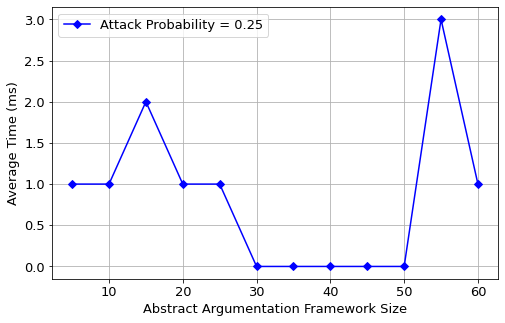}
	\caption{Average time in milliseconds for verifying all preference sets (attack reversal)}
	\label{fig:graph_grounded_verification2_approx_scaled}
    \end{subfigure}	
    \label{fig:graph_grounded_verification_approx_scaled}
\end{figure}

\begin{figure}[H]
	\centering
    \caption{Computing Preferences of the Preferred Extension for larger AAF Sizes (Approximate Algorithm)}	
    \begin{subfigure}[b]{.6\textwidth}
	\centering	
	\includegraphics[width=\textwidth]{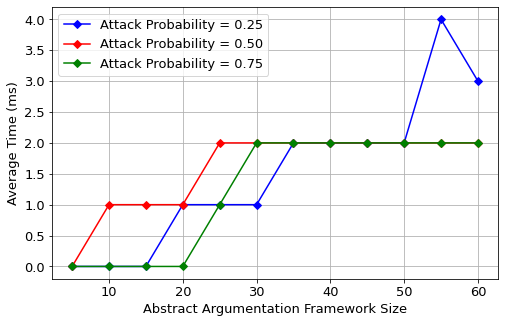}
	\caption{Average time in milliseconds for computing all preference sets}
	\label{fig:graph1_preferred_computing_approx_scalable}
	\end{subfigure}
 \hfill
    \begin{subfigure}[b]{.6\textwidth}
	\centering
	\includegraphics[width=\textwidth]{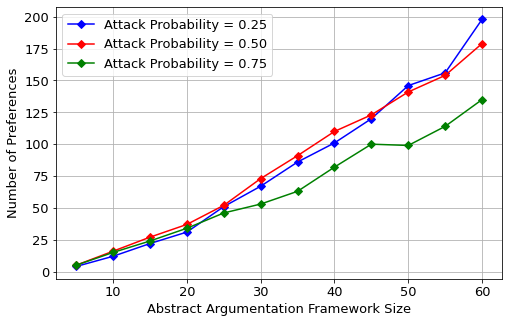}
	\caption{Average number of preferences in each preference set}
	\label{fig:graph3_preferred_computing_approx_scalable}
    \end{subfigure}
    \label{fig:graph_preferred_computing_approx_scalable}
\end{figure}

\begin{figure}[H]
	\centering
    \caption{Verifying Preferences of the Preferred Extension for larger AAF Sizes (Approximate Algorithm)}
    \begin{subfigure}[b]{.6\textwidth}
	\centering
	\includegraphics[width=\textwidth]{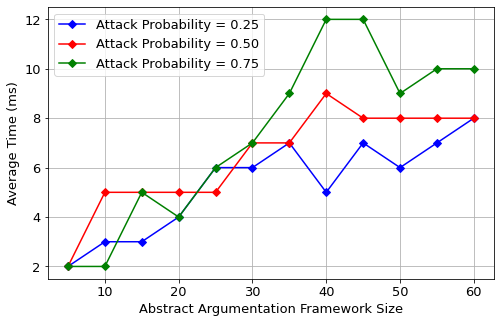}
	\caption{Average time in milliseconds for verifying all preference sets (attack removal)}
	\label{fig:graph_preferred_verification1_approx_scaled}
    \end{subfigure}
\hfill
    \begin{subfigure}[b]{.6\textwidth}
	\centering
	\includegraphics[width=\textwidth]{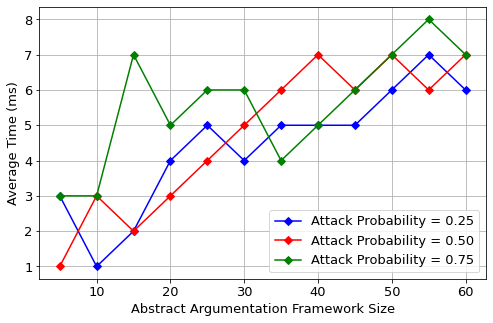}
	\caption{Average time in milliseconds for verifying all preference sets (attack reversal)}
	\label{fig:graph_preferred_verification2_approx_scaled}
    \end{subfigure}    
    \label{fig:graph_preferred_verification_approx_scaled}	
\end{figure}

\begin{figure}[H]
	\centering
    \caption{Computing Preferences of the Stable Extension for larger AAF Sizes (Approximate Algorithm)}
\begin{subfigure}[b]{.6\textwidth}
		\includegraphics[width=\linewidth]{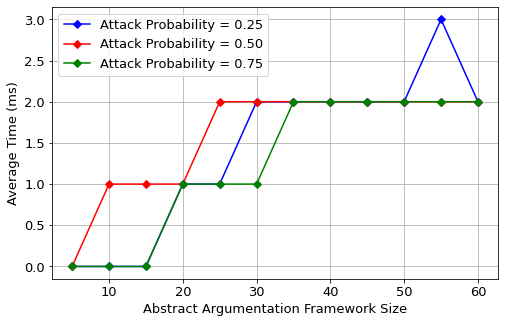}	  
    	\caption{Average time in milliseconds for computing all preference sets}
	\label{fig:graph1_stable_computing_approx_scaled}
	\end{subfigure}
\hfill 
\begin{subfigure}[b]{.6\textwidth}
	\centering
	\includegraphics[width=\textwidth]{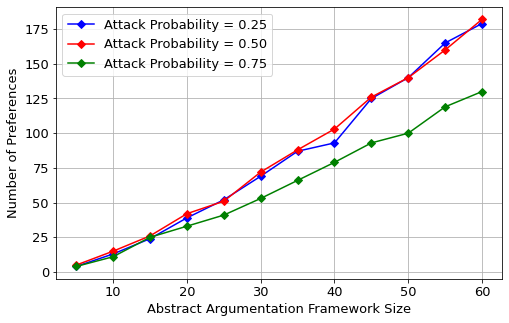}
	\caption{Average number of preferences in each preference set}
	\label{fig:graph3_stable_computing_approx_scaled}
\end{subfigure}
    \label{fig:graph_stable_computing_approx_scaled}
\end{figure}

\begin{figure}[H]
	\centering
    \caption{Verifying Preferences of the Stable Extension for larger AAF Sizes (Approximate Algorithm)}
    \begin{subfigure}[b]{.6\textwidth}
	\centering
	\includegraphics[width=\textwidth]{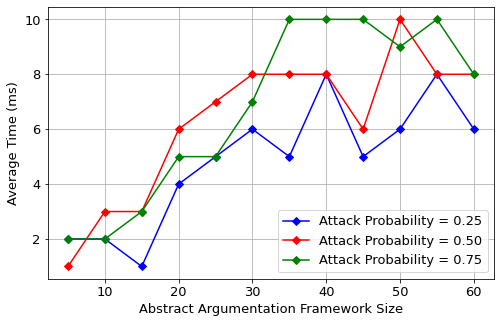}
	\caption{Average time in milliseconds for verifying all preference sets (attack removal)}
	\label{fig:graph_stable_verification1_approx_scaled}
    \end{subfigure}
\hfill  
    \begin{subfigure}[b]{.6\textwidth}
	\centering
	\includegraphics[width=\textwidth]{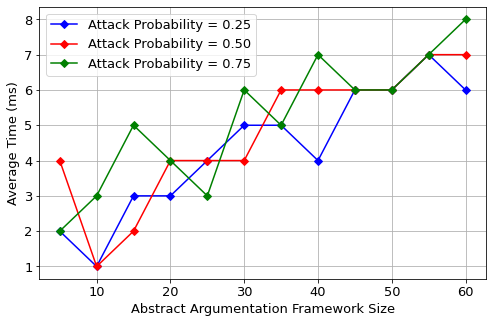}
	\caption{Average time in milliseconds for verifying all preference sets (attack reversal)}
	\label{fig:graph_stable_verification2_approx_scaled}
    \end{subfigure}	
    \label{fig:graph_stable_verification_approx_scaled}
\end{figure}

\section{Conclusions and Future Work}
\label{sec:conclusion}
In this paper we have described a novel extension-based approach to compute and verify (i.e., assess) abstract argument preferences. We have presented a novel algorithm that takes an abstract argumentation framework and a set of conflict-free arguments (extension) as input and computes all possible sets of preferences (restricted to three identified cases) that are valid for the acceptability of the arguments in the input extension. We have shown that the complexity of computing sets of preferences is exponential in the number of arguments, and thus, describe a novel approximate approach and algorithm to compute the preferences that is scalable. Furthermore, we have presented novel algorithms for verifying the computed sets of preferences, ensuring their validity. We have experimentally shown that both approaches, i.e., attack removal and reversal for the verification, output the desired input extension. We have implemented all the algorithms and have build a complete system for computing and verifying preferences. We have performed various experiments for the grounded, preferred and stable semantics with different attack probabilities in the abstract argumentation frameworks to evaluate the algorithms and performed an analysis of the results obtained.

This work has applications in decision support systems~\cite{Sprague:1993:DSS:167095,Mahesar17,Wei20-eswa} and recommender systems~\cite{Ricci:2010:RSH:1941884}, where the resulting decision(s) or recommendation(s) can be justified by the preference set(s). Another application would be to explore our work in dialogue strategies~\cite{ki-Thimm14,ijcai-RienstraTO13}, for instance in computational persuasion~\cite{argcom-Hunter18,HadouxHP23} or negotiation~\cite{RahwanRJMPS03} -- where an agent may have the capability of inferring preferences and reach her goal if (s)he enforces at least one of several desired sets of arguments with the application of preferences. The inferred preferences, in particular the unique and common preferences can be exploited in optimizing the choice of move in persuasion dialogues for behaviour change as well as in negotiation dialogues to reach agreement.

As future work, we plan to investigate different ways to aggregate and assess the sets of preferences. Additionally, we also plan to do an empirical evaluation of our proposed work on concrete examples. This will allow us to filter sets of preferences, i.e., to accept or reject them; or to rank the sets of preferences by human participants. Furthermore, we plan to extend our approach to value-based argumentation frameworks~\cite{bench-capon03-persuasion}, where it would be interesting to elicit values that the arguments promote or support to determine preferences over arguments. An extension of our work~\cite{prima-Mahesar18} to assumption-based argumentation frameworks ABA~\cite{aba_Dung2009,aba_Toni14} and ABA\textsuperscript{+}~\cite{CyrasT16} has been presented in~\cite{prima-Mahesar20}, where we were interested in eliciting preferences at the assumption level, however in the future, we aim to go beyond this by exploring other structured argumentation frameworks~\cite{Besnard2014-BESITS,ModgilP14}, in particular ASPIC\textsuperscript{+}~\cite{ModgilP14}. 
Furthermore, we intend to investigate the relationship between extension enforcement~\cite{ecai-Baumann12,CMKMM2015} and our work. 

\section*{Acknowledgements}
\label{acks}
\smallskip
\noindent This work was partially supported by the EPSRC, under grant EP/P011829/1, \emph{Supporting Security Policy with Effective Digital Intervention (SSPEDI)}.

\bibliographystyle{elsarticle-num}
\bibliography{elsarticle-main-num}

\begin{thebibliography}{10}
\expandafter\ifx\csname url\endcsname\relax
  \def\url#1{\texttt{#1}}\fi
\expandafter\ifx\csname urlprefix\endcsname\relax\def\urlprefix{URL }\fi
\expandafter\ifx\csname href\endcsname\relax
  \def\href#1#2{#2} \def\path#1{#1}\fi

\bibitem{Pigozzi2016}
G.~Pigozzi, A.~Tsouki{\`a}s, P.~Viappiani, Preferences in artificial
  intelligence, Annals of Mathematics and Artificial Intelligence 77~(3) (2016)
  361--401.

\bibitem{Walsh07}
T.~Walsh, Representing and reasoning with preferences, {AI} Magazine 28~(4)
  (2007) 59--70.

\bibitem{Konczak05votingprocedures}
K.~Konczak, Voting procedures with incomplete preferences, in: in Proc.
  IJCAI-05 Multidisciplinary Workshop on Advances in Preference Handling, 2005.

\bibitem{PINI20111272}
M.~Pini, F.~Rossi, K.~Venable, T.~Walsh, Incompleteness and incomparability in
  preference aggregation: Complexity results, Artificial Intelligence 175~(7)
  (2011) 1272 -- 1289.

\bibitem{Sprague:1993:DSS:167095}
R.~H. Sprague, Jr., H.~J. Watson (Eds.), Decision Support Systems (3rd Ed.):
  Putting Theory into Practice, Prentice-Hall, Inc., Upper Saddle River, NJ,
  USA, 1993.

\bibitem{Mahesar17}
Q.~Mahesar, V.~Dimitrova, D.~R. Magee, A.~G. Cohn, Uncertainty management for
  rule-based decision support systems, in: 29th {IEEE} International Conference
  on Tools with Artificial Intelligence, {ICTAI} 2017, Boston, MA, USA,
  November 6-8, 2017, {IEEE} Computer Society, 2017, pp. 884--891.

\bibitem{Ricci:2010:RSH:1941884}
F.~Ricci, L.~Rokach, B.~Shapira, P.~B. Kantor, Recommender Systems Handbook,
  1st Edition, Springer-Verlag, 2010.

\bibitem{ki-Thimm14}
M.~Thimm, Strategic argumentation in multi-agent systems, {KI} 28~(3) (2014)
  159--168.

\bibitem{ijcai-RienstraTO13}
T.~Rienstra, M.~Thimm, N.~Oren, Opponent models with uncertainty for strategic
  argumentation, in: Proceedings of the Twenty-Third International Joint
  Conference on Artificial Intelligence, {IJCAI/AAAI}, 2013, pp. 332--338.

\bibitem{argcom-Hunter18}
A.~Hunter, Towards a framework for computational persuasion with applications
  in behaviour change, Argument Comput. 9~(1) (2018) 15--40.

\bibitem{HadouxHP23}
E.~Hadoux, A.~Hunter, S.~Polberg, Strategic argumentation dialogues for
  persuasion: Framework and experiments based on modelling the beliefs and
  concerns of the persuadee, Argument Comput. 14~(2) (2023) 109--161.

\bibitem{RahwanRJMPS03}
I.~Rahwan, S.~D. Ramchurn, N.~R. Jennings, P.~McBurney, S.~Parsons,
  L.~Sonenberg, Argumentation-based negotiation, Knowl. Eng. Rev. 18~(4) (2003)
  343--375.

\bibitem{BESNARD2001203}
P.~Besnard, A.~Hunter, A logic-based theory of deductive arguments, Artificial
  Intelligence 128~(1) (2001) 203 -- 235.

\bibitem{Garcia:2004}
A.~J. Garc\'{\i}a, G.~R. Simari, Defeasible logic programming: An argumentative
  approach, Theory Pract. Log. Program. 4~(2) (2004) 95--138.

\bibitem{SIMARI1992125}
G.~R. Simari, R.~P. Loui, A mathematical treatment of defeasible reasoning and
  its implementation, Artificial Intelligence 53~(2) (1992) 125 -- 157.

\bibitem{AMGOUD2009413}
L.~Amgoud, H.~Prade, Using arguments for making and explaining decisions,
  Artificial Intelligence 173~(3) (2009) 413 -- 436.

\bibitem{Bonet:1996}
B.~Bonet, H.~Geffner, Arguing for decisions: A qualitative model of decision
  making, in: Proceedings of the Twelfth International Conference on
  Uncertainty in Artificial Intelligence, UAI'96, Morgan Kaufmann Publishers
  Inc., 1996, pp. 98--105.

\bibitem{Muller:2012}
J.~Muller, A.~Hunter, An argumentation-based approach for decision making, in:
  Proceedings of the 2012 IEEE 24th International Conference on Tools with
  Artificial Intelligence - Volume 01, ICTAI '12, IEEE Computer Society, 2012,
  pp. 564--571.

\bibitem{Dung95onthe}
P.~M. Dung, On the acceptability of arguments and its fundamental role in
  nonmonotonic reasoning, logic programming and n-person games, Artificial
  Intelligence 77 (1995) 321--357.

\bibitem{Amgoud:1998}
L.~Amgoud, C.~Cayrol, On the acceptability of arguments in preference-based
  argumentation, in: Proceedings of the Fourteenth Conference on Uncertainty in
  Artificial Intelligence, UAI'98, Morgan Kaufmann Publishers Inc., San
  Francisco, CA, USA, 1998, pp. 1--7.

\bibitem{MODGIL2009901}
S.~Modgil, Reasoning about preferences in argumentation frameworks, Artificial
  Intelligence 173~(9) (2009) 901 -- 934.

\bibitem{Prakken1997ArgumentBasedEL}
H.~Prakken, G.~Sartor, Argument-based extended logic programming with
  defeasible priorities, Journal of Applied Non-Classical Logics 7 (1997)
  25--75.

\bibitem{Amgoud2002}
L.~Amgoud, C.~Cayrol, A reasoning model based on the production of acceptable
  arguments, Annals of Mathematics and Artificial Intelligence 34~(1) (2002)
  197--215.

\bibitem{AMGOUD2014585}
L.~Amgoud, S.~Vesic, Rich preference-based argumentation frameworks,
  International Journal of Approximate Reasoning 55~(2) (2014) 585 -- 606.

\bibitem{KaciT08}
S.~Kaci, L.~van~der Torre, Preference-based argumentation: Arguments supporting
  multiple values, Int. J. Approx. Reasoning 48~(3) (2008) 730--751.

\bibitem{prima-Mahesar18}
Q.~Mahesar, N.~Oren, W.~W. Vasconcelos, Computing preferences in abstract
  argumentation, in: {PRIMA} 2018: Principles and Practice of Multi-Agent
  Systems - 21st International Conference, Tokyo, Japan, October 29 - November
  2, 2018, Proceedings, Vol. 11224 of Lecture Notes in Computer Science,
  Springer, 2018, pp. 387--402.

\bibitem{bench-capon03-persuasion}
T.~J.~M. Bench-Capon, Persuasion in practical argument using value-based
  argumentation frameworks, Journal of Logic and Computation 13~(3) (2003)
  429--448.

\bibitem{CAMINADA2007286}
M.~Caminada, L.~Amgoud, On the evaluation of argumentation formalisms,
  Artificial Intelligence 171~(5) (2007) 286 -- 310.

\bibitem{Amgoud2011}
L.~Amgoud, S.~Vesic, A new approach for preference-based argumentation
  frameworks, Annals of Mathematics and Artificial Intelligence 63~(2) (2011)
  149--183.

\bibitem{CyrasT16}
K.~Cyras, F.~Toni, {ABA+:} assumption-based argumentation with preferences, in:
  Principles of Knowledge Representation and Reasoning: Proceedings of the
  Fifteenth International Conference, {KR}, 2016, pp. 553--556.

\bibitem{Wakaki17}
T.~Wakaki, Assumption-based argumentation equipped with preferences and its
  application to decision making, practical reasoning, and epistemic reasoning,
  Comput. Intell. 33~(4) (2017) 706--736.

\bibitem{ModgilP14}
S.~Modgil, H.~Prakken, The \emph{ASPIC}\({}^{\mbox{+}}\) framework for
  structured argumentation: a tutorial, Argument {\&} Computation 5~(1) (2014)
  31--62.

\bibitem{BesnardH14}
P.~Besnard, A.~Hunter, Constructing argument graphs with deductive arguments: a
  tutorial, Argument {\&} Computation 5~(1) (2014) 5--30.

\bibitem{prima-Mahesar20}
Q.~Mahesar, N.~Oren, W.~W. Vasconcelos, Preference elicitation in
  assumption-based argumentation, in: {PRIMA} 2020: Principles and Practice of
  Multi-Agent Systems - 23rd International Conference, Nagoya, Japan, November
  18-20, 2020, Proceedings, Vol. 12568 of Lecture Notes in Computer Science,
  Springer, 2020, pp. 199--214.

\bibitem{aba_Dung2009}
P.~M. Dung, R.~A. Kowalski, F.~Toni, Assumption-Based Argumentation, Springer
  US, 2009.

\bibitem{aba_Toni14}
F.~Toni, A tutorial on assumption-based argumentation, Argument {\&}
  Computation 5~(1) (2014) 89--117.

\bibitem{ijar-HungH21}
N.~D. Hung, V.~Huynh, Revealed preference in argumentation: Algorithms and
  applications, Int. J. Approx. Reason. 131 (2021) 214--251.

\bibitem{BENFERHAT1993411}
S.~Benferhat, D.~Dubois, H.~Prade, Argumentative inference in uncertain and
  inconsistent knowledge bases, in: D.~Heckerman, A.~Mamdani (Eds.),
  Uncertainty in Artificial Intelligence, Morgan Kaufmann, 1993, pp. 411 --
  419.

\bibitem{Cayrol-Royer-Saurel}
C.~Cayrol, V.~Royer, C.~Saurel, Management of preferences in assumption-based
  reasoning, in: B.~Bouchon-Meunier, L.~Valverde, R.~R. Yager (Eds.), IPMU
  '92---Advanced Methods in Artificial Intelligence, Springer Berlin
  Heidelberg, 1993, pp. 13--22.

\bibitem{Thimm:2017e}
M.~Thimm, The formal argumentation libraries of tweety, in: Proc. of the
  International Workshop on Theory and Applications of Formal Argument, 2017.

\bibitem{Wei20-eswa}
L.~Wei, H.~Du, Q.~Mahesar, K.~A. Ammari, D.~R. Magee, B.~Clarke, V.~Dimitrova,
  D.~Gunn, D.~Entwisle, H.~Reeves, A.~G. Cohn, A decision support system for
  urban infrastructure inter-asset management employing domain ontologies and
  qualitative uncertainty-based reasoning, Expert Syst. Appl. 158 (2020)
  113461.

\bibitem{Besnard2014-BESITS}
P.~Besnard, A.~Garcia, A.~Hunter, S.~Modgil, H.~Prakken, G.~Simari, F.~Toni,
  Introduction to structured argumentation, Argument and Computation 5~(1)
  (2014) 1--4.

\bibitem{ecai-Baumann12}
R.~Baumann, What does it take to enforce an argument? minimal change in
  abstract argumentation, in: Proceedings of the 20th European Conference on
  Artificial Intelligence, 2012, pp. 127--132.

\bibitem{CMKMM2015}
S.~Coste-Marquis, S.~Konieczny, J.-G. Mailly, P.~Marquis, Extension enforcement
  in abstract argumentation as an optimization problem, in: Proceedings of the
  Twenty-Fourth International Joint Conference on Artificial Intelligence,
  2015, pp. 2876--2882.

\end{thebibliography}

\newpage 

% \section*{Appendix}
% \label{sec:appendix}
% \addcontentsline{toc}{section}{Appendix}
%

\appendix

\section{Preference Sets}
\label{sec:appendix-a}

\begin{table}[!h]
	\caption{Preference sets for all conflict-free extensions}
	\label{table:conflictfree-preferences}
	\centering\renewcommand\arraystretch{1.2}
	\small
	\begin{tabular}{|c|p{70mm}|}
		\hline
		Conflict-free Extensions & Preference Sets  \\ \hline
		$\{A,C,E\}$ &  
		$\{ \{C>D, A>B, C>B, E>D\}$ \newline $\{C>D, A>B, C>B, E=D\}$ \newline $\{C>D, A>B, C>B, D>E\}$ 
		\newline
		$\{C>D, A>B, C=B, E>D\}$ \newline $\{C>D, A>B, C=B, E=D\}$ \newline $\{C>D, A>B, C=B, D>E\}$ 
		\newline
		$\{C>D, A=B, C>B, E>D\}$ \newline $\{C>D, A=B, C>B, E=D\}$ \newline $\{C>D, A=B, C>B, D>E\}$ 
		\newline
		$\{C>D, A=B, C=B, E>D\}$ \newline $\{C>D, A=B, C=B, E=D\}$ \newline $\{C>D, A=B, C=B, D>E\} \}$ 		
		\\
		\hline
		$\{A,D\}$ & 	 $\{ \{ D>C, A>B, D>E \}$
		\newline $\{ D>C, A>B, D=E \}$
		\newline $\{ D>C, A=B, D>E \}$
		\newline $\{ D>C, A=B, D=E \} \}$ 
		\\
		\hline
		$\{B,D\}$ & $\{ \{B>A, B>C, D>C, D>E\}$
		\newline $\{B>A, B>C, D>C, D=E\} \}$		
		\\
		\hline
		$\{A,C\}$ & 
		$\{ \{C>D, A>B, C>B\}$
		\newline 
		$\{C>D, A>B, C=B\}$
		\newline 		
		$\{C>D, A=B, C>B\}$			
		\newline 		
		$\{C>D, A=B, C=B\} \}$			
		\\
		\hline
		$\{A,E\}$ & $\{ \{E>D, A>B\}$
		\newline $\{E>D, A=B\} \}$		
		\\
		\hline
		$\{B,E\}$ & $\{ \{B>A, B>C, E>D\} \}$	
		\\
		\hline
		$\{C,E\}$ & $\{ \{C>D, C>B, E>D\}$
		\newline $\{C>D, C>B, E=D\}$	
		\newline $\{C>D, C>B, D>E\}$
		\newline $\{C>D, C=B, E>D\}$
		\newline $\{C>D, C=B, E=D\}$		
		\newline $\{C>D, C=B, D>E\} \}$
		\\
		\hline	
		$\{A\}$ & $\{ \{A>B\}$
		\newline $\{A=B\} \}$
		\\
		\hline	
		$\{B\}$ & $\{ \{B>A, B>C\} \}$
		\\
		\hline				
		$\{C\}$ & $\{ \{C>D, C>B\}$
		\newline $\{C>D, C=B\} \}$
		\\
		\hline			
		$\{D\}$ & $\{ \{D>C, D>E\}$
		\newline $\{D>C, D=E\} \}$
		\\
		\hline			
		$\{E\}$ & $\{ \{E>D\} \}$
		\\
		\hline			
		$\emptyset$ & $\emptyset$ 
		\\
		\hline
	\end{tabular}
\end{table}

\newpage 

\section{Data sets generated for the Experimental Analysis}
\label{sec:appendix-b}
 
We present below all the data sets in the form of tables that were generated to conduct the experimental analysis presented in Section~\ref{sec:experimental_evaluation}.
The abbreviations used in the titles of columns in the following tables are given as follows:
\begin{itemize}
    \item \textbf{AAF Size:} Abstract Argumentation Framework Size.
    \item \textbf{Ext Size:} Extension Size.
    \item \textbf{Attacks:} Number of Attacks.
    \item \textbf{Preference Sets:} Number of Preference Sets.
    \item \textbf{Preferences:} Number of Preferences.
    \item \textbf{CTime:} Time for Computing all Preference Sets in milliseconds.
    \item \textbf{VTime1:} Time for Verifying all Preference Sets in milliseconds, using Attack Removal Method.
    \item \textbf{VTime2:} Time for Verifying all Preference Sets in milliseconds, using Attack Reversal Method.
\end{itemize}

\begin{center}
    \large{\textbf{Data sets for the Original Algorithm}}
\end{center}

\begin{table}[H]
    \caption{Grounded Extension with Attack Probability $0.25$}
	\label{table:grounded_25}
	\centering\renewcommand\arraystretch{1.2}
\begin{tabular}{ |p{1cm}|p{1cm}|p{1.5cm}|p{2cm}| p{2cm}|p{1.5cm}|p{1.5cm}|p{1.5cm}| }
%  \hline
%  \multicolumn{8}{|c|}{Grounded Extension with Attack Probability $0.25$} \\
 \hline
AAF Size & Ext Size & Attacks & Preference Sets & Preferences & CTime (ms) & VTime$1$ (ms) & VTime$2$ (ms)\\
 \hline
4 & 2 & 3 &	8 &	2 &	0 &	3 &	4\\
\hline
5 &	3 &	3 &	12 &	2 &	1 &	4 &	5 \\
\hline 
6 &	3 &	6 &	51 &	4 &	3 &	7 &	8 \\
\hline 
7 &	2 &	10 &	159 &	5 &	4 &	11 &	14\\
\hline 
8 &	2 &	12 &	242 &	4 &	6 &	17 &	18 \\
\hline 
9 &	4 &	16 &	1726 &	9 &	27 &	65 &	76\\
\hline 
10 &	3 &	22 &	19354 &	9 &	1452 &	497 &	663\\
\hline 
11 &	2 &	27 &	2500 &	5 &	130 &	111 &	118\\
\hline 
12 &	3 &	30 &	160976 &	8 &	102678 &	5589 &	7450\\
\hline 
\end{tabular}
\end{table}

\vspace*{1 cm}

\begin{table}[H]
\caption{Grounded Extension with Attack Probability $0.50$}
	\label{table:grounded_50}
	\centering\renewcommand\arraystretch{1.2}
\begin{tabular}{ |p{1cm}|p{1cm}|p{1.5cm}|p{2cm}| p{2cm}|p{1.5cm}|p{1.5cm}|p{1.5cm}| }
%  \hline
%  \multicolumn{8}{|c|}{Grounded Extension with Attack Probability $0.50$} \\
 \hline
AAF Size & Ext Size & Attacks & Preference Sets & Preferences & CTime (ms) & VTime$1$ (ms) & VTime$2$ (ms)\\
 \hline
4 &	2 &	5 &	18 &	3 &	1 &	4 &	7\\
\hline 
5 &	2 &	8 &	21 &	3 &	1 &	4 &	8\\
\hline 
6 &	2 &	13 &	408 &	6 &	7 &	27 &	19\\
\hline 
7 &	2 &	19 &	1869 &	7 &	40 &	61 &	69\\
\hline 
8 &	2 &	25 &	2279 &	7 &	40 &	80 &	71\\
\hline 
9 &	2 &	29 &	50088 &	7 &	9163 &	1361 &	1789\\
\hline 
10 &	2 &	41 &	113597 &	9 &	24540 &	3574 &	4461\\
\hline 
11 &	2 &	52 &	352546 &	8 &	1689643 &	10334 &	13022\\
\hline 
12 &	1 &	60 &	2190 &	7 &	59 &	120 &	124\\

\hline 
\end{tabular}
\end{table}

\vspace*{1 cm}

\begin{table}[H]
    \caption{Grounded Extension with Attack Probability $0.75$}
	\label{table:grounded_75}
	\centering\renewcommand\arraystretch{1.2}
\begin{tabular}{ |p{1cm}|p{1cm}|p{1.5cm}|p{2cm}| p{2cm}|p{1.5cm}|p{1.5cm}|p{1.5cm}| }
%  \hline
%  \multicolumn{8}{|c|}{Grounded Extension with Attack Probability $0.75$} \\
 \hline
AAF Size & Ext Size & Attacks & Preference Sets & Preferences & CTime (ms) & VTime$1$ (ms) & VTime$2$ (ms)\\
 \hline
4 &	2 &	7 &	19 &	3 &	1 &	5 &	7\\
\hline 
5 &	1 &	13 &	52 &	4 &	2 &	7 &	9\\
\hline 
6 &	2 &	19 &	509 &	6 &	9 &	32 &	24\\
\hline 
7 &	2 &	28 &	10506 &	8 &	481 &	273 &	315\\
\hline 
8 &	2 &	36 &	11539 &	9 &	336 &	358 &	452\\
\hline 
9 &	2 &	49 &	89555 &	10 &	9906 &	3245 &	3760\\
\hline 
10 &	2 &	61 &	308365 &	12 &	187235 &	9002 &	11324\\
\hline 
11 &	1 &	75 &	559968 &	11 &	709698 &	19996 &	24010\\
\hline 
12 &	1 &	87 &	88733 &	9 &	537216 &	4400 &	5438\\
\hline 
\end{tabular}
\end{table}

\begin{table}[H]
    \caption{Preferred Extension with Attack Probability $0.25$}
	\label{table:preferred_25}
	\centering\renewcommand\arraystretch{1.2}
\begin{tabular}{ |p{1cm}|p{1cm}|p{1.5cm}|p{2cm}| p{2cm}|p{1.5cm}|p{1.5cm}|p{1.5cm}| }
%  \hline
%  \multicolumn{8}{|c|}{Preferred Extension with Attack Probability $0.25$} \\
 \hline
AAF Size & Ext Size & Attacks & Preference Sets & Preferences & CTime (ms) & VTime$1$ (ms) & VTime$2$ (ms)\\
\hline
4 &	2 &	4 &	2 &	2 &	0 &	1 &	3\\
\hline 
5 &	2 &	7 &	4 &	4 &	1 &	5 &	5\\
\hline 
6 &	3 &	8 &	8 &	5 &	1 &	10 &	8\\
\hline 
7 &	3 &	12 &	25 &	7 &	3 &	9 &	11\\
\hline 
8 &	3 &	15 &	437 &	8 &	12 &	53 &	72\\
\hline 
9 &	3 &	20 &	380 &	10 &	7 &	53 &	96\\
\hline 
10 &	4 &	21 &	837 &	12 &	13 &	89 &	143\\
\hline 
11 &	4 &	27 &	1008 &	14 &	28 &	122 &	183\\
\hline 
12 &	4 &	32 &	9613 &	15 &	1246 &	843 &	1311\\
\hline 
13 &	4 &	41 &	10736 &	18 &	501 &	971 &	1608\\
\hline 
14 &	4 &	49 &	5389 &	21 &	774 &	581 &	885\\
\hline 
15 &	4 &	52 &	20442 &	24 &	3554 &	2059 &	3677\\
\hline 
16 &	5 &	60 &	20480 &	26 &	24654 &	1908 &	3548\\
\hline 
\end{tabular}
\end{table}

\vspace*{1 cm}

\begin{table}[H]
    \caption{Preferred Extension with Attack Probability $0.50$}
	\label{table:preferred_50}
	\centering\renewcommand\arraystretch{1.2}
\begin{tabular}{ |p{1cm}|p{1cm}|p{1.5cm}|p{2cm}| p{2cm}|p{1.5cm}|p{1.5cm}|p{1.5cm}| }
%  \hline
%  \multicolumn{8}{|c|}{Preferred Extension with Attack Probability $0.50$} \\
 \hline
AAF Size & Ext Size & Attacks & Preference Sets & Preferences & CTime (ms) & VTime$1$ (ms) & VTime$2$ (ms)\\
\hline
4 &	2 &	7 &	3 &	3 &	0 &	2 &	6\\
\hline
5 &	2 &	10 &	6 &	5 &	1 &	6 &	9\\
\hline
6 &	2 &	15 &	7 &	7 &	1 &	8 &	12\\
\hline
7 &	2 &	22 &	10 &	9 &	1 &	11 &	10\\
\hline
8 &	2 &	28 &	17 &	11 &	1 &	7 &	13\\
\hline
9 &	2 &	38 &	60 &	12 &	3 &	15 &	21\\
\hline
10 &	3 &	45 &	102 &	16 &	5 &	23 &	25\\
\hline
11 &	2 &	56 &	682 &	17 &	32 &	85 &	117\\
\hline
12 &	2 &	66 &	204 &	17 &	8 &	50 &	42\\
\hline
13 &	2 &	81 &	210 &	20 &	9 &	43 &	50\\
\hline
14 &	3 &	90 &	3189 &	25 &	399 &	344 &	584\\
\hline
15 &	3 &	105 &	1997 &	27 &	129 &	223 &	353\\
\hline
16 &	2 &	122 &	6886 &	26 &	5190 &	692 &	1276\\
\hline
\end{tabular}
\end{table}

\qquad

\begin{table}[H]
    \caption{Preferred Extension with Attack Probability $0.75$}
	\label{table:preferred_75}
	\centering\renewcommand\arraystretch{1.2}
\begin{tabular}{ |p{1cm}|p{1cm}|p{1.5cm}|p{2cm}| p{2cm}|p{1.5cm}|p{1.5cm}|p{1.5cm}| }
%  \hline
%  \multicolumn{8}{|c|}{Preferred Extension with Attack Probability $0.75$} \\
 \hline
AAF Size & Ext Size & Attacks & Preference Sets & Preferences & CTime (ms) & VTime$1$ (ms) & VTime$2$ (ms)\\
\hline
4 &	1 &	9 &	2 &	3 &	1 &	3 &	5\\
\hline 
5 &	1 &	15 &	3 &	5 &	0 &	4 &	5\\
\hline 
6 &	1 &	23 &	3 &	6 &	0 &	3 &	7\\
\hline 
7 &	1 &	32 &	7 &	7 &	1 &	5 &	9\\
\hline 
8 &	1 &	42 &	9 &	9 &	1 &	9 &	7\\
\hline 
9 &	2 &	54 &	9 &	11 &	1 &	10 &	7\\
\hline 
10 &	2 &	67 &	14 &	12 &	2 &	17 &	7\\
\hline 
11 &	2 &	81 &	41 &	16 &	3 &	14 &	21\\
\hline 
12 &	2 &	99 &	9 &	16 &	1 &	7 &	15\\
\hline 
13 &	2 &	115 &	40 &	21 &	2 &	19 &	17\\
\hline 
14 &	2 &	134 &	161 &	21 &	6 &	55 &	44\\
\hline 
15 &	2 &	157 &	106 &	26 &	5 &	35 &	41\\
\hline 
16 &	2 &	184 &	48 &	28 &	3 &	20 &	33\\
\hline 
\end{tabular}
\end{table}

\qquad

\begin{table}[H]
    \caption{Stable Extension with Attack Probability $0.25$}
	\label{table:stable_25}
	\centering\renewcommand\arraystretch{1.2}
\begin{tabular}{ |p{1cm}|p{1cm}|p{1.5cm}|p{2cm}| p{2cm}|p{1.5cm}|p{1.5cm}|p{1.5cm}| }
%  \hline
%  \multicolumn{8}{|c|}{Preferred Extension with Attack Probability $0.25$} \\
 \hline
AAF Size & Ext Size & Attacks & Preference Sets & Preferences & CTime (ms) & VTime$1$ (ms) & VTime$2$ (ms)\\
\hline
4 &	2 &	4 &	2 &	2 &	0 &	3 &	1\\
\hline
5 &	3 &	6 &	4 &	4 &	1 &	4 &	4\\
\hline
6 &	3 &	10 &	17 &	6 &	1 &	12 &	10\\
\hline
7 &	3 &	11 &	40 &	7 &	2 &	15 &	14\\
\hline
8 &	4 &	14 &	346 &	9 &	7 &	42 &	62\\
\hline
9 &	4 &	19 &	144 &	11 &	5 &	30 &	29\\
\hline
10 &	4 &	22 &	6426 &	13 &	330 &	483 &	720\\
\hline
11 &	4 &	29 &	715 &	14 &	18 &	104 &	122\\
\hline
12 &	4 &	37 &	7910 &	17 &	382 &	686 &	1161\\
\hline
13 &	4 &	39 &	17856 &	19 &	2952 &	1545 &	2463\\
\hline
14 &	5 &	48 &	14598 &	23 &	1992 &	1425 &	2334\\
\hline
15 &	5 &	52 &	73677 &	23 &	47656 &	5104 &	10876\\
\hline
16 &	5 &	64 &	27853 &	26 &	32881 &	20035 &	5970\\
\hline
\end{tabular}
\end{table}

\qquad

\begin{table}[H]
    \caption{Stable Extension with Attack Probability $0.50$}
	\label{table:stable_50}
	\centering\renewcommand\arraystretch{1.2}
\begin{tabular}{ |p{1cm}|p{1cm}|p{1.5cm}|p{2cm}| p{2cm}|p{1.5cm}|p{1.5cm}|p{1.5cm}| }
%  \hline
%  \multicolumn{8}{|c|}{Preferred Extension with Attack Probability $0.25$} \\
 \hline
AAF Size & Ext Size & Attacks & Preference Sets & Preferences & CTime (ms) & VTime$1$ (ms) & VTime$2$ (ms)\\
\hline
4 &	2 &	7 &	3 &	3 &	0 &	3 &	7\\
\hline
5 &	2 &	10 &	3 &	5 &	0 &	3 &	5\\
\hline
6 &	2 &	15 &	5 &	6 &	1 &	6 &	11\\
\hline
7 &	2 &	21 &	13 &	7 &	2 &	10 &	10\\
\hline
8 &	2 &	30 &	17 &	10 &	1 &	9 &	12\\
\hline
9 &	2 &	37 &	75 &	13 &	4 &	22	18\\
\hline
10 &	2 &	46 &	2251 &	14 &	75 &	181 &	233\\
\hline
11 &	2 &	54 &	99 &	16 &	5 &	28 &	24\\
\hline
12 &	3 &	62 &	797 &	19 &	39 &	93 &	144\\
\hline
13 &	3 &	77 &	1825 &	22 &	291 &	176 &	337\\
\hline
14 &	3 &	92 &	819 &	24 &	37 &	102 &	142\\
\hline
15 &	3 &	104 &	5574 &	26 &	1382 &	584	 & 905\\
\hline
16 &	3 &	117 &	2726 &	29 &	189 &	328 &	645\\
\hline
\end{tabular}
\end{table}

\qquad

\begin{table}[H]
    \caption{Stable Extension with Attack Probability $0.75$}
	\label{table:stable_75}
	\centering\renewcommand\arraystretch{1.2}
\begin{tabular}{ |p{1cm}|p{1cm}|p{1.5cm}|p{2cm}| p{2cm}|p{1.5cm}|p{1.5cm}|p{1.5cm}| }
%  \hline
%  \multicolumn{8}{|c|}{Preferred Extension with Attack Probability $0.25$} \\
 \hline
AAF Size & Ext Size & Attacks & Preference Sets & Preferences & CTime (ms) & VTime$1$ (ms) & VTime$2$ (ms)\\
\hline
4 &	1 &	10 &	1 &	3 &	0 &	2 &	2\\
\hline
5 &	1 &	16 &	3 &	5 &	0 &	4 &	4\\
\hline
6 &	2 &	22 &	5 &	7 &	1 &	6 &	4\\
\hline
7 &	1 &	32 &	4 &	7 &	1 &	5 &	4\\
\hline
8 &	2 &	41 &	15 &	10 &	1 &	11 &	10\\
\hline
9 &	1 &	57 &	4 &	10 &	1 &	8 &	6\\
\hline
10 &	2 &	68 &	6 &	12 &	1 &	4 &	9\\
\hline
11 &	1 &	85 &	12 &	13 &	2 &	6 &	12\\
\hline
12 &	2 &	99 &	14 &	17 &	2 &	11 &	18\\
\hline
13 &	2 &	117 &	39 &	19 &	3 &	23 &	12\\
\hline
14 &	2 &	134 &	43 &	21 &	3 &	22 &	20\\
\hline
15 &	2 &	156 &	324 &	25 &	15 &	86 &	93\\
\hline
16 &	2 &	177 &	346 &	26 &	15 &	96 &	112\\
\hline
\end{tabular}
\end{table}

\newpage 

\begin{center}
    \large{\textbf{Data sets for the Approximate Algorithm}}
\end{center}

\begin{table}[H]
    \caption{Grounded Extension with Attack Probability $0.25$}
	\label{table:grounded_approx_25}
	\centering\renewcommand\arraystretch{1.2}
\begin{tabular}{ |p{1cm}|p{1cm}|p{1.5cm}|p{2cm}| p{2cm}|p{1.5cm}|p{1.5cm}|p{1.5cm}| }
 \hline
AAF Size & Ext Size & Attacks & Preference Sets & Preferences & CTime (ms) & VTime$1$ (ms) & VTime$2$ (ms)\\
 \hline
 4 &	2 &	3 &	1 &	2 &	0 &	0 &	0\\
\hline 
5 &	3 &	3 &	1 &	2 &	0 &	1 &	0\\
\hline 
6 &	3 &	6 &	1 &	4 &	0 &	0 &	0\\
\hline 
7 &	2 &	10 &	1 &	5 &	0 &	0 &	1\\
\hline 
8 &	2 &	12  &	1 &	4 &	0 &	0 &	1\\
\hline 
9 &	4 &	16 &	1 &	9 &	0 &	0 &	0\\
\hline 
10 &	3 &	22 &	1 &	9 &	0 &	0 &	0\\
\hline 
11 &	2 &	27 &	1 &	5 &	0 &	1 &	1\\
\hline 
12 &	3 &	30 &	1 &	8 &	0 &	0 &	1\\
\hline 
 
\end{tabular}
\end{table}

\qquad

\begin{table}[H]
    \caption{Grounded Extension with Attack Probability $0.50$}
	\label{table:grounded_approx_50}
	\centering\renewcommand\arraystretch{1.2}
\begin{tabular}{ |p{1cm}|p{1cm}|p{1.5cm}|p{2cm}| p{2cm}|p{1.5cm}|p{1.5cm}|p{1.5cm}| }
 \hline
AAF Size & Ext Size & Attacks & Preference Sets & Preferences & CTime (ms) & VTime$1$ (ms) & VTime$2$ (ms)\\
 \hline
4 &	2 &	5 &	1 &	3 &	0 &	1 &	0
\\
\hline 
5 &	2 &	8 &	1 &	3 &	0 &	0 &	1\\
\hline 
6 &	2 &	13 &	1 &	6 &	0 &	0 &	0\\
\hline 
7 &	2 &	19 &	1 &	7 &	0 &	0 &	0\\
\hline 
8 &	2 &	25 &	1 &	7 &	0 &	1 &	0\\
\hline 
9 &	2 &	29 &	1 &	7 &	0 &	0 &	0\\
\hline 
10 &	2 &	41 &	1 &	9 &	0 &	1 &	1\\
\hline 
11 &	2 &	52 &	1 &	8 &	0 &	1 &	1\\
\hline 
12 &	1 &	60 &	1 &	7 &	0 &	1 &	0\\
\hline 
\end{tabular}
\end{table}

\qquad

\begin{table}[H]
    \caption{Grounded Extension with Attack Probability $0.75$}
	\label{table:grounded_approx_75}
	\centering\renewcommand\arraystretch{1.2}
\begin{tabular}{ |p{1cm}|p{1cm}|p{1.5cm}|p{2cm}| p{2cm}|p{1.5cm}|p{1.5cm}|p{1.5cm}| }
 \hline
AAF Size & Ext Size & Attacks & Preference Sets & Preferences & CTime (ms) & VTime$1$ (ms) & VTime$2$ (ms)\\
 \hline
4 &	2 &	7 &	1 &	3 &	0 &	0 &	0\\
\hline 
5 &	1 &	13 &	1 &	4 &	0 &	0 &	0\\
\hline 
6 &	2 &	19 &	1 &	6 &	0 &	0 &	0\\
\hline 
7 &	2 &	28 &	1 &	8 &	0 &	0 &	0\\
\hline 
8 &	2 &	36 &	1 &	9 &	0 &	0 &	1\\
\hline 
9 &	2 &	49 &	1 &	10 &	0 &	1 &	0\\
\hline 
10 &	2 &	61 &	1 &	12 &	1 &	0 &	1\\
\hline 
11 &	1 &	75 &	1 &	11 &	0 &	1 &	1\\
\hline 
12 &	1 &	87 &	1 &	9 &	0 &	1 &	0
\\
\hline 
\end{tabular}
\end{table}

\qquad

\begin{table}[H]
    \caption{Preferred Extension with Attack Probability $0.25$}
	\label{table:preferred_approx_25}
	\centering\renewcommand\arraystretch{1.2}
\begin{tabular}{ |p{1cm}|p{1cm}|p{1.5cm}|p{2cm}| p{2cm}|p{1.5cm}|p{1.5cm}|p{1.5cm}| }
 \hline
AAF Size & Ext Size & Attacks & Preference Sets & Preferences & CTime (ms) & VTime$1$ (ms) & VTime$2$ (ms)\\
 \hline
4 &	2 &	4 &	1 &	2 &	0 &	2 &	2\\
\hline 
5 &	2 &	7 &	1 &	4 &	1 &	1 &	3\\
\hline 
6 &	3 &	8 &	1 &	5 &	0 &	1 &	2\\
\hline 
7 &	3 &	12 &	1 &	7 &	1 &	5 &	2\\
\hline 
8 &	3 &	15 &	1 &	8 &	0 &	1 &	1\\
\hline 
9 &	3 &	20 &	1 &	10 &	0 &	2 &	2\\
\hline 
10 &	4 &	21 &	1 &	12 &	0 &	2 &	1\\
\hline 
11 &	4 &	27 &	1 &	14 &	0 &	1 &	0\\
\hline 
12 &	4 &	32 &	1 &	15 &	1 &	2 &	1\\
\hline 
13 &	4 &	41 &	1 &	18 &	1 &	2 &	1\\
\hline 
14 &	4 &	49 &	1 &	21 &	0 &	2 &	1\\
\hline 
15 &	4 &	52 &	1 &	24 &	0 &	2 &	1\\
\hline 
16 &	5 &	60 &	1 &	26 &	1 &	2 &	2 \\
\hline 
\end{tabular}
\end{table}

\qquad

\begin{table}[H]
    \caption{Preferred Extension with Attack Probability $0.50$}
	\label{table:preferred_approx_50}
	\centering\renewcommand\arraystretch{1.2}
\begin{tabular}{ |p{1cm}|p{1cm}|p{1.5cm}|p{2cm}| p{2cm}|p{1.5cm}|p{1.5cm}|p{1.5cm}| }
 \hline
AAF Size & Ext Size & Attacks & Preference Sets & Preferences & CTime (ms) & VTime$1$ (ms) & VTime$2$ (ms)\\
 \hline
4 &	2 &	7 &	1 &	3 &	0 &	2 &	2\\
\hline 
5 &	2 &	10 &	1 &	5 &	0 &	2 &	1\\
\hline 
6 &	2 &	15 &	1 &	7 &	0 &	1 &	1\\
\hline 
7 &	2 &	22 &	1 &	9 &	0 &	1 &	1\\
\hline 
8 &	2 &	28 &	1 &	11 &	0 &	3 &	1\\
\hline 
9 &	2 &	38 &	1 &	12 &	1 &	3 &	1\\
\hline 
10 &	3 &	45 &	1 &	16 &	1 &	2 &	3\\
\hline 
11 &	2 &	56 &	1 &	17 &	1 &	3 &	1\\
\hline 
12 &	2 &	66 &	1 &	17 &	0 &	2 &	1\\
\hline 
13 &	2 &	81 &	1 &	20 &	1 &	5 &	1\\
\hline 
14 &	3 &	90 &	1 &	25 &	1 &	2 &	2\\
\hline 
15 &	3 &	105 &	1 &	27 &	1 &	2 &	1\\
\hline 
16 &	2 &	122 &	1 &	26 &	0 &	3 &	1 
\\
\hline 
\end{tabular}
\end{table}

\qquad

\begin{table}[H]
    \caption{Preferred Extension with Attack Probability $0.75$}
	\label{table:preferred_approx_75}
	\centering\renewcommand\arraystretch{1.2}
\begin{tabular}{ |p{1cm}|p{1cm}|p{1.5cm}|p{2cm}| p{2cm}|p{1.5cm}|p{1.5cm}|p{1.5cm}| }
 \hline
AAF Size & Ext Size & Attacks & Preference Sets & Preferences & CTime (ms) & VTime$1$ (ms) & VTime$2$ (ms)\\
 \hline
4 &	1 &	9 &	1 &	3 &	0 &	1 &	2\\
\hline 
5 &	1 &	15 &	1 &	5 &	0 &	2 &	1\\
\hline 
6 &	1 &	23 &	1 &	6 &	0 &	3 &	2\\
\hline 
7 &	1 &	32 &	1 &	7 &	0 &	1 &	1\\
\hline 
8 &	1 &	42 &	1 &	9 &	0 &	1 &	2\\
\hline 
9 &	2 &	54 &	1 &	11 &	0 &	2 &	1\\
\hline 
10 &	2 &	67 &	1 &	12 &	0 &	1 &	1\\
\hline 
11 &	2 &	81 &	1 &	16 &	0 &	2 &	2\\
\hline 
12 &	2 &	99 &	1 &	16 &	1 &	2 &	1\\
\hline 
13 &	2 &	115 &	1 &	21 &	1 &	1 &	2\\
\hline 
14 &	2 &	134 &	1 &	21 &	0 &	1 &	1\\
\hline 
15 &	2 &	157 &	1 &	26 &	1 &	2 &	2\\
\hline 
16 &	2 &	184 &	1 &	28 &	1 &	2 &	3 
\\
\hline 
\end{tabular}
\end{table}

\qquad

\begin{table}[H]
    \caption{Stable Extension with Attack Probability $0.25$}
	\label{table:stable_25}
	\centering\renewcommand\arraystretch{1.2}
\begin{tabular}{ |p{1cm}|p{1cm}|p{1.5cm}|p{2cm}| p{2cm}|p{1.5cm}|p{1.5cm}|p{1.5cm}| }
 \hline
AAF Size & Ext Size & Attacks & Preference Sets & Preferences & CTime (ms) & VTime$1$ (ms) & VTime$2$ (ms)\\
 \hline
4 &	2 &	4 &	1 &	2 &	0 &	2 &	2\\
\hline 
5 &	3 &	6 &	1 &	4 &	0 &	2 &	1\\
\hline 
6 &	3 &	10 &	1 &	6 &	1 &	1 &	1\\
\hline 
7 &	3 &	11 &	1 &	7 &	0 &	1 &	1\\
\hline 
8 &	4 &	14 &	1 &	9 &	0 &	1 &	0\\
\hline 
9 &	4 &	19 &	1 &	11 &	0 &	0 &	0\\
\hline 
10 &	4 &	22 &	1 &	13 &	1 &	1 &	1\\
\hline 
11 &	4 &	29 &	1 &	14 &	1 &	1 &	0\\
\hline 
12 &	4 &	37 &	1 &	17 &	0 &	2 &	3\\
\hline 
13 &	4 &	39 &	1 &	19 &	0 &	2 &	2\\
\hline 
14 &	5 &	48 &	1 &	23 &	0 &	1 &	1\\
\hline 
15 &	5 &	52 &	1 &	23 &	1 &	2 &	1\\
\hline 
16 &	5 &	64 &	1 &	26 &	1 &	2 &	1 
\\
\hline 
\end{tabular}
\end{table}

\qquad

\begin{table}[H]
    \caption{Stable Extension with Attack Probability $0.50$}
	\label{table:stable_50}
	\centering\renewcommand\arraystretch{1.2}
\begin{tabular}{ |p{1cm}|p{1cm}|p{1.5cm}|p{2cm}| p{2cm}|p{1.5cm}|p{1.5cm}|p{1.5cm}| }
 \hline
AAF Size & Ext Size & Attacks & Preference Sets & Preferences & CTime (ms) & VTime$1$ (ms) & VTime$2$ (ms)\\
 \hline
4 &	2 &	7 &	1 &	3 &	0 &	1 &	3\\
\hline 
5 &	2 &	10 &	1 &	5 &	0 &	2 &	2\\
\hline 
6 &	2 &	15 &	1 &	6 &	1 &	2 &	1\\
\hline 
7 &	2 &	21 &	1 &	7 &	0 &	2 &	1\\
\hline 
8 &	2 &	30 &	1 &	10 &	0 &	4 &	1\\
\hline 
9 &	2 &	37 &	1 &	13 &	1 &	1 &	1\\
\hline 
10 &	2 &	46 &	1 &	14 &	0 &	2 &	1\\
\hline 
11 &	2 &	54 &	1 &	16 &	1 &	3 &	1\\
\hline 
12 &	3 &	62 &	1 &	19 &	1 &	2 &	2\\
\hline 
13 &	3 &	77 &	1 &	22 &	1 &	2 &	1\\
\hline 
14 &	3 &	92 &	1 &	24 &	1 &	3 &	1\\
\hline 
15 &	3 &	104 &	1 &	26 &	0 &	3 &	1\\
\hline 
16 &	3 &	117 &	1 &	29 &	1 &	2 &	2 
\\
\hline 
\end{tabular}
\end{table}

\qquad

\begin{table}[H]
    \caption{Stable Extension with Attack Probability $0.75$}
	\label{table:stable_75}
	\centering\renewcommand\arraystretch{1.2}
\begin{tabular}{ |p{1cm}|p{1cm}|p{1.5cm}|p{2cm}| p{2cm}|p{1.5cm}|p{1.5cm}|p{1.5cm}| }
 \hline
AAF Size & Ext Size & Attacks & Preference Sets & Preferences & CTime (ms) & VTime$1$ (ms) & VTime$2$ (ms)\\
 \hline
4 &	1 &	10 &	1 &	3 &	0 &	3 &	1\\
\hline 
5 &	1 &	16 &	1 &	5 &	0 &	2 &	1\\
\hline 
6 &	2 &	22 &	1 &	7 &	0 &	3 &	3\\
\hline 
7 &	1 &	32 &	1 &	7 &	0 &	1 &	2\\
\hline 
8 &	2 &	41 &	1 &	10 &	1 &	1 &	1\\
\hline 
9 &	1 &	57 &	1 &	10 &	0 &	4 &	2\\
\hline 
10 &	2 &	68 &	1 &	12 &	1 &	3 &	2\\
\hline 
11 &	1 &	85 &	1 &	13 &	1 &	3 &	3\\
\hline 
12 &	2 &	99 &	1 &	17 &	0 &	1 &	2\\
\hline 
13 &	2 &	117 &	1 &	19 &	0 &	2 &	1\\
\hline 
14 &	2 &	134 &	1 &	21 &	0 &	1 &	1\\
\hline 
15 &	2 &	156 &	1 &	25 &	0 &	2 &	2\\
\hline 
16 &	2 &	177 &	1 &	26 &	0 &	3 &	4 
\\
\hline 
\end{tabular}
\end{table}

\newpage 

\begin{center}
 \large{\textbf{Data sets for the Approximate Algorithm for larger AAF Sizes}}    
\end{center}

\begin{table}[H]
    \caption{Grounded Extension with Attack Probability $0.25$ for larger AAF Sizes}
	\label{table:grounded_scaled_approx_25}
	\centering\renewcommand\arraystretch{1.2}
\begin{tabular}{ |p{1cm}|p{1cm}|p{1.5cm}|p{2cm}| p{2cm}|p{1.5cm}|p{1.5cm}|p{1.5cm}| }
 \hline
AAF Size & Ext Size & Attacks & Preference Sets & Preferences & CTime (ms) & VTime$1$ (ms) & VTime$2$ (ms)\\
 \hline
5 &	2 &	4 &	1 &	3 &	0 &	1 &	1\\
\hline 
10 &	4 &	19 &	1 &	9 &	0 &	1 &	1\\
\hline 
15 &	2 &	49 &	1 &	6 &	0 &	2 &	2\\
\hline 
20 &	2 &	89 &	1 &	12 &	0 &	1 &	1\\
\hline 
25 &	1 &	149 &	1 &	8 &	0 &	1 &	1\\
\hline 
30 &	1 &	207 &	1 &	10 &	0 &	1 &	0\\
\hline 
35 &	1 &	293 &	1 &	9 &	0 &	1 &	0\\
\hline 
40 &	1 &	380 &	1 &	10 &	0 &	1 &	0\\
\hline 
45 &	1 &	492 &	1 &	11 &	0 &	1 &	0\\
\hline 
50 &	1 &	602 &	1 &	12 &	0 &	1 &	0\\
\hline 
55 &	1 &	747 &	1 &	15 &	0 &	1 &	3\\
\hline 
60 &	1 &	867 &	1 &	16 &	0 &	1 &	1\\
\hline 
\end{tabular}
\end{table}

\qquad

\begin{table}[H]
    \caption{Preferred Extension with Attack Probability $0.25$ for larger AAF Sizes}
	\label{table:preferred_scaled_approx_25}
	\centering\renewcommand\arraystretch{1.2}
\begin{tabular}{ |p{1cm}|p{1cm}|p{1.5cm}|p{2cm}| p{2cm}|p{1.5cm}|p{1.5cm}|p{1.5cm}| }
 \hline
AAF Size & Ext Size & Attacks & Preference Sets & Preferences & CTime (ms) & VTime$1$ (ms) & VTime$2$ (ms)\\
 \hline
5 &	2 &	5 &	1 &	4 &	0 &	2 &	3\\
\hline 
10 &	4 &	23 &	1 &	12 &	0 &	3 &	1\\
\hline 
15 &	4 &	55 &	1 &	22 &	0 &	3 &	2\\
\hline 
20 &	4 &	92 &	1 &	31 &	1 &	4 &	4\\
\hline 
25 &	6 &	153 &	1 &	51 &	1 &	6 &	5\\
\hline 
30 &	6 &	225 &	1 &	67 &	1 &	6 &	4\\
\hline 
35 &	6 &	298 &	1 &	86 &	2 &	7 &	5\\
\hline 
40 &	7 &	394 &	1 &	101 &	2 &	5 &	5\\
\hline 
45 &	7 &	500 &	1 &	120 &	2 &	7 &	5\\
\hline 
50 &	8 &	606 &	1 &	146 &	2 &	6 &	6\\
\hline 
55 &	7 &	733 &	1 &	156 &	4 &	7 &	7\\
\hline 
60 &	8 &	879 &	1 &	198 &	3 &	8 &	6\\
\hline  
\end{tabular}
\end{table}

\qquad

\begin{table}[H]
    \caption{Preferred Extension with Attack Probability $0.50$ for larger AAF Sizes}
	\label{table:preferred_scaled_approx_50}
	\centering\renewcommand\arraystretch{1.2}
\begin{tabular}{ |p{1cm}|p{1cm}|p{1.5cm}|p{2cm}| p{2cm}|p{1.5cm}|p{1.5cm}|p{1.5cm}| }
 \hline
AAF Size & Ext Size & Attacks & Preference Sets & Preferences & CTime (ms) & VTime$1$ (ms) & VTime$2$ (ms)\\
 \hline
5 &	2 &	11 &	1 &	5 &	0 &	2 &	1\\
\hline 
10 &	3 &	45 &	1 &	16 &	1 &	5 &	3\\
\hline 
15 &	3 &	104 &	1 &	27 &	1 &	5 &	2\\
\hline 
20 &	3 &	190 &	1 &	37 &	1 &	5 &	3\\
\hline 
25 &	3 &	302 &	1 &	52 &	2 &	5 &	4\\
\hline 
30 &	4 &	438 &	1 &	73 &	2 &	7 &	5\\
\hline 
35 &	4 &	592 &	1 &	91 &	2 &	7 &	6\\
\hline 
40 &	4 &	785 &	1 &	110 &	2 &	9 &	7\\
\hline 
45 &	4 &	984 &	1 &	123 &	2 &	8 &	6\\
\hline 
50 &	4 &	1229 &	1 &	141 &	2 &	8 &	7\\
\hline 
55 &	4 &	1488 &	1 &	154 &	2 &	8 &	6\\
\hline 
60 &	4 &	1774 &	1 &	179 &	2 &	8 &	7\\
\hline 
\end{tabular}
\end{table}

\qquad

\begin{table}[H]
    \caption{Preferred Extension with Attack Probability $0.75$ for larger AAF Sizes}
	\label{table:preferred_scaled_approx_75}
	\centering\renewcommand\arraystretch{1.2}
\begin{tabular}{ |p{1cm}|p{1cm}|p{1.5cm}|p{2cm}| p{2cm}|p{1.5cm}|p{1.5cm}|p{1.5cm}| }
 \hline
AAF Size & Ext Size & Attacks & Preference Sets & Preferences & CTime (ms) & VTime$1$ (ms) & VTime$2$ (ms)\\
 \hline
5 &	2 &	15 &	1 &	5 &	0 &	2 &	3\\
\hline 
10 &	2 &	70 &	1 &	15 &	0 &	2 &	3\\
\hline 
15 &	2 &	157 &	1 &	24 &	0 &	5 &	7\\
\hline 
20 &	2 &	288 &	1 &	34 &	0 &	4 &	5\\
\hline 
25 &	2 &	453 &	1 &	46 &	1 &	6 &	6\\
\hline 
30 &	2 &	651 &	1 &	53 &	2 &	7 &	6\\
\hline 
35 &	2 &	893 &	1 &	63 &	2 &	9 &	4\\
\hline 
40 &	2 &	1165 &	1 &	82 &	2 &	12 &	5\\
\hline 
45 &	3 &	1478 &	1 &	100 &	2 &	12 &	6\\
\hline 
50 &	2 &	1847 &	1 &	99 &	2 &	9 &	7\\
\hline 
55 &	2 &	2227 &	1 &	114 &	2 &	10 &	8\\
\hline 
60 &	3 &	2640 &	1 &	135 &	2 &	10 &	7\\
\hline 
\end{tabular}
\end{table}

\qquad

\begin{table}[H]
    \caption{Stable Extension with Attack Probability $0.25$ for larger AAF Sizes}
	\label{table:stable_scaled_approx_25}
	\centering\renewcommand\arraystretch{1.2}
\begin{tabular}{ |p{1cm}|p{1cm}|p{1.5cm}|p{2cm}| p{2cm}|p{1.5cm}|p{1.5cm}|p{1.5cm}| }
 \hline
AAF Size & Ext Size & Attacks & Preference Sets & Preferences & CTime (ms) & VTime$1$ (ms) & VTime$2$ (ms)\\
 \hline
5 &	2 &	6 &	1 &	4 &	0 &	2 &	2\\
\hline 
10 &	4 &	24 &	1 &	13 &	0 &	2 &	1\\
\hline 
15 &	5 &	50 &	1 &	24 &	0 &	1 &	3\\
\hline 
20 &	5 &	93 &	1 &	39 &	1 &	4 &	3\\
\hline 
25 &	6 &	148 &	1 &	52 &	1 &	5 &	4\\
\hline 
30 &	6 &	215 &	1 &	69 &	2 &	6 &	5\\
\hline 
35 &	6 &	300 &	1 &	87 &	2 &	5 &	5\\
\hline 
40 &	6 &	388 &	1 &	93 &	2 &	8 &	4\\
\hline 
45 &	7 &	497 &	1 &	125 &	2 &	5 &	6\\
\hline 
50 &	7 &	627 &	1 &	140 &	2 &	6 &	6\\
\hline 
55 &	8 &	740 &	1 &	165 &	3 &	8 &	7\\
\hline 
60 &	8 &	873 &	1 &	179 &	2 &	6 &	6\\
\hline 
\end{tabular}
\end{table}

\qquad

\begin{table}[H]
    \caption{Stable Extension with Attack Probability $0.50$ for larger AAF Sizes}
	\label{table:stable_scaled_approx_50}
	\centering\renewcommand\arraystretch{1.2}
\begin{tabular}{ |p{1cm}|p{1cm}|p{1.5cm}|p{2cm}| p{2cm}|p{1.5cm}|p{1.5cm}|p{1.5cm}| }
 \hline
AAF Size & Ext Size & Attacks & Preference Sets & Preferences & CTime (ms) & VTime$1$ (ms) & VTime$2$ (ms)\\
 \hline
5 &	2 &	10 &	1 &	5 &	0 &	1 &	4\\
\hline 
10 &	2 &	46 &	1 &	15 &	1 &	3 &	1\\
\hline 
15 &	3 &	105 &	1 &	26 &	1 &	3 &	2\\
\hline 
20 &	3 &	193 &	1 &	42 &	1 &	6 &	4\\
\hline 
25 &	3 &	302 &	1 &	51 &	2 &	7 &	4\\
\hline 
30 &	4 &	442 &	1 &	72 &	2 &	8 &	4\\
\hline 
35 &	4 &	595 &	1 &	88 &	2 &	8 &	6\\
\hline 
40 &	4 &	777 &	1 &	103 &	2 &	8 &	6\\
\hline 
45 &	4 &	996 &	1 &	126 &	2 &	6 &	6\\
\hline 
50 &	4 &	1231 &	1 &	140 &	2 &	10 &	6\\
\hline 
55 &	4 &	1498 &	1 &	160 &	2 &	8 &	7\\
\hline 
60 &	4 &	1762 &	1 &	182 &	2 &	8 &	7\\
\hline 
\end{tabular}
\end{table}

\qquad

\begin{table}[H]
    \caption{Stable Extension with Attack Probability $0.75$ for larger AAF Sizes}
	\label{table:stable_scaled_approx_50}
	\centering\renewcommand\arraystretch{1.2}
\begin{tabular}{ |p{1cm}|p{1cm}|p{1.5cm}|p{2cm}| p{2cm}|p{1.5cm}|p{1.5cm}|p{1.5cm}| }
 \hline
AAF Size & Ext Size & Attacks & Preference Sets & Preferences & CTime (ms) & VTime$1$ (ms) & VTime$2$ (ms)\\
 \hline
 5 &	1 &	16 &	1 &	4 &	0 &	2 &	2\\
\hline 
10 &	1 &	71 &	1 &	11 &	0 &	2 &	3\\
\hline 
15 &	2 &	156 &	1 &	25 &	0 &	3 &	5\\
\hline 
20 &	2 &	289 &	1 &	33 &	1 &	5 &	4\\
\hline 
25 &	2 &	445 &	1 &	41 &	1 &	5 &	3\\
\hline 
30 &	2 &	651 &	1 &	53 &	1 &	7 &	6\\
\hline 
35 &	2 &	890 &	1 &	66 &	2 &	10 &	5\\
\hline 
40 &	2 &	1166 &	1 &	79 &	2 &	10 &	7\\
\hline 
45 &	2 &	1480 &	1 &	93 &	2 &	10 &	6\\
\hline 
50 &	2 &	1832 &	1 &	100 &	2 &	9 &	6\\
\hline 
55 &	2 &	2228 &	1 &	119 &	2 &	10 &	7\\
\hline 
60 &	2 &	2656 &	1 &	130 &	2 &	8 &	8\\
\hline 
\end{tabular}
\end{table}








% \end{thebibliography}
\end{document}